\documentclass[openany,12pt]{book}  
\usepackage{amsmath}
\usepackage{amsthm}
\usepackage{amssymb}
\usepackage{geometry}
\usepackage[T1]{fontenc}
\usepackage[utf8]{inputenc}
\usepackage{newtxtext,newtxmath}
\usepackage{algorithm}
\usepackage{caption} 
\usepackage{algpseudocode}
\usepackage{graphicx} 
\usepackage{url}
\usepackage{tabularx}
\usepackage{booktabs}
\usepackage{tikz}
\usetikzlibrary{arrows.meta}

\usepackage[authoryear,round]{natbib}
\usepackage{hyperref}
\hypersetup{
  unicode,
  psdextra,
  colorlinks=true,
  citecolor=blue,
  linkcolor=blue,
  urlcolor=blue
}

\newtheorem{assumption}{Assumption}[section]
\newtheorem{definition}{Definition}[section]
\newtheorem{proposition}{Proposition}[section]
\newtheorem{theorem}{Theorem}[section]
\newtheorem{corollary}{Corollary}[section]

\newtheorem{example}{Example}[section]

\newcommand{\indep}{\mathop{\hspace{0.5mm} \perp\!\!\!\!\!\!\perp\hspace{0.5mm} }}
\usepackage{longtable}

\usepackage{makeidx}
\makeindex

\DeclareMathOperator*{\argmin}{arg\,min}
\DeclareMathOperator*{\argmax}{arg\,max}
\title{\textbf{Statistical Inference via Generative Models:\\[4mm]Flow Matching and Causal Inference}\vspace{100mm}}

\author{Shinto Eguchi}

\date{}

\begin{document}

\maketitle

\chapter*{Preface}
\addcontentsline{toc}{chapter}{Preface}

In recent years, generative AI has advanced with extraordinary speed, producing striking results in image synthesis, language generation, and related domains.
From the perspective of statistics, however, these developments still often appear under the sign of opacity.
The predictive performance may be impressive, yet the underlying mechanism is frequently difficult to interpret, analyze, or trust.
It is therefore unsurprising that many statisticians remain cautious about using such methods for statistical inference, model diagnosis, causal analysis, or experimental design.
This caution is not a symptom of conservatism.
It arises from a legitimate intellectual demand.
Statistics is concerned not merely with the reproduction of plausible observations, but with the disciplined clarification of what can be identified under explicit assumptions, and with the assessment of how accurately such quantities can be inferred.
The aim of this book is to reinterpret generative AI in the language of statistics, using \textbf{flow matching}\index{flow matching} as a concrete and timely focal point.

Its central message is straightforward.
As long as generative models are regarded simply as devices for producing visually or empirically convincing data, statisticians will remain outside observers.
Once they are understood instead as methods for the \textbf{nonparametric learning} of high-dimensional probability distributions\index{nonparametric learning}, an entirely different picture emerges.
If conditional distributions can be learned, then missing-data imputation becomes a problem of principled sampling.
If counterfactual distributions can be generated, then intervention effects become estimable objects.
If the temporal evolution of a distribution can be represented, then dynamic structure itself becomes amenable to analysis.
In this sense, generative modeling offers not merely a new class of algorithms, but a new computational language through which central statistical questions may be formulated and addressed in high dimensions.

Although flow matching is typically implemented with neural networks, this book does not treat neural networks as the main subject.
They serve only as flexible approximators of a velocity field, or more generally a vector field, $v_t(x)$.
What is fundamental is the representation of distributional change as a conservation law coupled with a dynamical system, expressed through the \textbf{continuity equation}\index{continuity equation}:
\begin{equation}
\partial_t \rho_t(x) + \nabla \cdot \bigl(\rho_t(x)\, v_t(x)\bigr) = 0 .
\end{equation}
From a historical point of view, one may recall Charles Stein's characterization of probability distributions through duality with differential operators, embodied in the \textbf{Stein identity}\index{Stein identity}.
That perspective remains deeply relevant to modern generative modeling, albeit in transformed form.
Score matching describes a distribution through the static geometry of the score field $\nabla \log \rho(x)$.
Flow matching extends this viewpoint by replacing the static target with a general \emph{time-dependent} vector field $v_t(x)$ and by learning not only the distribution itself but also the path along which it is transported.

This extension turns the phrase ``transport of high-dimensional distributions'' from an abstract metaphor into a concrete computational object.
The broader purpose of the present volume is to place generative models firmly within the domain of statistical inference.
In particular, we study how generative models can be used to estimate nuisance components while preserving inferential validity through orthogonalization and cross-fitting, in the framework of \textbf{double/debiased machine learning}\index{double machine learning} (DDML).
For statistics, the mere ability to generate samples is never sufficient.
What is required is \textbf{generation that supports inference}.
To attain this, one must design procedures in which the first-order effect of generative error is removed, or at least sharply attenuated, by principled statistical construction.

The exposition begins with two foundational themes---score matching and optimal transport---and develops flow matching from a minimal set of ideas: the continuity equation and vector fields.
It then moves toward structured statistical problems, including censoring and missingness in survival analysis, and further toward the estimation of counterfactual distributions in causal inference.
Generative models are, I believe, most valuable not where statistical structure is absent, but precisely where such structure is rich and explicit.
If this book succeeds in helping readers see generative AI not as a black box external to statistics, but as a new methodology for statistical inference itself, it will have achieved its purpose.
I also hope that it may serve as a point of departure for further research at the intersection of statistics, machine learning, and dynamical modeling.

\vspace{20mm}
\hspace{80mm}March 2026,

\medskip
\hspace{80mm}Shinto Eguchi

\hspace{80mm}The Institute of Statistical Mathematics

\hspace{80mm}Email: \texttt{eguchi@ism.ac.jp}

\tableofcontents

\chapter{Generative Models and Statistics: Motivation and Overview}
\index{generative model}
\index{flow matching}
\label{chap:intro}

A generative model is a technology that learns a high-dimensional distribution behind data and enables sampling from it.
The purpose of this chapter is to reinterpret generative models in the vocabulary of statistics and to clarify why \emph{flow matching} occupies a natural position for statistical inference.
After reading this chapter, you should be able to explain the following points:
\begin{itemize}
\item A generative model can be re-defined as an estimator of a distribution (a probability measure).
\item In high dimensions, sampling can be a more important computational principle than explicit density evaluation.
\item When distributional deformation is treated as a continuous-time motion (an ODE), the continuity equation appears naturally.
\end{itemize}

\section{What is a Generative Model?}
\label{sec:why_generative}

For example, an image with $28\times 28$ pixels can be represented as a high-dimensional vector $X\in\mathbb{R}^{784}$.
In such settings, explicitly estimating a density $p(x)$ and evaluating $\log p(x)$ often becomes a computational bottleneck.
A generative model introduces a reference distribution $Z\sim\pi(z)$ (e.g., standard Gaussian) and a mapping (generator) $G_\theta$ such that
\begin{align}\label{transform}
X = G_\theta(Z).
\end{align}
Instead of directly manipulating the density formula, it learns a \emph{sampleable representation} of the distribution.
In statistical terms, \eqref{transform} is a transformation model in a broad sense\index{transformation model},
and the resulting distribution can be viewed as the pushforward\index{pushforward} $P_\theta=(G_\theta)_\#\pi$
(see \eqref{oshidashi} below).

{Sampleability greatly expands statistical operations.}
Once we can generate i.i.d.\ samples from $X\sim P_\theta$, many computations become feasible via Monte Carlo.
That is, for any integrable function $\phi$,
\[
\mathbb{E}_{P_\theta}[\phi(X)]
\]
can be approximated by the empirical average of generated samples.
The density of the transformed distribution $P_\theta$ is often complicated, and likelihood evaluation can be expensive.
Moreover, if conditional generation is available, missing-data imputation and counterfactual generation can be described within the same framework.
For covariates $X$ and a response $Y$, using a latent variable $\varepsilon\sim\pi$,
\[
Y = G_\theta(X,\varepsilon),
\]
we obtain a sampleable representation of the conditional distribution of $Y\,|\,X$.
Whereas classical regression mainly targets the mean structure $\mathbb{E}[Y\,|\,X]$, a generative model can learn the entire conditional distribution,
including variance, multimodality, and tail behavior.

\paragraph{From distribution estimation to functional estimation}
The ultimate goal of statistics is often not the distribution itself, but estimation of a functional $\psi(P)$
(e.g., means, quantiles, causal effects, policy values).
A generative model provides an approximate distribution $\widehat P_\theta$, from which we can compute $\psi(\widehat P_\theta)$ by Monte Carlo.
A key caveat is that the estimation error of $\widehat P_\theta$ can be amplified through the sensitivity of $\psi$,
as quantified by the influence function\index{influence function}.
This issue can be organized in the language of semiparametric inference;
in Chapter~6 we describe a principled route to reconcile generation and inference via orthogonalization and cross-fitting.

\paragraph{A note on terminology}
In a measure-theoretic treatment, a probability distribution (measure) $P$ and a density $p$ should be distinguished.
Following common practice in statistics and machine learning, however, we sometimes refer to a density $p$ as a ``distribution''
when no confusion arises from the context.

\section{Generative Models as Distribution Calibration}
\label{sec:paradigm_shift}

\paragraph{Model misspecification as a distributional distortion}
A central goal of statistical modeling is to summarize a phenomenon by interpretable parameters and to enable estimation and testing.
In real data, however, a parametric model often captures only the ``skeleton'' of a phenomenon,
while much of the shape and distortion of the distribution is governed by infinite-dimensional degrees of freedom.
It is therefore natural to view misspecification not merely as a parameter mismatch, but as an infinite-dimensional deviation---a distortion of the distribution itself.

\paragraph{Base model plus residual, and calibration}
In this book we adopt a semiparametric viewpoint: we decompose the data-generating distribution into
(i) a parametric base model that we wish to preserve and interpret (main effects, causal effects, etc.), and
(ii) a nonparametric component (nuisance)\index{nuisance} representing deviations from the base model.
The key is not simply to ``absorb'' nuisance flexibly, but to design the learning procedure so that it does not destroy target inference.

This viewpoint can be expressed succinctly as follows.
Let $p_\theta$ be the base model, and represent the observed distribution by an additional transformation $T_\eta$:
\[
X = T_\eta(X^{(0)}),\qquad X^{(0)}\sim p_\theta.
\]
That is, the observed variable $X$ is modeled as an unknown transformation of a base latent variable $X^{(0)}$.
Here $\theta$ plays the role of an interpretable parameter, while $\eta$ accounts for misspecification and complexity.
Throughout the book, the term \emph{calibration} means not a black-box correction,
but an explicit modeling of distributional distortion as an additional transformation that can be learned from data.

\begin{quote}
\small
\noindent\textbf{Design principles compatible with inference (preview).}
Learning a nuisance component with high flexibility can invalidate inference if one relies on naive plug-in estimation.
A powerful remedy in semiparametric inference is to construct estimating equations that are first-order insensitive
to nuisance perturbations (orthogonality)\index{orthogonality}, and to further reduce overfitting bias via cross-fitting\index{cross-fitting}.
This principle is systematized as double/debiased machine learning (DDML)\index{double machine learning} \citep{chernozhukov2018ddml}.
In Chapters~5--7, we explain how to use generative models as nuisance estimators while still guaranteeing statistical inference.
\end{quote}

We focus on the flow matching as a main tool.
There are many approaches to generative modeling, including likelihood-based methods (normalizing flows),
diffusion/score-based methods, and GAN-type methods
\citep{kingma2014auto,rezende2015variational,dinh2017realnvp,sohl2015deep,ho2020ddpm,song2021scorebased,goodfellow2014gan}.
We focus on flow matching for three reasons.
First, it describes distributional deformation through the continuity equation, naturally connecting to score estimation and optimal transport.
Second, its training objective can be implemented as sample-based regression (often a squared loss), which is stable and easy to explain.
Third, the velocity field need not be a gradient field; it can represent general distribution transformations including dependence and irreversibility.

The basic framework of flow matching (FM) was proposed by \citet{lipman2022flowmatching}.
For conditional flow matching (CFM) and its generalizations, see \citet{tong2023cfm}.
A closely related idea, rectified flow, is represented by \citet{liu2022rectified}.
As background, continuous normalizing flows (CNF) are discussed in \citet{chen2018neuralode,grathwohl2019ffjord},
and a unified formulation of diffusion/score SDEs is provided by \citet{ho2020ddpm,song2020scoresde}.

\section{Why Do Differential Equations Appear?}
\label{sec:why_ode}

Because they can learn the deformation process instead of constructing a map directly.
The goal is to obtain a mapping $T$ that transports a reference distribution to a target (data) distribution.
In high-dimensional spaces, it is difficult to directly construct a good nonlinear map $T$ in one shot.
We therefore relax the problem from ``finding a map'' to ``finding a process of movement.''
Consider a continuous path of densities $\{\rho_t\}_{t\in[0,1]}$ such that $t=0$ corresponds to the reference distribution and $t=1$ to the data distribution.
We describe particle motion by the ODE
\[
\frac{d}{dt}X_t = v_t(X_t),
\]
where $v_t$ is a time-dependent vector field.
The core idea is that estimating $v_t$---a local rule of motion---is often easier than directly constructing the global map $T$,
because it can be treated as function approximation (regression).

\paragraph{The continuity equation:} It is generated from the consistency between particle motion and density evolution.
When particles move according to a velocity field\index{velocity field} $v_t$, the population density must satisfy mass conservation,
leading to the continuity equation\index{continuity equation}
\[
\partial_t\rho_t(x)+\nabla\cdot\{\rho_t(x)v_t(x)\}=0.
\]
Although the divergence operator $\nabla\cdot$ may look heavy at first sight, in one dimension it reduces to
$\partial_t\rho+\partial_x(\rho v)=0$,
which simply states a conservation law: the change in density is determined by inflow and outflow of flux $\rho v$.
In Chapter~3 we provide an intuitive explanation and confirm that this equation follows directly from Gauss's divergence theorem.

\section{Structure and Notation}
\label{sec:notation_roadmap}

\paragraph{Roadmap of the book}
We briefly outline the flow of the book.
This chapter provided a high-level view of the interface between statistics and generative modeling and introduced our standpoint and vocabulary.
Chapter~2 organizes score matching and clarifies that the score function is a differential object that characterizes distributions.
We then show how Stein identities generate expectation equalities (estimating equations) for unknown distributions.

Chapter~3 focuses on flow matching, including how to learn a velocity field and how to implement distribution transport.
Chapter~4 interprets neural vector-field approximation as nonparametric regression and discusses statistical properties from the viewpoints of approximation,
generalization, stability, and asymptotic normality.
Chapter~5 connects the framework to linear regression, copulas, survival analysis, and missing-data imputation.
Chapter~6 treats causal inference, in particular, generation of counterfactual distributions.
Chapter~7 addresses goodness-of-fit testing and model diagnostics via kernel methods and score matching.

\paragraph{Reproducibility and code}
The code for numerical experiments in this book is publicly available at the GitHub repository
\[
\text{\url{https://github.com/shinto-eguchi/Statistical-flow-matching-en}}
\]

\bigskip

We summarize the mathematical notation used throughout the book. See the symbol list at the end of the book for additional details.

\bigskip

\paragraph{Probability and measure notation}
Let $P$ denote a probability measure and $X$ a random variable.
When a density exists, we write $p(x)$ and denote $X\sim P$ or $X\sim p$.
Expectation is denoted by $\mathbb{E}$, and expectation under $P$ by $\mathbb{E}_P$.
Given samples $X_1,\dots,X_n$, the empirical measure is written as
\[
\widehat P_n=\frac1n\sum_{i=1}^n \delta_{X_i}.
\]
\index{empirical measure}

\paragraph{Vector fields and differential operators}
We write a vector field as $v:\mathbb{R}^d\to\mathbb{R}^d$.
The gradient is denoted by $\nabla$, and the divergence by $\nabla\cdot$.
The score function\index{score function} is defined as
\[
s_p(x)=\nabla \log p(x).
\]
A vector field is called a gradient field if there exists a scalar potential $\varphi$ such that $v=\nabla \varphi$.

\paragraph{Pushforward and distribution transformation}
\index{pushforward}
For a mapping $T:\mathbb{R}^d\to\mathbb{R}^d$, the pushforward of $P$ is denoted by $T_\# P$; that is,
\begin{align}\label{oshidashi}
X\sim P \quad \Longrightarrow \quad T(X)\sim T_\# P.
\end{align}
A generative model represents $P_\theta$ as $P_\theta=(G_\theta)_\#\pi$ using a reference distribution $\pi$ and a mapping $G_\theta$.

\paragraph{Summary}
We re-defined generative models in the vocabulary of statistics and described a viewpoint that treats high-dimensional distribution estimation
as learning a sampleable transformation.
We also confirmed that the continuity equation appears naturally when distributional deformation is modeled as a continuous-time motion.
From the next chapter onward, we develop the mathematical tools---score functions, Stein identities, optimal transport, and vector fields---and bridge
theory and applications centered on flow matching.

\chapter{Score matching}

In high-dimensional problems, it is often more natural---both computationally and theoretically---to focus on the \emph{(spatial) score function} that describes the local shape of a distribution, rather than to manipulate the density $p(x)$ itself.
\index{score function}
\index{Stein identity}
\index{Fisher divergence}
\index{score matching}
We proceed in four steps:
First, we recall why likelihood-based training becomes difficult in high dimensions and how the score function
\[
s_p(x)=\nabla \log p(x)
\]
avoids the normalizing constant.
Second, we derive the score matching objective and explain its practical variants, including denoising score matching,
viewed as regression of a vector field.
Third, we introduce the Stein identity as the key device that turns distributional correctness into expectation equalities,
and we clarify the relationship between score matching and the Fisher divergence.
Finally, we highlight a change of viewpoint that will be central in the next chapter:
we move from \emph{gradient fields} (scores) to \emph{general vector fields} (velocities), which leads naturally to flow matching.
\index{denoising score matching}

\section{Gradient, Divergence, and Laplacian}
In this book, we view learning high-dimensional distributions through the lens of regression of vector fields and their stability.
As a result, notation from multivariate calculus appears frequently.
The goal of this section is not to provide fully rigorous mathematical definitions, but to set up the minimal toolkit needed to follow subsequent algebraic manipulations.

Let $x=(x_1,\dots,x_d)^\top\in\mathbb{R}^d$.
Let $f:\mathbb{R}^d\to\mathbb{R}$ be a scalar field and $v:\mathbb{R}^d\to\mathbb{R}^d$ be a vector field.
We write the partial derivative as $\partial_i f(x)=\partial f(x)/\partial x_i$.

\paragraph{Gradient}
The gradient is a vector that represents the direction of steepest increase of a scalar field $f$ and its magnitude:
\[
\nabla f(x)
=
\left(
\partial_1 f(x),\dots,\partial_d f(x)
\right)^\top .
\]
As a first-order approximation,
\[
f(x+h)\approx f(x)+\nabla f(x)^\top h
\]
holds.
Thus, $\nabla f(x)$ indicates the direction that increases $f$ most rapidly locally.

\paragraph{Jacobian and Hessian}
For a vector field $v(x)=(v_1(x),\dots,v_d(x))^\top$, the Jacobian is
\[
\nabla v(x)
=
\left(\partial_j v_i(x)\right)_{i,j=1}^d,
\]
a matrix collecting the first-order rates of change of each component.
For a scalar field $f$, the Hessian is
\[
\nabla^2 f(x)
=
\left(\partial_i\partial_j f(x)\right)_{i,j=1}^d,
\]
which encodes second-order curvature information.
As a second-order approximation,
\[
f(x+h)\approx f(x)+\nabla f(x)^\top h+\frac{1}{2}h^\top\{\nabla^2 f(x)\}h
\]
holds.

\paragraph{Divergence}
The divergence measures whether a vector field locally ``expands'' (sources) or ``contracts'' (sinks):
\[
\nabla\cdot v(x)
=
\sum_{i=1}^d \partial_i v_i(x).
\]
In particular, when $v=\nabla f$,
\[
\nabla\cdot(\nabla f)=\Delta f
\]
so divergence and the Laplacian are directly linked.
In this book, divergences of the score function $s_\theta(x)$ and of a velocity field $v_\theta(t,x)$ appear repeatedly.
It is enough to interpret $\nabla\cdot v$ as a \emph{local rate of volume change}.

\paragraph{Laplacian}
The Laplacian is an operator that captures the strength of diffusive smoothing of a scalar field:
\[
\Delta f(x)
=
\sum_{i=1}^d \partial_i^2 f(x).
\]
Using the Hessian,
\[
\Delta f(x)=\mathrm{tr}\{\nabla^2 f(x)\}
\]
holds.
When $\Delta$ appears later in discussions of score matching or diffusion processes,
it suffices to view it as ``the sum of second derivatives'' or ``the total curvature.''

\paragraph{Trace}
For a square matrix $A\in\mathbb{R}^{d\times d}$, the trace is the sum of diagonal entries:
\[
\mathrm{tr}(A)=\sum_{i=1}^d A_{ii}.
\]
The trace is stable under coordinate changes, and in particular the identity $\mathrm{tr}(\nabla^2 f)=\Delta f$ is important.
We also frequently use, whenever dimensions are compatible,
\[
\mathrm{tr}(AB)=\mathrm{tr}(BA).
\]

\paragraph{Norms and inner products}
The inner product and Euclidean norm are
\[
\langle a,b\rangle=a^\top b,\qquad \|a\|=(a^\top a)^{1/2}.
\]
For matrices, we sometimes use the Frobenius norm
\[
\|A\|_{\mathrm{F}}=\{\mathrm{tr}(A^\top A)\}^{1/2}.
\]
Many bounds in this book reduce to measuring ``squared errors of vector fields'' or ``the size of a Jacobian''
in terms of these norms.

\paragraph{Perspective on differential operators}
In later chapters, differential objects play a more central role than distributions or densities themselves, for example,
\[
\text{score function }\;\; s(x)=\nabla \log p(x),
\qquad
\text{velocity field }\;\; v(t,x).
\]
\index{score function}
\index{velocity field}
With this viewpoint, it is often helpful to read $\nabla$ as ``local directional information,''
$\nabla\cdot$ as ``local volume change,'' and $\Delta$ as ``curvature (diffusion).''
From this point on, we will not re-define these operators rigorously each time;
instead, we will use the necessary calculus rules guided by the above intuition.

\section{Score functions}
\index{score function}

Many models take the form
\[
p(x)=\frac{\tilde p(x)}{Z},\qquad
Z=\int_{\mathbb{R}^d}\tilde p(x)\,dx,
\]
where evaluating $Z$ becomes an intractable integration problem in high dimensions.
In maximum likelihood estimation, the term
$\log p(x)=\log\tilde p(x)-\log Z$
appears explicitly and can become a major bottleneck.
In statistical physics, $Z$ is called the \emph{partition function}, the total mass obtained by summing the unnormalized weight $\tilde p(x)$ over all states.
For example, if
\[
\tilde p(x)=\exp\{-\beta E(x)\},
\]
then $-\beta^{-1}\log Z$ corresponds to the (Helmholtz) free energy, and $Z$ summarizes macroscopic properties of the system.

\paragraph{Change of viewpoint:}  We work with the ``gradient'' rather than the density.
The \emph{score function}
\[
s_p(x)=\nabla \log p(x)
\]
does not involve $Z$.
Indeed,
\[
\nabla \log p(x)=\nabla \log\tilde p(x)-\nabla \log Z=\nabla \log\tilde p(x),
\]
because $Z$ is a constant with respect to $x$.
Thus, once the unnormalized density $\tilde p$ is specified, we can work with local shape information without evaluating the normalizing constant.
For instance, for $N(\mu,\Sigma)$,
\[
s_p(x)=-\Sigma^{-1}(x-\mu),
\]
so the vector points toward the mean direction.
One may interpret the score as indicating ``which direction increases the density locally.''
In this sense, the score is an arrow pointing toward high-density regions.

\paragraph{Two notions of ``score''}
\begin{itemize}
\item \textbf{Fisher score:} $\nabla_\theta\log p(x;\theta)$ (parameter gradient of the log-likelihood).
\item \textbf{Stein score:} $\nabla_x\log p(x;\theta)$ (spatial gradient in the data domain).
\end{itemize}
The Fisher score is fundamental for estimating a statistical parameter $\theta$ and describes local sensitivity of the likelihood.
In information geometry, $\theta$ is viewed as a coordinate on the statistical model
${\mathcal M}=\{p(\cdot;\theta):\theta\in\Theta\}$,
and interpreting the Fisher score as a tangent vector leads to intrinsic structures such as the Fisher information metric and natural gradients \citep{amari2000methods}.

In contrast, the Stein score is a gradient with respect to the data variable $x$ and directly describes the local geometry of the distribution.
In this book, when we simply say ``score,'' we refer to this Stein score unless otherwise stated.
In generative modeling, we continuously move $x$ and deform/transport distributions, so $\nabla_x\log p(x;\theta)$ naturally connects to velocity fields and the continuity equation, and plays a central role.

\paragraph{The score determines the distribution (up to a constant)}
On a connected domain, integrating $\nabla_x\log p$ along paths recovers $\log p$ up to an additive constant.
However, not every vector field can be written as $\nabla\log p$; an integrability (consistency) condition is required.
This point reappears in Section~\ref{subsec:grad-to-vector} when we move from gradient fields to general vector fields.

\section{Score Matching}
\label{sec:score-matching}
\index{score matching}

We derive the score matching principle of \citet{hyvarinen2005score} as a method to estimate the score
$s_0(x)=\nabla_x\log p_0(x)$
of an unknown true density $p_0$ from observations
\[
X_1,\dots,X_n\sim p_0.
\]

\paragraph{Fisher divergence}
\index{Fisher divergence}
For a model $p(x;\theta)$, define the score function $s_\theta(x)=\nabla_x\log p(x;\theta)$ and consider
\begin{align}
\mathcal{L}(\theta)
&=
\tfrac12\,\mathbb{E}_{p_0}\!\left[\|s_\theta(X)-s_0(X)\|^2\right].
\label{eq:SM_FisherDiv}
\end{align}
This criterion involves the unknown $s_0$ and is not directly computable.

\paragraph{Eliminate the true score by integration by parts}
To transform the cross term, assume boundary terms vanish (e.g., $p_0(x)s_\theta(x)$ decays sufficiently fast at infinity, specifically $\lim_{\lVert x\rVert\to\infty} p_0(x)s_\theta(x) = 0$).
Since $s_0=\nabla_x\log p_0=\nabla_x p_0/p_0$,
\[
\mathbb{E}_{p_0}\!\left[s_\theta(X)^\top s_0(X)\right]
=
\int s_\theta(x)^\top \nabla_x p_0(x)\,dx
=
-\mathbb{E}_{p_0}\!\left[\nabla_x\cdot s_\theta(X)\right],
\]
where $\nabla\cdot$ denotes divergence.
Dropping constants, minimizing \eqref{eq:SM_FisherDiv} is therefore equivalent to minimizing
\begin{equation}
J(\theta)
=
\mathbb{E}_{p_0}\!\left[
\tfrac12\,\|s_\theta(X)\|^2+\nabla_x\cdot s_\theta(X)
\right].
\label{eq:SM_population}
\end{equation}
The essential point is that the right-hand side does not involve the unknown form of $p_0$ and can be estimated from data.

\paragraph{Empirical objective}
Define
\begin{equation}
\widehat{J}_n(\theta)
=
\frac1n\sum_{i=1}^n
\left\{
\tfrac12\|s_\theta(X_i)\|^2+\nabla_x\cdot s_\theta(X_i)
\right\}.
\label{eq:SM_empirical}
\end{equation}
Minimizing \eqref{eq:SM_empirical} yields $\widehat\theta_n$.
This is a standard $M$-estimator and is consistent under usual conditions.
Under model misspecification, the limit $\theta^\ast$ can be interpreted as the Fisher-divergence projection \citep{white1982maximum}.

\paragraph{Handling divergence (second derivatives)}
The empirical criterion \eqref{eq:SM_empirical} includes
\[
\nabla_x\cdot s_\theta(x)=\mathrm{tr}\{\nabla_x^2\log p(x;\theta)\},
\]
a trace of the Hessian.
In high dimensions, practical options include
(i) stochastic trace estimation of Hutchinson type, and
(ii) denoising score matching, which avoids explicit second derivatives by adding noise \citep{vincent2011connection}
(see Appendix~\ref{app:trace-dsm}).

\paragraph{Example: energy-based model}
If $p_\theta(x)\propto \exp\{-U_\theta(x)\}$, then $s_\theta(x)=-\nabla_x U_\theta(x)$ and
\[
J(\theta)
=
\mathbb{E}_{p_0}\!\left[
\tfrac12\|\nabla U_\theta(X)\|^2-\Delta U_\theta(X)
\right],
\]
where $\Delta$ is the Laplacian.
The normalizing constant does not appear.

\begin{example}[Quartic potential model]
As an example of an unnormalized model with an analytically inconvenient normalizing constant, consider
\[
p_\theta(x)=\frac{1}{Z(\theta)}f_\theta(x),\qquad
f_\theta(x)=\exp(\theta_1 x+\theta_2 x^2+\theta_3 x^4).
\]
Here $\log Z(\theta)$ is an integral depending on $\theta$, so numerical integration can dominate maximum likelihood computation.

\begin{figure}[t]
  \centering
  \includegraphics[width=0.5\linewidth]{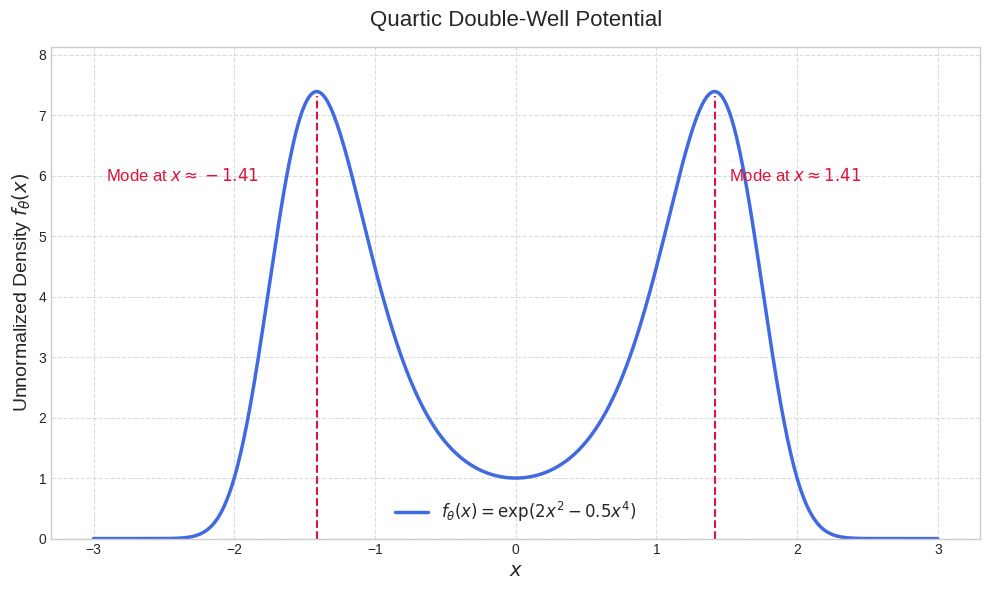}
  \caption{Quartic potential model $p_\theta(x)$ with $\theta=(0,2,-0.5)$.}
  \label{fig:quartic}
\end{figure}

In score matching, the normalizing constant cancels because we use spatial derivatives.
Indeed,
\[
s_\theta(x)=\nabla_x\log p_\theta(x)=\nabla_x\log f_\theta(x)
=\theta_1+2\theta_2 x+4\theta_3 x^3,
\]
and in one dimension the divergence term is explicit:
\[
\nabla_x\cdot s_\theta(x)=\frac{d}{dx}s_\theta(x)=2\theta_2+12\theta_3 x^2.
\]
Hence the empirical score matching objective becomes
\[
\widehat J_n(\theta)=\frac1n\sum_{i=1}^n\Bigl\{\frac12
\bigl(\theta_1+2\theta_2 x_i+4\theta_3 x_i^3\bigr)^2
+2\theta_2+12\theta_3 x_i^2
\Bigr\}.
\]
This is a quadratic form in $\theta$, so minimization reduces to solving linear equations and is computationally light.
When $\theta_3<0$ and $\theta_2>0$, the density can become bimodal, a regime in which likelihood-based fitting tends to be numerically fragile (see Figure~\ref{fig:quartic}).
\end{example}

\begin{example}[Gaussian graphical model (GGM)]
\label{ex:ggm-sm}
Consider the $d$-dimensional Gaussian model
\[
p_K(x)
=
(2\pi)^{-d/2}(\det K)^{1/2}\exp\!\left(-\tfrac12 x^\top K x\right),
\qquad K\succ 0,
\]
where $K=\Sigma^{-1}$ is the precision matrix (inverse covariance).
Assume the data are centered so that $\mu=0$.
In a GGM, $K_{ij}=0$ $(i\ne j)$ corresponds to $X_i\indep X_j\mid X_{-\{i,j\}}$, so the zero pattern of $K$ encodes the edge set of an undirected graph.

The score is $s_K(x)=-Kx$, and the divergence is $\nabla\!\cdot s_K(x)=-\mathrm{tr}(K)$.
Therefore the score matching objective becomes
\[
\widehat J(K)=\tfrac12\,\mathrm{tr}(K^2S)-\mathrm{tr}(K),
\qquad
S=\frac1n\sum_{i=1}^n (x_i-\bar x)(x_i-\bar x)^\top.
\]
The key point is that $\log\det K$ does not appear.
If $n>d$ and $S$ is nonsingular, then without regularization $\widehat K=S^{-1}$, which coincides with the MLE;
in high dimensions, regularization is essential.

Compared to the lightweight computation of regularized score matching in GGMs, regularized likelihood-based methods such as the graphical lasso are known to be dominated by the cost of optimizing an objective containing $\log\det$ in high dimensions.
Quantitative comparisons of RMSE and wall-clock time depend on the regularization parameter and stopping criteria.
We defer details to Appendix~\ref{app:ggm} and emphasize here only the conceptual point:
avoiding the normalizing constant yields a different computational principle.
\end{example}

\section{Stein's Lemma}
\index{Stein!lemma}
\index{Fisher information}
\label{subsec:stein-fisher}

This section makes the Stein identity underlying score matching explicit and clarifies its equivalence to Fisher divergence.

Let $X\sim N(0,I_d)$.
For a smooth function $f:\mathbb{R}^d\to\mathbb{R}^d$ such that boundary terms vanish, we have
\begin{equation}
\mathbb{E}\bigl[X^\top f(X)\bigr]
=
\mathbb{E}\bigl[\nabla_x\cdot f(X)\bigr].
\label{eq:stein-normal}
\end{equation}
This is the prototype of the integration-by-parts argument used in this chapter \citep{stein1981estimation}.

\index{Stein!identity}
For any smooth density $p$, the following identity holds:
\begin{equation}
\mathbb{E}_p\!\left[s_p(X)^\top f(X)+\nabla_x\cdot f(X)\right]=0.
\label{eq:stein-identity}
\end{equation}
The operator
\[
(\mathcal{T}_p f)(x)=s_p(x)^\top f(x)+\nabla_x\cdot f(x)
\]
is called the \emph{Stein operator}.

\paragraph{Score matching and Fisher divergence}
\index{score matching}
\index{Fisher divergence}
Setting $p=p_0$ and $f=s_\theta$ in \eqref{eq:stein-identity} yields
\[
\mathbb{E}_{p_0}\!\left[\nabla_x\cdot s_\theta(X)\right]
=
-\mathbb{E}_{p_0}\!\left[s_0(X)^\top s_\theta(X)\right],
\]
and therefore
\begin{equation}
J(\theta)+\tfrac12\mathbb{E}_{p_0}\!\left[\|s_0(X)\|^2\right]
=
\tfrac12\mathbb{E}_{p_0}\!\left[\|s_\theta(X)-s_0(X)\|^2\right].
\label{eq:fisher-equivalence}
\end{equation}
Thus score matching is equivalent to minimizing Fisher divergence.
Moreover, $\mathcal{I}(p_0)=\mathbb{E}_{p_0}[\|s_0(X)\|^2]$ corresponds to the Fisher information for a location parameter.
Equation \eqref{eq:fisher-equivalence} shows that a natural discrepancy between distributions arises as a squared error of scores.
For robust estimation via density-weighted Stein operators and weighted Fisher divergence, see \citet{eguchi2025robust}.

\paragraph{Connection to KSD and SVGD}
Since the Stein identity can be viewed as a characterization of $p$, one defines a Stein discrepancy\index{Stein!discrepancy}
\[
\mathcal{S}(p\|q;\mathcal{F})
=
\sup_{f\in\mathcal{F}}
\Bigl(\mathbb{E}_{X\sim p}\bigl[(\mathcal{T}_q f)(X)\bigr]\Bigr)^2,
\]
where $\mathcal{F}$ is a chosen function class.
If $\mathcal{F}$ is sufficiently rich, then
\[
\mathcal{S}(p\|q;\mathcal{F})
=0\quad
\Longleftrightarrow\quad p=q.
\]
This forms a basis for goodness-of-fit testing and particle methods.
In particular, if $\mathcal{F}$ is taken as an RKHS, the kernel Stein discrepancy (KSD) admits a closed form.
Stein variational gradient descent (SVGD) builds on this, and the resulting estimator takes the form of a $U$-statistic.
We defer concrete formulas and estimation details to Section~\ref{U-stat}.

\paragraph{Revisiting Stein estimators}
Stein's lemma also simplifies risk calculations\index{Stein!lemma}.
As a classical example, consider estimating $\mu$ under $X\sim N(\mu,I_d)$ \citep{james1961estimation}.
Write an estimator as
\[
\delta(X)=X+g(X),
\]
where $g:\mathbb{R}^d\to\mathbb{R}^d$ is smooth.
The squared error risk is
\begin{align}\nonumber
R(\mu,\delta)&=\mathbb{E}_\mu\!\left[\|\delta(X)-\mu\|^2\right]\\[3mm] \nonumber
&=
\mathbb{E}_\mu\!\left[\|X-\mu\|^2\right]
+2\mathbb{E}_\mu\!\left[(X-\mu)^\top g(X)\right]
+\mathbb{E}_\mu\!\left[\|g(X)\|^2\right].
\end{align}
Let $Z=X-\mu\sim N(0,I_d)$.
By \eqref{eq:stein-normal},
\[
\mathbb{E}_\mu\!\left[(X-\mu)^\top g(X)\right]
=
\mathbb{E}\bigl[Z^\top g(Z+\mu)\bigr]
=
\mathbb{E}\bigl[\nabla_x\cdot g(X)\bigr].
\]
Hence
\begin{equation}
R(\mu,\delta)
=
d
+2\,\mathbb{E}_\mu\!\left[\nabla_x\cdot g(X)\right]
+\mathbb{E}_\mu\!\left[\|g(X)\|^2\right].
\label{eq:stein-risk}
\end{equation}
This expression does not involve $\mu$ explicitly; it evaluates risk solely through the divergence and squared norm of $g$.
The James--Stein estimator uses the shrinkage term $g(x)=-(d-2)x/\|x\|^2$, and \eqref{eq:stein-risk} implies that its risk is smaller than that of the MLE $X$ when $d\ge 3$:
\[
R(\mu, \delta^{\mathrm{JS}})
=
d
-
(d-2)^2 \,\mathbb{E}\left[ \frac{1}{\|X\|^2} \right].
\]
For $d\ge3$, $\mathbb{E}[1/\|X\|^2]$ exists and is strictly positive.
As seen in Figure~\ref{fig:james-stein-field}, $g(X)$ forms a strong attractive field toward the origin, geometrically corresponding to a large negative divergence; this provides an intuitive source of risk reduction.
The technique of ``killing boundary terms and pushing unknown quantities into expectations'' is structurally the same as in score matching.

\begin{figure}[t]
  \centering
  \includegraphics[width=0.5\linewidth]{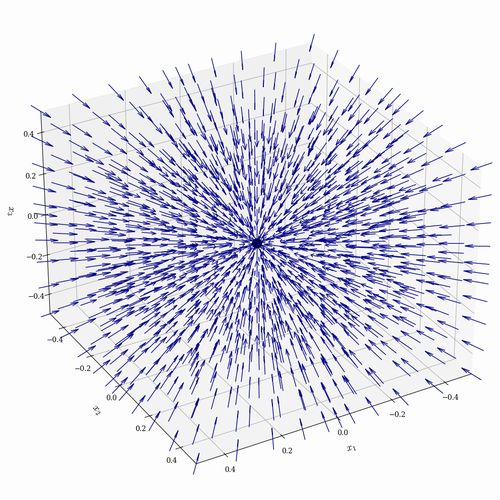}
  \caption{Vector field induced by a James--Stein-type shrinkage term.}
  \label{fig:james-stein-field}
\end{figure}

\paragraph{Stein identities on manifolds}
Let $(\mathcal{M},g)$ be an oriented Riemannian manifold without boundary.
Then, Gauss' divergence theorem yields\index{Stein!identity}
\[
\mathbb{E}_{P}\Bigl[\Delta f(X)+\langle \nabla\log p(X),\nabla f(X)\rangle\Bigr]=0,
\]
where $dP=p\,d\mathrm{vol}$ and $\Delta$ is the Laplace--Beltrami operator.
On compact boundaryless manifolds such as the sphere $\mathbb{S}^{d-1}$, boundary terms vanish automatically.
We do not pursue the manifold theory here and instead point to classical references \citep{arnold_1989_mmc,marsden_ratiu_1999_mech_sym}.

\section{From Gradient Fields to Vector Fields}
\label{subsec:grad-to-vector}

So far, the main actor has been the score function $s(x)=\nabla_x\log p(x)$, a gradient field.
However, the goal of generation is not only to ``understand the shape of a distribution,'' but also to transport samples from a reference distribution $\pi$ to a target distribution $\rho$.
To this end, introduce a virtual time $t\in[0,1]$ and consider intermediate distributions $\{\rho_t\}$:
\[
\rho_0=\pi,\qquad \rho_1=\rho.
\]
Here $t$ is not physical time but rather a parameter indicating the progress of deformation.

\paragraph{Two routes: SDE and ODE}
There are two broad ways to move samples $X_t$.
\begin{itemize}
\item \textbf{Stochastic updates (SDE):} approach the target while mixing with noise (e.g., Langevin-type dynamics).
\item \textbf{Deterministic updates (ODE):} flow along a vector field $v_t$ (a transport flow).
\end{itemize}
This book focuses mainly on the latter, which leads to flow matching.

\paragraph{Arrows that move particles}
If particles move according to the ODE
\[
\frac{dX_t}{dt}=v_t(X_t),
\]
then the distribution $\rho_t$ follows the continuity equation\index{continuity equation}
\[
\partial_t\rho_t(x)+\nabla_x\cdot\{\rho_t(x)\,v_t(x)\}=0.
\]
Its meaning is a conservation law: density increases by inflow and decreases by outflow.
Note also that the continuity equation does not uniquely determine $v_t$ in general:
if $w_t$ satisfies $\nabla_x\cdot(\rho_t(x)\,w_t(x))=0$, then $v_t+w_t$ yields the same $\rho_t$.
This corresponds to a ``gauge freedom'' and is consistent with the divergence-free component in the Helmholtz decomposition.
We revisit the physical intuition and a proof from the viewpoint of stochastic processes in Chapter~3, Section~\ref{subsec:liouville-continuity}.

\paragraph{When is a vector field a gradient field?}
For a $C^1$ vector field $v$ on an open set $\Omega\subset\mathbb{R}^d$, a necessary condition for $v=\nabla\varphi$ is the local symmetry
\[
\frac{\partial v_i}{\partial x_j}=\frac{\partial v_j}{\partial x_i}
\qquad(i,j=1,\dots,d).
\]
If $\Omega$ is simply connected, this condition is also sufficient globally.
Thus scores (gradient fields) satisfy a stronger integrability constraint than general vector fields.

The distinction becomes clearer through the Helmholtz decomposition.
In $d=3$, any smooth vector field $v$ can be decomposed into the gradient of a scalar potential (curl-free part) and the curl of a vector potential (divergence-free part):
\begin{equation}
v = \underbrace{-\nabla \phi}_{\mathrm{curl}\text{-}\mathrm{free}} + \underbrace{\nabla \times A}_{\mathrm{divergence}\text{-}\mathrm{free}}.
\end{equation}
In general dimensions, this is understood as the Hodge decomposition (via differential forms).
By definition, the score function $s(x)=\nabla \log p(x)$ contains only the gradient component and has no rotational (curl) component.
The James--Stein shrinkage term $g$ is also curl-free (see Figure~\ref{fig:james-stein-field}).
In contrast, the velocity field $v_t(x)$ learned in flow matching can include rotational components.
This allows transport paths that ``swirl'' or ``twist,'' providing geometric flexibility in representing distributional transport.\index{velocity field}\index{flow matching}

\paragraph{Division of roles: scores and vector fields}
\begin{itemize}
\item The score $s_t(x)=\nabla_x\log\rho_t(x)$ describes the local shape (local gradient) of the intermediate distribution.
\item The vector field $v_t(x)$ specifies the law of motion that transports particles.
\end{itemize}
Diffusion models primarily learn scores, while flow matching primarily learns vector fields.
This is the key generalization from gradient fields to vector fields and serves as the starting point of the next chapter.

\subsubsection*{Guidance for readers}
This chapter develops the foundations of score estimation.
Chapters~3 and~4 provide a detailed account of flow matching and include discussions that may feel unfamiliar to readers trained in classical statistical inference.
For statisticians, it is not necessary to read the book strictly in order.
To quickly grasp how generative models enter statistical inference, we recommend reading Chapters~5 and~6 first,
where you will see concrete applications such as missing-data imputation, inverse problems, counterfactual distribution estimation, and policy evaluation.
Returning to Chapters~3 and~4 afterward often makes the discussions of vector fields and differential equations much more transparent,
by re-positioning them as computational principles for distribution estimation.

\chapter{Flow Matching: Mathematics and Computation}
\index{flow matching}
\index{ODE}
\index{SDE}
\index{optimal transport}
\index{vector field}

This chapter provides a unified view of distribution transformation as transport driven by a time-dependent vector field.
We begin by deriving the evolution equation of the density from particle dynamics (ODE/SDE), namely the continuity equation and the Fokker--Planck equation.
Next, we review optimal transport (OT) at the minimum level needed to use it as a geometric reference axis between distributions.
Then we introduce conditional flow matching (CFM), where a designed probability path turns vector-field learning into an $L^2$ regression problem.
Finally, we explain how to generate samples by numerically integrating the learned vector field, and how scores and velocity fields describe the same density evolution in different coordinates.
Throughout this chapter, time $t\in[0,1]$ is a \emph{virtual time} used to transport distributions, not a physical time.

\section{Stochastic processes and the continuity equation}
\index{SDE}
\index{Liouville equation}
\label{subsec:liouville-continuity}

When a random variable moves in continuous time, how does its \emph{population distribution} evolve?
For readers trained in statistics, differential equations may look intimidating at first,
but the essence is a simple conservation law:
\emph{when a cloud of particles flows, how does the density move?}
We start with a physical picture (fluid flow) and then formalize it with minimal notation.

\paragraph{Physical picture: a distribution as a ``cloud of particles''}
Think of $\rho_t(x)$ as the ``thickness'' of a cloud of particles at time $t$.
If many particles are near $x$, then $\rho_t(x)$ is large; if few are near $x$, it is small.
Suppose that at each location $x$ we have an arrow that may depend on time,
\[
v_t(x)\in\mathbb{R}^d,
\]
which tells a particle at $x$ where and how fast to move. This is the \emph{velocity field}.
In the language of fluid mechanics, $v_t$ is the flow velocity and $\rho_t$ is the density.

The key intuition is:
\begin{quote}
Particles neither appear nor disappear; they are only transported.
\end{quote}
If particles flow out of a region, the density there decreases; if particles flow in, the density increases.
Writing this ``in--out bookkeeping'' as an equation yields the \emph{continuity equation}.

\paragraph{Deterministic motion (ODE) and the continuity equation}
Consider noise-free deterministic dynamics. A particle follows the ODE
\begin{equation}
\frac{dX_t}{dt}=v_t(X_t),\qquad t\in[0,1].
\label{eq:ode-flow}
\end{equation}
Given an initial distribution $X_0\sim\rho_0$, the ODE \eqref{eq:ode-flow} determines $X_t$,
and we denote its density by $\rho_t$.

Then $\rho_t$ satisfies the PDE
\begin{equation}
\partial_t \rho_t(x)+\nabla_x\cdot\{\rho_t(x)\,v_t(x)\}=0.
\label{eq:continuity}
\end{equation}
This is the continuity equation\index{continuity equation}, also called the Liouville equation in many texts.
Although \eqref{eq:continuity} may look complicated, it simply states
\emph{time change of density equals minus the divergence of the flux}.
Indeed,
\[
\nabla_x\cdot\{\rho_t(x)\,v_t(x)\}
=
\sum_{j=1}^d \partial_{x_j}\bigl(\rho_t(x)\,v_{t,j}(x)\bigr),
\]
and the product $\rho_t v_t$ is the \emph{flux}: how much mass is transported per unit time in each direction.
The divergence $\nabla\cdot(\rho v)$ measures, near $x$, ``outflow minus inflow.''
If outflow dominates, the divergence is positive and $\partial_t\rho<0$ (density decreases).
If inflow dominates, the divergence is negative and density increases.
This local bookkeeping, repeated at every point, ensures global mass conservation:
no particle suddenly emerges or vanishes; each particle simply travels continuously along the flux $\rho_t v_t$.

\paragraph{Gauss' divergence theorem}
Let $\Omega\subset\mathbb{R}^d$ be a domain and let $D\subset\Omega$ be any subdomain with sufficiently smooth boundary $\partial D$.
Since $\rho_t(x)$ is mass per unit volume, the total mass in $D$ is
\[
M_D(t)=\int_D \rho_t(x)\,dx.
\]
If the particle flow is given by $v_t$, then the mass leaving $D$ per unit time through the boundary is
\[
\int_{\partial D} \rho_t(x)\,v_t(x)\cdot n(x)\,dS(x),
\]
where $n(x)$ is the outward unit normal and $dS(x)$ is the surface area element.
Thus the conservation law ``mass decrease inside equals outflow through the boundary'' is
\begin{equation}
\frac{d}{dt}\int_D \rho_t(x)\,dx
=
-\int_{\partial D} \rho_t(x)\,v_t(x)\cdot n(x)\,dS(x).
\label{eq:continuity_integral}
\end{equation}
By Gauss' divergence theorem,
\[
\int_{\partial D} \rho_t v_t\cdot n\,dS
=
\int_D \nabla_x\cdot(\rho_t v_t)\,dx,
\]
so \eqref{eq:continuity_integral} becomes
\begin{equation}
\frac{d}{dt}\int_D \rho_t(x)\,dx
=
-\int_D \nabla_x\cdot(\rho_t(x)\,v_t(x))\,dx.
\label{eq:continuity_weak}
\end{equation}
If $\rho_t$ is smooth enough so that we can interchange $t$-differentiation and integration, we obtain
\begin{equation}
\int_D \bigl\{\partial_t\rho_t(x)+\nabla_x\cdot(\rho_t(x)\,v_t(x))\bigr\}\,dx=0.
\label{eq:continuity_localization}
\end{equation}
Since $D$ is arbitrary, localization implies (almost everywhere)
\begin{equation}
\partial_t\rho_t(x)+\nabla_x\cdot(\rho_t(x)\,v_t(x))=0.
\label{eq:continuity1}
\end{equation}
Hence the continuity equation is simply the localization of mass conservation on an infinitesimal region.

\paragraph{Density update along trajectories}
Let $\Phi_t:x_0\mapsto x_t$ be the flow map of the ODE \eqref{eq:ode-flow}, i.e., $X_t=\Phi_t(X_0)$.\index{ODE}
Then the distribution is the pushforward\index{pushforward} 
\[
\rho_t=(\Phi_t)_{\#}\rho_0,
\]
and \eqref{eq:continuity1} expresses mass conservation.

Along a trajectory, the chain rule gives
\begin{equation}
\frac{d}{dt}\log\rho_t(X_t)
=
\partial_t\log\rho_t(X_t)
+
v_t(X_t)^\top \nabla_x\log\rho_t(X_t).
\label{eq:material-derivative-logrho}
\end{equation}
Here $\frac{d}{dt}$ denotes the total derivative along the path $t\mapsto X_t$.
For any smooth $f(t,x)$,
\[
\frac{d}{dt} f(t,X_t)
=
\partial_t f(t,X_t) + v_t(X_t)^\top \nabla_x f(t,X_t),
\]
where the right-hand side is the \emph{material derivative} (time derivative observed while riding on the particle).
If $\Phi_t$ is locally a diffeomorphism, combining \eqref{eq:material-derivative-logrho} with \eqref{eq:continuity1} yields the standard formula
\begin{equation}
\frac{d}{dt}\log \rho_t\!\left(X_t\right)
=
-\nabla_x\cdot v_t\!\left(X_t\right).
\label{eq:logdens_along_flow}
\end{equation}
Therefore, evaluating density along trajectories reduces to integrating divergence:
\[
\log \rho_1(X_1)
=
\log \rho_0(X_0)
-\int_0^1 \nabla\cdot v_t(X_t)\,dt.
\]
This viewpoint is identical to likelihood evaluation in continuous normalizing flows, and in our context it highlights the unifying picture
``distribution transformation $=$ learning a vector field.''

\paragraph{Stochastic dynamics (SDE):} It shows what changes when noise is present. \index{SDE}
Consider a continuous-time process with noise, such as Langevin dynamics:
\begin{equation}
dX_t = v_t(X_t)\,dt + \sqrt{2}\,dW_t,
\label{eq:sde}
\end{equation}
where $W_t$ is a $d$-dimensional Brownian motion (see, e.g., \citep{yanase2015probability}).
If $\rho_t$ denotes the density of $X_t$, then the density is not only transported by the flow but also \emph{diffused} (smoothed) by the noise.
As a result, $\rho_t$ satisfies the Fokker--Planck equation\index{Fokker--Planck equation}
\begin{equation}
\partial_t\rho_t(x)+\nabla_x\cdot\{\rho_t(x)\,v_t(x)\}
=
\Delta \rho_t(x),
\label{eq:fokker-planck}
\end{equation}
where
$
\Delta=\sum_{j=1}^d \partial_{x_j}^2
$
is the Laplacian.\index{ODE}
The left-hand side is the same transport (advection) term as in the ODE case, while the right-hand side $\Delta\rho_t$ represents diffusion (smoothing).
The coefficient $\sqrt{2}$ in \eqref{eq:sde} is chosen so that the generator becomes $v\cdot\nabla+\Delta$ and hence the diffusion term is exactly $\Delta\rho_t$.

\paragraph{Statistical viewpoint: ODE versus SDE}
Both ODEs and SDEs move distributions, but their qualitative behavior differs:
\begin{itemize}
\item ODE \eqref{eq:ode-flow}: particles flow \emph{deterministically}. The same initial point leads to the same path.
\item SDE \eqref{eq:sde}: particles flow \emph{stochastically}. Even with the same initial point, paths branch due to randomness.
\end{itemize}
In statistical terms, ODE-based methods learn a \emph{transport map} that pushes forward a base distribution,
whereas SDE-based methods realize a target distribution via Markov dynamics and diffusion.

\paragraph{Outlook: learning distribution transport}
Our goal is to learn a vector field $v_t$ that transports $\rho_0=\pi$ to $\rho_1=\rho$.
As preparation, we emphasized:
\begin{itemize}
\item given $v_t$, the distribution evolves according to \eqref{eq:continuity1} (ODE) or \eqref{eq:fokker-planck} (SDE);
\item thus the \emph{time evolution of a distribution} can be described by a \emph{local flow} (a velocity field).
\end{itemize}
In the following sections, we will show that flow matching formulates ``learn the velocity field'' as a regression problem.

\paragraph{The linear ODE/SDE case}
While our discussion targets general velocity fields $v_\theta(t,x)$, many aspects become transparent in the linear case.
For the linear ODE
\[
\dot x_t = A x_t,
\]
the flow map is given by the matrix exponential,
\[
x_t = e^{tA}x_0,\qquad T_t(x)=e^{tA}x,
\]
where $e^A=\sum_{k=0}^\infty A^k/k!$.
The Lipschitz constant is
\[
\mathrm{Lip}(T_t)=\|e^{tA}\|_{\mathrm{op}},
\]
and using the symmetric part $S=(A+A^\top)/2$ we have
\[
\|e^{tA}\|_{\mathrm{op}}\le \exp\{t\,\lambda_{\max}(S)\}.
\]
Hence $\lambda_{\max}(S)<0$ implies exponential contraction of trajectories, whereas $\lambda_{\max}(S)>0$ implies exponential error amplification,
providing an intuitive notion of ODE stability.
Moreover, since $\nabla\cdot(Ax)=\mathrm{tr}(A)$, the density $\rho_t$ satisfies the continuity equation
\[
\partial_t \rho_t(x)+\nabla_x\cdot\{\rho_t(x)\,A x\}=0,
\]
and the pushforward formula yields
\[
\rho_t(x)=\rho_0(e^{-tA}x)\,|\det(e^{-tA})|
=\rho_0(e^{-tA}x)\,\exp\{-t\,\mathrm{tr}(A)\}.
\]
In particular, if $\rho_0=\mathcal{N}(\mu_0,\Sigma_0)$, then Gaussianity is preserved and
\[
\rho_t=\mathcal{N}(\mu_t,\Sigma_t),\qquad
\mu_t=e^{tA}\mu_0,\qquad
\Sigma_t=e^{tA}\Sigma_0 e^{tA^\top}.
\]

For the linear SDE (Ornstein--Uhlenbeck type)
\[
dX_t = A X_t\,dt + B\,dW_t,\qquad D:=BB^\top,
\]
the density $\rho_t$ satisfies the Fokker--Planck equation
\[
\partial_t \rho_t(x)
=
-\nabla_x\cdot\{\rho_t(x)\,A x\}
+\frac12\sum_{i,j} D_{ij}\,\partial_{x_i}\partial_{x_j}\rho_t(x),
\]
(and for isotropic diffusion $B=gI$, the diffusion term is $(1/2)g^2\Delta\rho_t$).
The transition distribution is available in closed form,
\[
X_t\,|\,X_0=x_0 \sim \mathcal{N}\!\bigl(e^{tA}x_0,\ Q_t\bigr),\qquad
Q_t=\int_0^t e^{sA} D e^{sA^\top}\,ds,
\]
so if $\rho_0$ is Gaussian then $\rho_t$ is also Gaussian with mean and covariance governed by
\[
\dot\mu_t=A\mu_t,\qquad
\dot\Sigma_t=A\Sigma_t+\Sigma_tA^\top + D.
\]
Thus, in the linear case, Lipschitz bounds of the flow, stability (error amplification), and PDE evolution (continuity/Fokker--Planck) all fit together in a closed-form way.\index{SDE}

\section{Optimal transport}
\index{optimal transport}
\index{Wasserstein distance}
\label{subsec:ot-monge-kantorovich}

Optimal transport originated as the geometry of resource allocation,
and has evolved into a unified language that simultaneously provides
a distance between probability distributions and a notion of deformation/transport between them \citep{Villani2009}.
In statistics, Wasserstein distances naturally enter two-sample testing, density estimation, Bayesian computation, and robust inference,
because they capture shape differences beyond moments and include support mismatch \citep{panaretos2019statistical,goldfeld2024statistical}.
In machine learning, OT appears as a training objective or regularizer in generative models (WGAN, diffusion, flows),
and its dynamical formulation connects to stochastic processes, PDEs, and gradient flows (JKO).
Recent developments include entropic OT (Sinkhorn) for high-dimensional computation,
partial/unbalanced OT, and extensions to structured spaces (groups, graphs, manifolds) \citep{cuturi2013sinkhorn,arjovsky2017wgan,peyre2019computational,Sato2023}.

In this section we define, mathematically, the ``least wasteful'' way to move probability mass from one distribution to another.
Although OT is now ubiquitous in statistics and machine learning,
it can be surprisingly hard to keep track of ``what is optimized'' and ``what are the variables.''
We build the formulation carefully from a physical metaphor (moving piles of sand).

\paragraph{Picture: a probability distribution as a mass distribution}
View $\pi$ and $\rho$ as two mass distributions over $\mathbb{R}^d$.
(We use the same symbols for measures and densities; when needed we write $\pi(x)$ and $\rho(x)$.)
Think of $\pi$ as the initial pile of sand and $\rho$ as the desired final pile.
Since mass is conserved, the total masses agree (both equal to $1$ for probability measures).
The problem is to decide \emph{which grains go where}.

\paragraph{Cost design: pay according to transport distance}
To measure inefficiency, specify a cost for moving unit mass from $x$ to $y$:
\[
c(x,y)\ge 0.
\]
A standard choice is
\[
c(x,y)=\|x-y\|^p \quad (p\ge 1),
\]
and the squared distance ($p=2$) plays a central role in analysis and geometry.

\paragraph{The Monge problem: a transport map (deterministic assignment)}
The most direct viewpoint is to seek a \emph{transport map} $T:\mathbb{R}^d\to\mathbb{R}^d$
that sends $x$ to $T(x)$.
The constraint ``pushing forward $\pi$ by $T$ yields $\rho$'' is written
\[
T_{\#}\pi=\rho,
\]
meaning that if $X\sim\pi$ then $T(X)\sim\rho$.\index{pushforward}
Under this constraint, Monge's problem is
\begin{equation}
\inf_{T:\,T_{\#}\pi=\rho}\ \int_{\mathbb{R}^d} c\bigl(x,T(x)\bigr)\,d\pi(x).
\label{eq:monge}
\end{equation}
This formulation is elegant but analytically inconvenient, and $T$ may fail to exist.
If $\pi$ is discrete and $\rho$ is continuous, or if mass needs to split (e.g., one point in $\pi$ must be sent to two points in $\rho$),
a deterministic map is too restrictive.

\paragraph{Concrete pushforward formula}
While $T_{\#}\pi=\rho$ is an abstract constraint, it becomes explicit if $T$ is smooth and invertible.
Assume $\pi$ and $\rho$ have densities $\pi(x)$ and $\rho(y)$, and $T$ is a $C^1$ diffeomorphism.
Then for $Y=T(X)$ with $X\sim\pi$, the change-of-variables formula gives
\begin{equation}
\rho(y)=\pi\!\left(T^{-1}(y)\right)\,
\left|\det \nabla T^{-1}(y)\right|.
\label{eq:pushforward-inverse}
\end{equation}
Equivalently, with $y=T(x)$,
\begin{equation}
\rho\!\left(T(x)\right)\,
\left|\det \nabla T(x)\right|
=
\pi(x).
\label{eq:pushforward-forward}
\end{equation}
Here $\nabla T(x)$ is the Jacobian matrix and $|\det\nabla T(x)|$ quantifies how volumes expand or contract under $T$.
Intuitively, a small volume element $dx$ near $x$ is mapped to $dy\approx |\det\nabla T(x)|\,dx$,
so imposing mass conservation $\pi(x)\,dx=\rho(y)\,dy$ yields \eqref{eq:pushforward-forward}.
This formula is fundamental for understanding distribution transformation by flows, including the ODE-based maps that appear later.

\paragraph{The Kantorovich problem}
Kantorovich (1940s) relaxed the deterministic assignment and represented transport as a \emph{plan}.
Let $\gamma$ be a probability measure on $\mathbb{R}^d\times\mathbb{R}^d$,
interpreted as ``how much mass is transported from $x$ to $y$.''
Let
\[
\gamma\in\Gamma(\pi,\rho)
\]
denote that $\gamma$ has marginals $\pi$ and $\rho$, i.e.,
\[
\int \gamma(x,y)\,dy=\pi(x),\qquad  \int \gamma(x,y)\,dx=\rho(y).
\]
Then the Kantorovich problem is
\begin{equation}
\inf_{\gamma\in\Gamma(\pi,\rho)}\ \int_{\mathbb{R}^d\times\mathbb{R}^d} c(x,y)\,d\gamma(x,y).
\label{eq:kantorovich}
\end{equation}
Because $\gamma$ allows splitting mass, existence and convexity are ensured, and this becomes the foundation for theory and computation.
Monge optimizes over \emph{maps} $T$; Kantorovich optimizes over \emph{couplings} $\gamma$.\index{coupling}
Keeping track of what the optimization variables are is essential.

\paragraph{Wasserstein distance}
For $c(x,y)=\|x-y\|^p$, define
\begin{equation}
W_p(\pi,\rho)
=
\left(
\inf_{\gamma\in\Gamma(\pi,\rho)}
\int \|x-y\|^p\,d\gamma(x,y)
\right)^{1/p}.
\label{eq:wasserstein}
\end{equation}
This is the $p$-Wasserstein distance.
Whereas KL divergence measures pointwise density mismatch, $W_p$ measures the geometric cost of \emph{moving mass in space}.
Consequently, if supports are slightly translated, Wasserstein remains finite and meaningful, while KL may even become infinite.

{Why the squared distance is special?}
In the $p=2$ case ($W_2$), the geometry is especially rich.
Under suitable conditions, if an optimal transport map $T$ exists, then it can be written as
\[
T(x)=\nabla \varphi(x)
\]
for a convex potential $\varphi$.
A standard proof sketch is: optimal plans have \emph{cyclically monotone} support (a property specific to the squared cost),
which by Rockafellar's theorem implies that the support lies in the graph of the subgradient of a convex function.
If the map is single-valued, one obtains $T(x)=\nabla\psi(x)$ $\pi$-a.e.
For our purposes, the key message is not the technical conditions but the geometric intuition:
\emph{OT is inherently geometric}.
This also connects naturally to our earlier distinction between gradient fields (scores) and general vector fields (flows).

\paragraph{Dynamic formulation {\rm (the continuity equation and transport energy)}}
In Section~\ref{subsec:liouville-continuity}, we saw that a velocity field $v_t$ moves a distribution via the continuity equation
\[
\partial_t\rho_t+\nabla\cdot(\rho_t v_t)=0,\qquad \rho_0=\pi,\ \rho_1=\rho.
\]
OT can be expressed as selecting, among all admissible ways of moving distributions,
the one with minimal total cost of motion.
For $W_2$, the Benamou--Brenier dynamic formulation states \citep{benamou2000computational,villani2009optimal}:
\begin{align}\nonumber
W_2^2(\pi,\rho)
=
\inf_{\{\rho_t,v_t\}}\Big\{&
\int_0^1\int_{\mathbb{R}^d} \|v_t(x)\|^2\,\rho_t(x)\,dx\,dt :\\[3mm]
&\;\partial_t\rho_t+\nabla\cdot(\rho_t v_t)=0,
\ \rho_0=\pi,\ \rho_1=\rho\Big\}.
\label{eq:benamou-brenier}
\end{align}
The objective is the time integral of squared speed weighted by density, analogous to kinetic energy.
Thus $W_2$ is the distance obtained by transporting a distribution with the least kinetic energy.
This formulation tightly connects OT to the Liouville/continuity equation.

A full proof of \eqref{eq:benamou-brenier} is beyond our scope.
Since the argument is technical, it can be skipped on a first reading; we only give a sketch.

\smallskip
\noindent
\textbf{(i) Static $\Rightarrow$ dynamic ($\le$).}
Take any transport plan $\gamma\in\Gamma(\pi,\rho)$ and consider the linear interpolation
\[
X_t=(1-t)x+ty\qquad (x,y)\sim\gamma.
\]
Let $\rho_t:=(X_t)_\#\gamma$ and define $V(x,y):=y-x$.
Assuming $\gamma$ is sufficiently regular such that the conditional expectation is well-defined, for any test function $\phi\in C_c^\infty(\mathbb{R}^d)$,
\begin{align*}
\frac{d}{dt}\int \phi(z),d\rho_t(z)
&=
\frac{d}{dt}\int \phi\bigl((1-t)x+ty\bigr),d\gamma(x,y)\\[2mm]
&=
\int \nabla\phi(X_t)\cdot (y-x),d\gamma\\[2mm]
&=
\int \nabla\phi(z)\cdot v_t(z),d\rho_t(z),
\end{align*}
where one can take a velocity field $v_t(z)=\mathbb{E}[y-x\mid X_t=z]$.\index{velocity field}
Integration by parts then yields $\partial_t\rho_t+\nabla\cdot(\rho_t v_t)=0$.
Moreover, the action (kinetic energy) along this interpolated path is upper bounded by
\(
\int \|x-y\|^2\,d\gamma.
\)
Choosing $\gamma$ as an optimal plan gives
\[
\inf_{(\rho,v)}\int_0^1\int \|v_t\|^2\,\rho_t
\le
W_2^2(\pi,\rho).
\]

\smallskip
\noindent
\textbf{(ii) Dynamic $\Rightarrow$ static ($\ge$).}
Conversely, assume a smooth pair $(\rho_t,v_t)$ satisfies $\partial_t\rho_t+\nabla\cdot(\rho_t v_t)=0$.
Define the flow map $X_t$ by
\[
\dot X_t(x)=v_t(X_t(x)),\qquad X_0(x)=x.
\]
For smooth solutions, the continuity equation implies $\rho_t=(X_t)_\#\rho_0$.
Define the transport plan
\[
\gamma:=(\mathrm{id},X_1)_\#\rho_0,
\]
so $\gamma\in\Gamma(\pi,\rho)$.
By Cauchy--Schwarz,
\[
\|X_1(x)-x\|^2
=
\Bigl\|\int_0^1 \dot X_t(x)\,dt\Bigr\|^2
\le
\int_0^1 \|\dot X_t(x)\|^2\,dt
=
\int_0^1 \|v_t(X_t(x))\|^2\,dt.
\]
Integrating with respect to $\pi$ gives
\[
\int \|x-y\|^2\,d\gamma(x,y)
=
\int \|X_1(x)-x\|^2\,d\pi(x)
\le
\int_0^1\int \|v_t(z)\|^2\,d\rho_t(z)\,dt.
\]
Since the left-hand side is at least $W_2^2(\pi,\rho)$, we obtain
\[
W_2^2(\pi,\rho)
\le
\int_0^1\int \|v_t\|^2\,d\rho_t\,dt
\]
for any admissible $(\rho,v)$, hence
\[
W_2^2(\pi,\rho)
\le
\inf_{(\rho_t,v_t)}\int_0^1\int \|v_t\|^2\,\rho_t\,dt.
\]

\smallskip
\noindent
Combining (i) and (ii) yields \eqref{eq:benamou-brenier}.
In nonsmooth settings, one can proceed via approximation and weak convergence, or via the superposition principle.

\paragraph{Relation to flow matching: what does the learned vector field approximate?}\index{flow matching}
Flow matching learns a time-dependent vector field $v_t(x)$ that transports $\pi$ to $\rho$ from data.
Importantly, flow matching is not automatically optimal in the OT sense of \eqref{eq:benamou-brenier}.
Typically, flow matching proceeds by:
\begin{itemize}
\item first specifying an intermediate path $\{\rho_t\}_{t\in[0,1]}$ (choosing the ``route''),
\item then estimating a compatible velocity field $v_t$ by regression.
\end{itemize}
If the learning objective is designed to be consistent with an energy minimization of the form \eqref{eq:benamou-brenier},
the learned $v_t$ can be interpreted as approximating an ``OT-efficient'' transport.
In that case, the flow can be viewed as a distributional map aligned with Wasserstein geometry.\index{Wasserstein distance}

OT can feel difficult because the optimization variables are not finite-dimensional parameters but infinite-dimensional objects,
\[
\text{a map }T,\quad \text{a coupling }\gamma,\quad \text{or a velocity field }v_t.
\]
Once the physical picture of mass conservation and kinetic energy is internalized, however,
the continuity-equation viewpoint becomes natural.
In the next section, we use this OT-guided intuition to introduce the concrete objectives used in conditional flow matching.

\section{Conditional flow matching}
\index{conditional flow matching}
\label{subsec:probpath-cfm}

We now arrive at the point where the ``learning'' face of flow matching becomes clear \citep{lipman2023flowmatching}.
So far we have seen that distribution transport is governed by the continuity equation \eqref{eq:continuity}.
However, directly estimating $v_t$ from \eqref{eq:continuity} would require evaluating and integrating $\rho_t$ at each time,
which is statistically and computationally inconvenient.
The key trick is to work \emph{conditionally}, which makes vector-field regression easier
by (i) reducing ambiguity in multimodal settings, (ii) reducing the variance of the learning signal,
and (iii) improving coverage of thin regions.

Thus, CFM is best viewed not as increasing representational power, but as improving the \emph{conditioning} of the learning problem under the same expressivity,
making regression easier in terms of multimodality, variance, and coverage.
This viewpoint also aligns with later statistical tasks such as missing-data imputation and counterfactual generation,
which are inherently problems of sampling from conditional distributions.

\paragraph{The big picture: what is unknown, and what do we design?}\index{vector field}
Our goal is to learn a time-dependent vector field $v_t(x)$ that realizes transport
\[
X_0\sim \pi  \qquad{\Longrightarrow}\qquad X_1\sim \rho,
\]
where
\begin{itemize}
\item $\pi$: a tractable base distribution for random numbers (known),
\item $\rho$: the data distribution (unknown, but we have samples),
\item $v_t$: the velocity field to be learned (unknown).
\end{itemize}
A crucial design variable is an intermediate path
\(
\{\rho_t\}_{t\in[0,1]},
\)
called a \textbf{probability path}.
As long as $\rho_0=\pi$ and $\rho_1=\rho$, the path is not unique:
there are many ways to move one distribution to the other.
Flow matching exploits this freedom to stabilize learning.

\paragraph{Designing a probability path: intuitive interpolation}
A familiar choice is linear interpolation.
Sample a noise point $X_0\sim\pi$ and a data point $X_1\sim\rho$, and set
\begin{align}\label{line}
X_t = (1-t)X_0+tX_1.
\end{align}
Then each point moves along a straight line as $t$ increases.
The distribution of $X_t$ defines $\rho_t$.
Importantly, once an \emph{interpolation rule for samples} is specified,
many quantities needed for learning (notably the target velocity defined below) can be computed without explicitly evaluating $\rho_t$.

In practice, one often mixes noise into the interpolation, or adjusts variance as a function of $t$ to improve stability,
especially to connect with diffusion-model paths.
The key point is:
\begin{quote}
The goal is not to \emph{evaluate} formulas for $\rho_t$, but to \emph{construct} a sampling rule that defines $\rho_t$.
\end{quote}

\paragraph{Conditional paths: fix $X_1$}
A second key idea is to treat intermediate distributions conditionally on the data endpoint.
Sample $X_1\sim\rho$ and fix $X_1=x_1$, then consider a conditional path
\[
X_t \sim \rho_t(\cdot\,|\,x_1).
\]
For linear interpolation \eqref{line}, conditional on $x_1$, we have $X_0\sim\pi$ and $X_t$ becomes an affine transform of $\pi$ (shift and scaling).
This ``conditioning on $x_1$'' will greatly simplify the learning objective.

\paragraph{Learning a velocity field by regression: conditional flow matching (CFM)}\index{conditional flow matching}
Intuitively, the ``true velocity'' of a particle at time $t$ is determined by the sampling rule used to create the path.
Under conditioning, this velocity can often be computed as a \emph{target velocity}
\[
u_t(x\,|\,x_1),
\]
and we fit $v_t(x)$ to match it in least squares:
\begin{equation}
\min_{\theta}\ 
\mathbb{E}\left[
\left\|v_\theta(t,X_t)-u_t(X_t\,|\,X_1)\right\|^2
\right].
\label{eq:cfm-pop}
\end{equation}
The expectation is over $t\sim\mathrm{Unif}[0,1]$, $X_1\sim\rho$, and $X_t\sim\rho_t(\cdot\,|\,X_1)$.
Thus flow matching training is essentially a \emph{nonlinear least squares} problem:
\begin{itemize}
\item regressors: $(t,X_t)$,
\item response: target velocity $u_t(X_t\,|\,X_1)$,
\item regression function: $v_\theta(t,x)$.
\end{itemize}
In statistical language, this is an instance of $M$-estimation with a squared loss.

\paragraph{Why conditioning helps: advantages from the regression viewpoint}
In \eqref{eq:cfm-pop}, conditioning on $X_1$ reduces learning difficulty in three ways:
\begin{itemize}
\item[(i)] {Less blur from multimodality:}
if the desired mode is ambiguous given $(t,x)$, conditioning on $X_1$ resolves that ambiguity and yields a more consistent teacher signal.
\item[(ii)] {Lower conditional variance of the teacher signal:}
more information reduces the variability of $u_t$, which tends to reduce the variance of minibatch gradients and stabilize training.
\item[(iii)] {A global problem becomes many local tasks:}
sampling the endpoint $X_1$ makes each minibatch a local regression near that endpoint, which helps cover low-density regions.
\end{itemize}

\paragraph{Empirical $L^2$ loss}
Given data $x_1,\ldots,x_n\overset{\mathrm{iid}}{\sim}\rho$, approximate $\rho$ by the empirical measure
\[
\widehat\rho_n=\frac1n\sum_{i=1}^n\delta_{x_i}.
\]
Then the empirical objective corresponding to \eqref{eq:cfm-pop} is
\begin{align}
\widehat{\mathcal{L}}_n(\theta)
&=
\mathbb{E}_{t\sim\mathrm{Unif}[0,1]}
\Biggl[
\frac1n\sum_{i=1}^n
\mathbb{E}_{X_t\sim\rho_t(\cdot\,|\,x_i)}
\left\|
v_\theta(t,X_t)-u_t(X_t\,|\,x_i)
\right\|^2
\Biggr].
\label{eq:cfm-erm}
\end{align}
This is ordinary empirical risk minimization for the regression problem with inputs $(t,X_t,x_i)$ and target $u_t(X_t\,|\,x_i)$.

For each $i$, draw $t_{ik}\sim\mathrm{Unif}[0,1]$ and represent conditional sampling $X_t\sim\rho_t(\cdot\,|\,x_i)$ by
$X_{t_{ik}}^{(ik)}=T(t_{ik},x_i,\xi_{ik})$ using random seeds $\xi_{ik}$.
Then \eqref{eq:cfm-erm} can be approximated by
\begin{equation}
\widehat{\mathcal{L}}_{n,K}(\theta)
=
\frac1{nK}\sum_{i=1}^n\sum_{k=1}^K
\left\|
v_\theta\!\left(t_{ik},X_{t_{ik}}^{(ik)}\right)
-
u_{t_{ik}}\!\left(X_{t_{ik}}^{(ik)}\,|\,x_i\right)
\right\|^2.
\label{eq:cfm-empirical}
\end{equation}
In practice, one often sets $K=1$ and runs minibatch SGD.

\paragraph{Why integrals disappear}
The biggest advantage is that we never need to evaluate $\rho_t$ (or its normalizing constants) inside training.
We only need:
\begin{itemize}
\item sample $x_1$ from data,
\item sample $x_0$ from the base distribution,
\item compute $x_t$ and the target velocity $u_t$ by algebra according to the designed path.
\end{itemize}
In other words, we do not ``compute densities''; we ``generate samples and regress.''
This shift is a major reason why deep learning based training can be stable in practice,
often enjoying computational advantages over likelihood-based methods such as continuous normalizing flows that rely on \eqref{eq:logdens_along_flow}.

\paragraph{Target velocity for the linear interpolation path: closed form}
For the linear interpolation \eqref{line} with $X_0\sim\pi$ and $X_1\sim\rho$, the time derivative is
\begin{equation}
\dot{X}_t=\frac{d}{dt}X_t = X_1-X_0.
\label{eq:linear_target_v_pair}
\end{equation}
Thus the teacher signal is simply the difference vector once a pair $(X_0,X_1)$ is given.
Moreover, for $t<1$, we can solve $X_0=(X_t-tX_1)/(1-t)$, so \eqref{eq:linear_target_v_pair} can be written as a function of $(t,X_t,X_1)$:
\begin{equation}
\dot{X}_t
=
\frac{X_1-X_t}{1-t}
\qquad (t<1).
\label{eq:linear_target_v_cond}
\end{equation}
This makes explicit that we can compute the teacher velocity from input $(t,X_t)$ and conditioning $X_1$.
Since \eqref{eq:linear_target_v_cond} can become large when $t\approx 1$, stabilization tricks include adding noise to the path or changing the sampling distribution of $t$.
The core point remains: CFM closes the learning loop using only \emph{sample generation + regression}, without density evaluation.

\paragraph{Why use neural networks? as flexible nonparametric regression}
Representing $v_\theta(t,x)$ by a neural network can be viewed as performing fast flexible nonparametric regression.
Because \eqref{eq:cfm-pop} uses squared loss, the population minimizer is (under regularity conditions) a conditional mean:
\[
v^\ast(t,x)=\mathbb{E}\left[u_t(X_t\,|\,X_1)\,|\,X_t=x\right].
\]
Neural networks can then be interpreted as approximators of this conditional mean function,
which reduces the ``black-box'' feel and reconnects to the statistical intuition ``deep learning = function approximation.''

\paragraph{What is ``conditional'' here?}
The conditioning here is not a causal conditional independence concept;
it simply means ``fix the data endpoint $x_1$ and define an interpolation path conditioned on it.''
The main benefit is technical: it makes the target velocity computable and closes the regression problem.

\paragraph{Choosing a coupling: variance reduction in learning}
In CFM, how we pair $(x_0,x_1)$ within a minibatch (the coupling) influences training stability through the variance of the teacher signal $u_t$.\index{stability}
With independent coupling (random pairing), even nearby $x_t$ may be associated with unrelated endpoints $x_1$,
increasing conditional variance in regression.

\medskip
\noindent\textbf{A concrete OT coupling (minibatch version).}
Let the minibatch size be $m$.
Prepare $\{x_0^{(j)}\}_{j=1}^m\sim\pi$ and data points $\{x_1^{(i)}\}_{i=1}^m$.
Define the cost matrix
\[
C_{ij}=\|x_0^{(j)}-x_1^{(i)}\|^2,
\]
and choose the optimal assignment (a permutation) $\sigma\in S_m$ by
\[
\sigma^\ast=\argmin_{\sigma\in S_m}\sum_{i=1}^m C_{i,\sigma(i)}.
\]
Then pair $(x_0^{(\sigma^\ast(i))},x_1^{(i)})$.
For the linear bridge, the teacher signal is roughly $u_t\simeq x_1-x_0$,
so OT pairing tends to make $\|x_1-x_0\|$ small and reduces the variability of $u_t$ among nearby $x_t$.

\paragraph{Practical approximation (Sinkhorn)}
Instead of exact assignment, one can solve an entropically regularized OT problem:
\[
\min_{P\in\mathbb{R}_+^{m\times m}}
\sum_{i,j}P_{ij}C_{ij}
+\varepsilon\sum_{i,j}P_{ij}(\log P_{ij}-1)
\quad \text{s.t.}\quad
P{1}=\tfrac1m {1},\ P^\top {1}=\tfrac1m {1}.
\]
Here $P=(P_{ij})$ is a coupling (transport plan) between the two empirical distributions
\[
\frac1m\sum_{i=1}^m\delta_{x_1^{(i)}}\quad \text{and}\quad \frac1m\sum_{j=1}^m\delta_{x_0^{(j)}},
\]
and $P_{ij}\ge0$ indicates the strength of matching $x_1^{(i)}$ with $x_0^{(j)}$.
The constraints enforce uniform marginals (each row and column sums to $1/m$), so $mP_{ij}$ can be interpreted as a conditional probability $\mathbb{P}(J=j\mid I=i)$.
Sampling pairs $(i,j)$ according to $P$ yields a randomized coupling.

OT coupling tends to match nearby points and functions as a \emph{variance reduction} device by stabilizing the teacher signal.
This effect is systematized in \citet{tong2023improving}.
For implementation details and code examples, see \citet{lipman2024flowmatchingguide}.

\begin{figure}[!tbp]
  \centering
  \IfFileExists{independent_coupling.jpg}{%
    \includegraphics[width=0.7\linewidth]{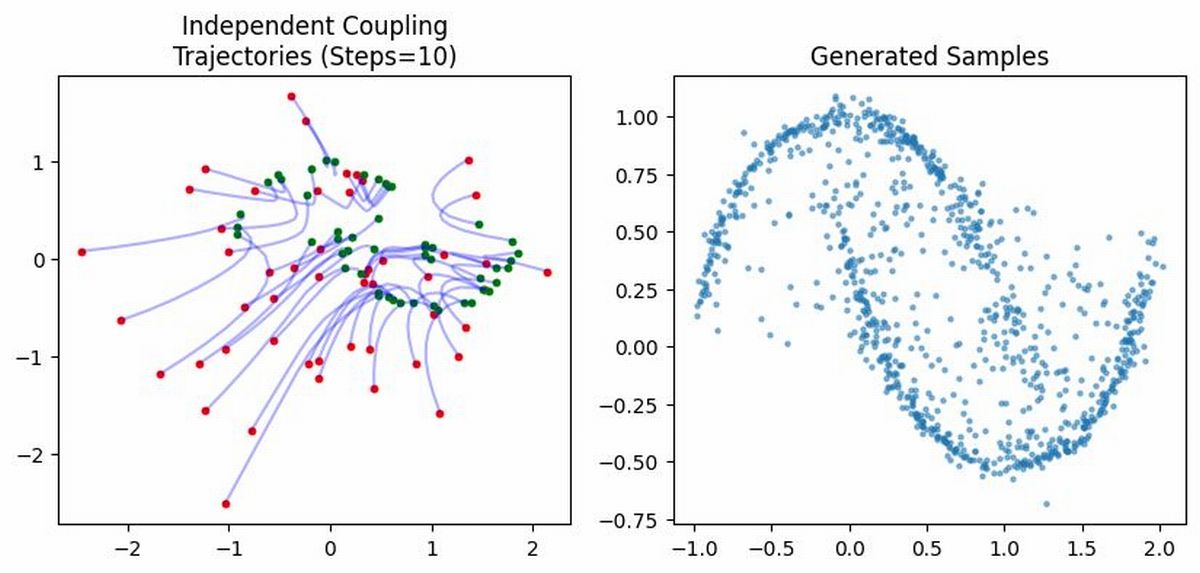}\\
    {\small (a) Independent coupling}
  }{%
    {\small (a) Omitted because independent\_coupling.jpg is not found.}
  }

  \vspace{1.5em}

  \IfFileExists{ot_coupling.jpg}{%
    \includegraphics[width=0.7\linewidth]{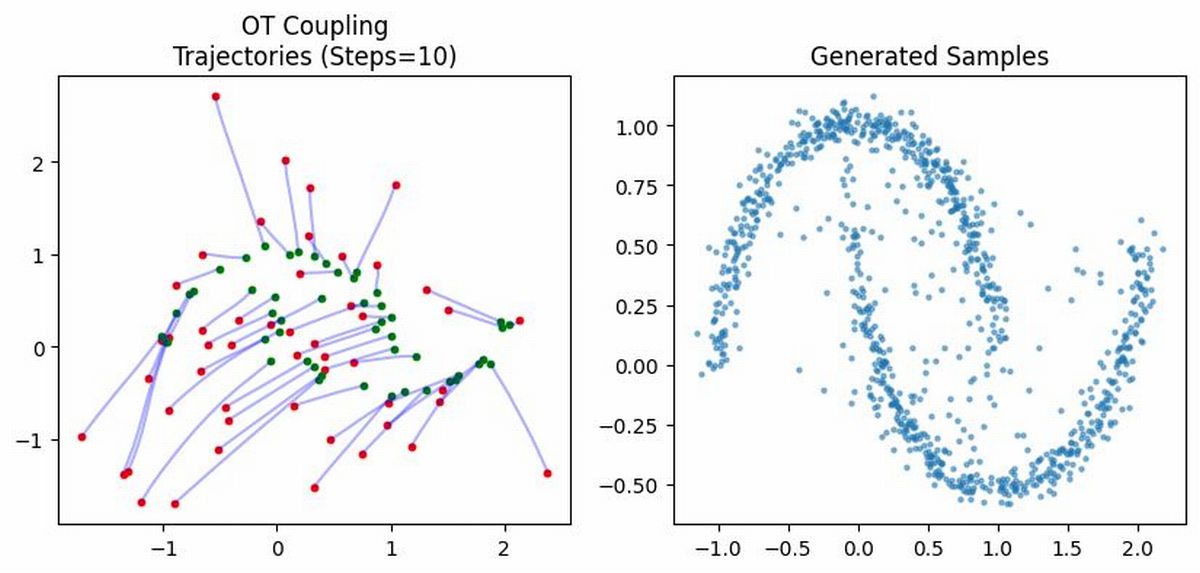}\\
    {\small (b) OT coupling}
  }{%
    {\small (b) Omitted because ot\_coupling.jpg is not found.}
  }

  \caption{Pairings induced by different couplings (illustration).}
  \label{fig:fm_coupling_compare}
\end{figure}

\paragraph{Outlook: which probability path should we choose?}
We have seen that flow matching reduces to regression.
The next question is: which conditional path $\rho_t(\cdot\,|\,x_1)$ yields stable learning and good generation?
Several design principles exist \citep{sohl2015deep,ho2020ddpm,song2021scorebased}:
\begin{itemize}
\item linear-interpolation paths (transport is intuitive),
\item noise-injected interpolation (smoothing for learning),
\item diffusion-consistent paths (connection to score learning).
\end{itemize}
The choice also affects how we interpret the learned $v_t$ (more OT-like or more diffusion-like).
We now explain how to generate samples once a vector field has been learned.

\section{Deterministic sampling via ODE integration}
\index{ODE}
\index{ODE solver}
\label{subsec:ode-sampling}

We now describe how to generate data points using the learned time-dependent vector field $v_t(x)$.
Whereas diffusion models are based on stochastic differential equations (SDEs),
flow matching typically generates samples by solving an ordinary differential equation (ODE) \citep{albergo2025stochastic}.\index{flow matching}
For statisticians, ODEs may feel unfamiliar, but what we need is only the intuition:
``advance in small time steps along the arrows.''
No advanced analysis is required for the purposes of this chapter \citep{chen2023geom}.

\paragraph{The generative backbone: from noise to data by flowing along arrows}
After training, we obtain an estimated vector field $\widehat v(t,x)$.
Generation is essentially two lines:
\begin{enumerate}
\item Sample an initial point $X_0\sim \pi$ (e.g., standard normal).
\item Solve the ODE
\begin{equation}
\frac{dX_t}{dt}=\widehat v(t,X_t),\qquad t\in[0,1],
\label{eq:gen-ode}
\end{equation}
from $t=0$ to $t=1$, and output $X_1$.
\end{enumerate}
If learning is successful, $X_1$ follows (approximately) the data distribution $\rho$.
The essence is simply ``assign a vector field to noise points and transport them by the flow.''

\paragraph{What is an ODE solver? approximating continuous time by discrete steps}\index{ODE}
Closed-form solutions of ODEs are rare, so we approximate numerically by discretizing time.
Modern ODE solvers implement high-order Runge--Kutta methods with adaptive step sizes,
implicit methods for stiffness, event handling, and more,
and have become ubiquitous as general-purpose engines for high-accuracy and safe integration.
Combined with automatic differentiation, they also enable sensitivity (gradient) computation and play a central role in Neural ODEs and continuous normalizing flows \citep{chen2018neuralode}.
In Python, \texttt{scipy.integrate.solve\_ivp} is standard; \texttt{diffrax} (JAX) and \texttt{torchdiffeq} (PyTorch) are also widely used.

The simplest scheme is the forward Euler method:
for step size $\Delta t$,
\[
X_{t+\Delta t}\approx X_t+\Delta t\,\widehat v(t,X_t).
\]
This is literally ``look at the arrow and take one step.''
More accurate solvers can be understood as smarter ways to perform this step to improve accuracy and stability.
In large-scale probabilistic and generative models, continuous-time representations have become practical,
and a common workflow is ``let the solver automatically trade accuracy for computation.''

\paragraph{Implication of determinism: generation is a map}
Since \eqref{eq:gen-ode} contains no injected noise, specifying $X_0$ uniquely determines the entire path $X_t$.
Thus we define a (numerical) map $\widehat T$ by
\[
X_1 = \widehat T (X_0).
\]
This is what it means that ODE-based generation is a deterministic transformation.
This property offers several advantages for statistical applications.

\paragraph{Advantage 1: inversion becomes tangible}
Solving the same ODE backward in time (integrating from $t=1$ to $t=0$) maps a data point $X_1$ back to the noise side.
That is, we can numerically construct an inverse for $\widehat T$.
This contrasts with SDE-based generation (diffusion models), where paths branch stochastically.
Having a usable inverse is directly useful for
\begin{itemize}
\item clustering and outlier diagnosis in latent space,
\item stability of representations for the same data point,
\item controllability in conditional generation (via label/covariate embeddings).
\end{itemize}

\paragraph{Advantage 2: the latent variable $X_0$ has a clear role}
Although $X_0$ is just a random input, the deterministic map $\widehat T$ associates it uniquely to $X_1$.
Thus $X_0$ can be treated as a coordinate of generation.
Statistically, it plays a role analogous to a latent factor in latent-variable models,
and (depending on the chosen dimension) may serve as a low-dimensional representation explaining $X_1$.
While interpretability is model-dependent, determinism helps stabilize meaning in latent space.

\paragraph{Advantage 3: computational efficiency (often lighter than SDE simulation)}
Diffusion-model generation often simulates an SDE (or an associated reverse-time process) with many small steps.
ODE simulation is simpler because it lacks noise injection, and higher-order solvers may require fewer steps.
Hence generation can often be faster.
That said, the required number of steps depends on smoothness of the vector field and training quality,
so it is more accurate to say ``ODE generation often has room to be made fast'' rather than ``ODE is always faster.''

\paragraph{SDE (diffusion) versus ODE: what is fundamentally different?}
The difference is ``noise or no noise,'' but the implications are substantial:
\begin{itemize}
\item \textbf{SDE:} includes diffusion (smoothing) and explores via randomness.
In principle it can reach complex distributions, but generation is stochastic and often requires many steps.
\item \textbf{ODE:} transports as a flow.
Generation is deterministic and easier to control; inversion and latent representations can be advantageous,
but poor vector fields may yield unstable flows (numerically, trajectories may fold or amplify errors).
\end{itemize}
From a statistical perspective, ODE advantages are pronounced when one wants the same output from the same input, or when one wants to analyze data via an inverse map.

\paragraph{Practical note: solver choice and step size}
An ODE solver trades off accuracy and computation.
Large steps are fast but inaccurate; small steps are accurate but slow.
In practice, one often:
\begin{itemize}
\item starts with a coarse discretization and increases steps if sample quality is insufficient, or
\item uses adaptive-step solvers that automatically adjust step size based on error estimates.
\end{itemize}
The point is that an ODE solver is not magical; it is simply a computational procedure for moving along an arrow field.

\begin{algorithm}[ht]
\caption{Deterministic sampling via ODE integration (flow matching)}
\label{alg:ode-sampling}
\begin{algorithmic}[1]
\Require Trained vector field $\widehat v(t,x)$, base distribution $\pi$, time interval $[0,1]$, number of steps $K$ (or error tolerance $\varepsilon$)
\Ensure Generated sample $x_1$
\State $x_0 \sim \pi$ \Comment{e.g., $x_0\sim N(0,I_d)$}
\State Prepare a time grid: $0=t_0<t_1<\cdots<t_K=1$ (e.g., $t_k=k/K$)
\State $x \leftarrow x_0$
\For{$k=0,1,\dots,K-1$}
  \State $\Delta t \leftarrow t_{k+1}-t_k$
  \State $x \leftarrow x + \Delta t\, \widehat v(t_k, x)$
  \Comment{forward Euler (minimal form)}
  \State \Comment{replace this line by a Runge--Kutta update for higher accuracy}
\EndFor
\State $x_1 \leftarrow x$
\State \Return $x_1$
\end{algorithmic}
\end{algorithm}

\section{Relationship between scores and velocity fields}
\index{score function}
\index{velocity field}
\label{subsec:score-velocity-bridge}

We summarize how the score field (a gradient field) learned by score matching
and the velocity field (a general vector field) learned by flow matching
are connected through continuous-time distribution equations.
The point here is not to introduce diffusion models as a specific model class;
rather, we emphasize a general structure:
\emph{stochastic (noisy) dynamics and deterministic flows can describe the same evolution of marginal densities}.

\paragraph{Two descriptions: particle dynamics and distribution dynamics}
Let $X_t\in\mathbb{R}^d$ be a particle (sample) evolving in continuous time, with density $\rho_t(x)$.
There are two representative descriptions:

\begin{itemize}
\item \textbf{Deterministic flow (ODE):}
\begin{equation}
\frac{dX_t}{dt}=v_t(X_t).
\label{eq:bridge-ode}
\end{equation}
Then $\rho_t$ satisfies the continuity equation
\begin{equation}
\partial_t\rho_t(x)+\nabla_x\cdot\{\rho_t(x)v_t(x)\}=0.
\label{eq:bridge-continuity}
\end{equation}

\item \textbf{Noisy flow (SDE):}
\begin{align}\label{eq:bridge-sde}
dX_t=f_t(X_t)\,dt+g(t)\,dW_t,
\end{align}
where $W_t$ is a $d$-dimensional Brownian motion and $g(t)\ge 0$ is (for simplicity) an isotropic diffusion coefficient.
Take a smooth test function with compact support $\varphi \in C_c^\infty(\mathbb{R}^d)$ and apply It\^{o}'s formula to $\varphi(X_t)$:
\[
d\varphi(X_t)
=
\nabla\varphi(X_t)\cdot f_t(X_t)\,dt
+\frac{g(t)^2}{2}\Delta\varphi(X_t)\,dt
+g(t)\,\nabla\varphi(X_t)\cdot dW_t .
\]
(See, e.g., \citet{oksendal2003stochastic}.)
Taking expectations and using that the stochastic integral has mean $0$ yields
\[
\frac{d}{dt}\mathbb{E}\bigl[\varphi(X_t)\bigr]
=
\mathbb{E}\Bigl[\nabla\varphi(X_t)\cdot f_t(X_t)
+\frac{g(t)^2}{2}\Delta\varphi(X_t)\Bigr].
\]
If $X_t$ has density $\rho_t$ then
\[
\mathbb{E}[\varphi(X_t)]
=
\int_{\mathbb{R}^d}\varphi(x)\rho_t(x)\,dx,
\]
so
\[
\int_{\mathbb{R}^d} \varphi(x)\,\partial_t\rho_t(x)\,dx
=
\int_{\mathbb{R}^d}\Bigl(\nabla\varphi(x)\cdot f_t(x)
+\frac{g(t)^2}{2}\Delta\varphi(x)\Bigr)\rho_t(x)\,dx .
\]
Assuming boundary terms vanish, integration by parts gives
\[
\int_{\mathbb{R}^d}\nabla\varphi(x)\cdot f_t(x)\,\rho_t(x)\,dx
=
-\int_{\mathbb{R}^d}\varphi(x)\,\nabla\cdot\{\rho_t(x)f_t(x)\}\,dx,
\]
\[
\int_{\mathbb{R}^d} (\Delta\varphi(x))\,\rho_t(x)\,dx
=
\int_{\mathbb{R}^d}\varphi(x)\,\Delta\rho_t(x)\,dx,
\]
hence for all $\varphi$,
\[
\int_{\mathbb{R}^d} \varphi(x)\,\Big\{\partial_t\rho_t(x)
+\nabla\cdot\{\rho_t(x)f_t(x)\}
-\frac{g(t)^2}{2}\Delta \rho_t(x)
\Big\}\,dx=0 .
\]
Therefore,
\begin{align}\label{eq:bridge-fp}
\partial_t\rho_t(x)
=
-\nabla\cdot\{\rho_t(x)f_t(x)\}
+\frac{g(t)^2}{2}\Delta \rho_t(x).
\end{align}
This is the Fokker--Planck (forward Kolmogorov) equation.
The first term represents advection (transport of probability mass by $f_t$),
and the second term represents diffusion (smoothing by Brownian noise), with diffusion coefficient $g(t)^2/2$ in the isotropic case.
\end{itemize}

Statistically, \eqref{eq:bridge-ode}--\eqref{eq:bridge-continuity} describe pure transport,
while \eqref{eq:bridge-sde}--\eqref{eq:bridge-fp} describe transport plus smoothing.
Although the underlying particle processes differ, it is crucial that
\emph{at the level of density evolution, the two descriptions can be transformed into each other}.

\paragraph{A score-based identity: the diffusion term is a divergence of $\rho_t s_t$}
Define the score field
\[
s_t(x)=\nabla_x\log \rho_t(x),
\]
assuming $\rho_t$ is positive and smooth.
Then the Laplacian term can be rewritten as
\begin{equation}
\Delta\rho_t
=
\nabla_x\cdot\{\rho_t\,\nabla_x\log\rho_t\}
=
\nabla_x\cdot\{\rho_t\,s_t\}.
\label{eq:bridge-laplace}
\end{equation}
This is a direct calculus identity and does not require probability theory.

\paragraph{An ODE that yields the same density evolution (probability flow ODE)}\index{ODE}
Substituting \eqref{eq:bridge-laplace} into \eqref{eq:bridge-fp} gives
\[
\partial_t\rho_t
=
-\nabla\cdot(\rho_t f_t)
+\frac{g(t)^2}{2}\nabla\cdot(\rho_t s_t)
=
-\nabla\cdot\Bigl(\rho_t\bigl[f_t-\frac{g(t)^2}{2}s_t\bigr]\Bigr).
\]
Thus, defining
\begin{equation}
v_t(x)=f_t(x)-\frac{g(t)^2}{2}\,s_t(x)
\label{eq:bridge-v}
\end{equation}
makes \eqref{eq:bridge-fp} take the same form as the continuity equation \eqref{eq:bridge-continuity}.
In words:
\begin{quote}
The marginal density evolution of an SDE can be represented as a deterministic ODE flow with the appropriately modified velocity field \eqref{eq:bridge-v}.
\end{quote}
This ODE is often called the \emph{probability flow} ODE, but we focus on the substance of \eqref{eq:bridge-v}.

\paragraph{Meaning of the formula: score learning and velocity learning are dual}
Equation \eqref{eq:bridge-v} shows that scores $s_t$ and velocity fields $v_t$ are, in a sense, interchangeable representations.
Equivalently:
\begin{itemize}
\item If we can learn the Stein score $s_t$, then \eqref{eq:bridge-v} constructs a velocity field and enables deterministic generation via an ODE.
\item If we can learn the velocity field directly, then we can perform distribution transport without evaluating densities.
\end{itemize}
From this viewpoint, score matching and flow matching share the design philosophy ``learn without density evaluation,''
but differ in whether they estimate a \emph{gradient field} (scores) or a \emph{general vector field} (velocities).

\paragraph{Connection to probability-path design: the shared regression structure}
Both frameworks also share the feature that learning appears as regression.
In score learning, one perturbs observations (e.g., with Gaussian noise) and estimates the score $s_t(x)=\nabla\log\rho_t(x)$ by squared loss.
In flow matching, one designs a probability path, constructs a target velocity, and reduces to $L^2$ regression.
In both cases, we ``generate samples and regress'' rather than ``compute densities,'' which contributes to practical stability.

We do not delve into diffusion-model details (forward/reverse processes or score-based generation) here.
Instead, we highlighted the general identity \eqref{eq:bridge-v} showing that
the score field $s_t=\nabla\log\rho_t$ and the velocity field $v_t$ describe the same density evolution $\{\rho_t\}$ in different coordinates.
This viewpoint underpins the design of flow matching: it is legitimate to learn $v_t$ as a regression function without density evaluation.
In essence, score matching and flow matching provide equivalent continuous representations of distribution transport, interconnected via a gauge transformation of the drift and diffusion terms in the Fokker--Planck equation.

\paragraph{Summary of this chapter}
\begin{itemize}
\item The continuity equation describes distribution evolution as mass conservation under a velocity field $v_t$ (Section~\ref{subsec:liouville-continuity}).
\item Optimal transport provides a geometric reference axis between distributions through couplings and Wasserstein distance (Section~\ref{subsec:ot-monge-kantorovich}).\index{optimal transport}
\item CFM designs probability paths to avoid density computation, reduces velocity learning to regression, and enables deterministic sample generation by integrating the learned $v_t$ (Sections~\ref{subsec:probpath-cfm}--\ref{subsec:ode-sampling}).
\end{itemize}

\paragraph{Preview of the next chapter}
The next chapter develops learning-theoretic guarantees:
how regression error in the estimated velocity field propagates to distributional transformation error at the final time
(e.g., in Wasserstein-type distances or test-function discrepancies).

\chapter{Theory of Flow Learning: Approximation, Generalization, and Orthogonality}
\index{generalization}

\section{A three-way decomposition of learning error}
\index{learning error}
In this chapter we assess, from a theoretical perspective, how well we can learn a flow (a vector field).
In deep learning--based estimation, error rarely has a single cause, so before any discussion we must clarify
what we mean by ``good'' or ``bad'' estimation.
We organize the learning error into the following three components:
(1) \emph{approximation error} (expressiveness limitation of the model class),
(2) \emph{estimation error} (finite-sample fluctuation), and
(3) \emph{optimization error} (imperfect convergence of the training algorithm).
The precise inequality varies with the setting, but keeping track of
\emph{which component a design choice aims to control}
makes the subsequent theory much easier to read.

\paragraph{Setup}
Let $v^\ast(t,x)$ denote the true (or ideal) velocity field as a function of time $t$ and state $x$,
and let $\widehat v(t,x)$ be the learned estimator.
As an error metric, consider the squared error risk with respect to a reference measure $\mu(dt,dx)$:
\[
\mathcal{E}(\widehat v)
=
\iint \| \widehat v(t,x)-v^\ast(t,x)\|^2\,\mu(dt,dx).
\]
Here $\mu$ corresponds to the sampling distribution used during training.
For the purposes of this book, it suffices to keep in mind such a squared-loss (regression) risk.

\paragraph{(1) Approximation error: expressiveness of the model class}
Fix a function class ${\cal V}$ represented, for example, by a neural network family.
Even with infinite data, the error cannot vanish unless $v^\ast\in{\cal V}$.
Let
\[
v_{\cal V}^\dagger
\in
\argmin_{v\in{\cal V}} \mathcal{E}(v)
\]
be the best approximation within ${\cal V}$.
Then $\mathcal{E}(v_{\cal V}^\dagger)$ is the unavoidable error induced by the model class.
Increasing network width or depth primarily aims to reduce this approximation component.

\paragraph{(2) Estimation error: finite-sample fluctuation}
In practice we observe only $n$ training samples.
We therefore minimize an empirical risk $\widehat{\mathcal{E}}_n(v)$ in place of the population risk $\mathcal{E}(v)$.
Writing (approximately)
\[
\widehat v_n
\approx
\argmin_{v\in{\cal V}} \widehat{\mathcal{E}}_n(v),
\]
we inevitably incur an error because $\widehat{\mathcal{E}}_n$ approximates $\mathcal{E}$ with finite data.
This \emph{estimation error} typically appears as a trade-off between the complexity of ${\cal V}$ and the sample size $n$.
Intuitively, overly complex models overfit (large estimation error),
whereas overly simple models suffer from dominant approximation error.

\paragraph{(3) Optimization error: imperfect convergence}
Deep learning rarely minimizes $\widehat{\mathcal{E}}_n$ exactly.
Instead, an iterative algorithm such as stochastic gradient descent is run for a finite number of steps.
Hence the obtained iterate $\widehat v_{n,T}$ may be suboptimal by
\[
\widehat{\mathcal{E}}_n(\widehat v_{n,T})
-
\inf_{v\in{\cal V}}\;\widehat{\mathcal{E}}_n(v).
\]
This \emph{optimization error} depends on the step size, batch size, number of iterations,
initialization, and regularization.
In deep learning practice, controlling optimization error is often the decisive factor.

\paragraph{How to read the decomposition}
Conceptually, one may view
\[
\mathcal{E}(\widehat v_{n,T})
=
\underbrace{\mathcal{E}(v_{\cal V}^\dagger)}_{\text{approximation}}
+
\underbrace{\bigl\{\mathcal{E}(\widehat v_n)-\mathcal{E}(v_{\cal V}^\dagger)\bigr\}}_{\text{estimation}}
+
\underbrace{\bigl\{\mathcal{E}(\widehat v_{n,T})-\mathcal{E}(\widehat v_n)\bigr\}}_{\text{optimization}},
\]
where the exact form depends on the setting.
Most of the learning theory in this chapter addresses approximation and estimation errors,
while implementation choices and numerical experiments relate strongly to optimization error.
In what follows, we will explicitly indicate which term is being controlled.

\section{Approximation theory for neural-network vector fields}
\index{neural network}
\index{vector field}
\label{subsec:nn-approx-vfield}

The goal of this section is \emph{not} to survey the technical details of deep learning with neural networks (NNs).
Instead, we translate NNs into the language of statistics and emphasize a simple viewpoint:
\begin{quote}
A neural network is a flexible basis expansion (i.e., a function class)
for approximating high-dimensional regression functions.
\end{quote}
In flow matching we learn not a scalar density but a $d$-dimensional vector field
\[
v_t(x)\in\mathbb{R}^d\qquad (t\in[0,1],\ x\in\mathbb{R}^d),
\]
so statistically we are estimating a vector-valued regression function
\[
(t,x)\mapsto v_t(x)
\]
with input dimension $d+1$.
From this viewpoint we review, at a high level, why NNs are useful as approximators.

Thinking of NNs primarily as ``input/hidden/output layers'' can obscure their statistical role.
The essence is simpler: an NN is a parameterized function family of the form
\[
f_\theta(x)
=
A_L\circ \sigma \circ A_{L-1}\circ \cdots \circ \sigma \circ A_1(x),
\]
where $\sigma:\mathbb{R}\to\mathbb{R}$ is a fixed nonlinearity (activation function),
each $A_\ell$ is an affine map $A_\ell(z)=W_\ell z+b_\ell$, and
$\theta=((W_1,b_1),\ldots,(W_L,b_L))$ collects all parameters.
Thus an NN is simply a \emph{parametric function class defined via repeated composition}
of linear maps and a fixed nonlinearity.

\paragraph{Lipschitz continuity and stabilization of learned velocity fields}\index{Lipschitz}
A map $g:\mathbb{R}^d\to\mathbb{R}^m$ is called $K$-Lipschitz (with respect to the Euclidean norm) if
\[
\|g(x)-g(y)\|\le K\|x-y\|
\qquad (\forall x,y\in\mathbb{R}^d).
\]
The smallest such $K$ is the Lipschitz constant, denoted
\[
\mathrm{Lip}(g):=\inf\Bigl\{K\ge 0:\ \|g(x)-g(y)\|\le K\|x-y\|\ \ (\forall x,y)\Bigr\}.
\]
Assume the activation $\sigma$ is Lipschitz (e.g., ReLU is $1$-Lipschitz).
Then the Lipschitz constant of the composition $f_\theta$ admits the bound
\[
\mathrm{Lip}(f_\theta)
\le
\Bigl(\prod_{\ell=1}^{L}\|W_\ell\|_{\mathrm{op}}\Bigr)\,\mathrm{Lip}(\sigma)^{L-1},
\]
where $\|\cdot\|_{\mathrm{op}}$ is the operator norm.
Consequently, regularization or normalization that controls the spectral norms $\|W_\ell\|_{\mathrm{op}}$
(e.g., spectral normalization) can be interpreted as suppressing excessively steep learned velocity fields.
This becomes important because ODE stability (error amplification) depends directly on such Lipschitz control.

\paragraph{Continuity with kernel methods and basis expansions: NNs as ``adaptive bases.''}
NNs may feel special only if we forget that they connect smoothly to classical nonparametric estimation.
A one-hidden-layer scalar network can be written as
\[
f(x)=\sum_{m=1}^M a_m\,\sigma(w_m^\top x+b_m),
\]
which is a linear combination of basis functions.
If $(w_m,b_m)$ are fixed, this is simply linear regression on the basis $\sigma(w_m^\top x+b_m)$.
Kernel regression and kernel density estimation also admit expansions such as
\[
\widehat{f}(x)=\sum_{i=1}^n \alpha_i\,K(x,X_i).
\]
The difference is that kernel methods fix the basis $K(\cdot,X_i)$ at data locations,
whereas NNs \emph{learn the basis shapes} through $(w_m,b_m)$.
Therefore, rather than a mysterious ``black box,''
an NN is naturally viewed as \emph{nonparametric regression with automatically designed bases}.

\paragraph{What does the universal approximation theorem guarantee?}
Roughly speaking:
\begin{quote}
With sufficiently many hidden units $M$, even a one-hidden-layer NN can approximate
any continuous function on a compact set arbitrarily well.
\end{quote}
This supports the view that NNs are not coefficient models for interpretation,
but flexible approximation classes optimized to fit a loss.

\paragraph{Vector-valued outputs: nothing fundamentally new}
In flow matching the target $v_t(x)$ is vector-valued.
We may approximate each component separately:
\[
v_t(x)=\bigl(v_{t,1}(x),\dots,v_{t,d}(x)\bigr),
\]
and apply scalar approximation results componentwise.
The main difficulty is therefore not ``being a vector field'' per se,
but the familiar difficulty of high-dimensional regression.

\paragraph{Depth and width: why can being ``deep'' help?}
Universal approximation emphasizes width, but blindly increasing width can worsen computation and overfitting.
Depth (composition) matters because
\begin{quote}
functions with compositional structure can sometimes be represented
with far fewer parameters by deep models than by shallow models.
\end{quote}
Velocity fields in flow matching often arise from combinations of local structures,
and ``local-to-global'' composition aligns naturally with depth.
In statistical terms, deep NNs can act as \emph{efficient approximators under suitable structural assumptions}.\index{flow matching}

\paragraph{The curse of dimensionality does not disappear}
In high dimensions, generic nonparametric estimation remains hard, and NNs are no exception.
The practical value of NNs lies in:
\begin{itemize}
\item flexible approximation classes (reducing bias),
\item practical optimization via gradients and automatic differentiation (computability),
\item potential representation efficiency for structured functions (compositional efficiency).
\end{itemize}
Thus we use NNs not because they are ``universal,'' but because
after reducing learning to regression, they provide a scalable, practical approximator.

\paragraph{Re-interpreting ``black-box'' behavior}
NNs look like black boxes because the target is not an explicit coefficient vector (like $\beta$ in linear regression),
but a function $f_\theta$ parameterized by a high-dimensional $\theta$.
From a statistical viewpoint,
\[
\text{black-box behavior} \;\approx\; \text{the price of using a rich function class}.
\]
In the same sense, kernel methods, splines, and Gaussian process regression are also evaluated
by the behavior of the fitted function rather than interpretability of coefficients.
NNs are not uniquely opaque; they exhibit a common feature of large-scale nonparametric estimation.

In many flow matching applications, $v_t$ is a tool for generation rather than a scientific object to interpret.
Still, there is room for diagnostics and ``making it visible'' through
\begin{itemize}
\item diagnostics via inverse mappings to latent space,\index{diagnostics}
\item sensitivity analysis under conditional generation,
\item structure through regularization and constraints (smoothness, divergence control, etc.).
\end{itemize}
We return to these points later.

\paragraph{Summary: NNs are one instance of nonparametric estimation}
The conclusion is:
\begin{quote}
In flow matching, neural networks serve as function approximators for nonparametric regression of vector fields,
positioned on the same continuum as kernel methods and basis expansions.
\end{quote}
Universal approximation indicates that expressiveness is not the fundamental bottleneck,
and depth can offer representation efficiency for structured targets.
In the next section we connect approximation error to statistical (finite-sample) error by viewing flow matching itself as a regression procedure.

\section{Flow matching as nonparametric regression}
\label{subsec:fm-npreg}
\index{nonparametric learning}

In the previous section we positioned NNs as function approximators and argued that we need not overemphasize their black-box character.
We now take one more step and make explicit that the learning rule in flow matching
fits almost exactly into the familiar statistical framework of \emph{$L^2$ regression with squared loss}.
The main actor here is not the NN itself, but the shape of the objective function.
This viewpoint naturally connects flow matching to kernel regression, splines, and local polynomial regression,
and prepares us to discuss error decomposition and regularization in statistical language.

\paragraph{Training pairs are \emph{synthetically generated}.}
In ordinary regression we observe data pairs $(Z_i,Y_i)$ and learn a function to predict $Y$ from $Z$.
Flow matching is formally the same---but with one decisive difference:
the training pairs used for regression are \emph{constructed synthetically} from observed data and noise.
Concretely, we sample
\[
t\sim \mathrm{Unif}\,[0,1],\qquad X_1\sim \rho,\qquad X_t\sim \rho_t(\cdot\,|\,X_1),
\]
and then define the input
\[
Z=(t,X_t)
\]
and the teacher signal (target velocity)
\[
Y=u_t(X_t\,|\,X_1).
\]
The pair $(Z,Y)$ is the regression data.

\paragraph{Population objective as an $L^2$ regression risk}
The population objective of flow matching (especially conditional flow matching) reduces to
\begin{equation}
\min_{v\in\mathcal{V}}
\ \mathbb{E}\Bigl[\|Y-v(Z)\|^2\Bigr],
\quad Z=(t,X_t),\ \ Y=u_t(X_t\,|\,X_1).
\label{eq:npreg-risk}
\end{equation}
Here $\mathcal{V}$ is any candidate function class (NNs, other smoothers, etc.).
Equation \eqref{eq:npreg-risk} is simply squared-loss risk minimization:
``learning a flow'' can be read as ``estimating a regression function.''

\paragraph{Remark on the probabilistic structure of the training pairs}
Conditioning on the observed endpoints $\{X_1^{(i)}\}_{i=1}^n$,
if we generate $t^{(i)}$ and bridge samples $X_{t^{(i)}}^{(i)}$ independently for each $i$,
then $(Z^{(i)},Y^{(i)})$ can be treated as conditionally i.i.d.
Under this viewpoint, \eqref{eq:npreg-risk} becomes isomorphic to standard empirical risk minimization,
and the usual complexity-versus-sample-size error decompositions apply.

\paragraph{The optimizer is the conditional mean: a basic fact of regression}
If the function class is sufficiently rich, then by the $L^2$ projection property,
the minimizer is given by the conditional mean:
\begin{equation}
v^\ast(z)=\mathbb{E}[Y\,|\,Z=z].
\label{eq:cond-mean}
\end{equation}
This is a fundamental theorem of regression and formalizes why the flow matching learning rule is ``natural''.
Importantly, it holds independently of the choice of NN.
A neural network is merely one implementation for approximating $v^\ast$.

\paragraph{Analogy with kernel regression and splines}
Conceptually:
\begin{itemize}
\item \textbf{Kernel regression (local averaging):}
approximate \eqref{eq:cond-mean} locally by averaging $Y$ near $Z$.
\item \textbf{Spline smoothing (penalized least squares):}
penalize roughness (e.g., norms of derivatives) to select a smooth $v$.
\item \textbf{NN regression:}
use a rich class $\mathcal{V}$ and approximately minimize \eqref{eq:npreg-risk}
by stochastic optimization.
\end{itemize}
All three estimate $v^\ast$ by squared loss; the difference is in
\emph{which structural assumptions} (smoothness, locality, compositional structure) are implicitly imposed.
Flow matching can be summarized as:
design the probability path to control the learning distribution of $Z$,
then track the conditional mean \eqref{eq:cond-mean} with a scalable approximator.

\paragraph{Error decomposition: a statistical ``map.''}
As regression, generation quality is influenced by at least three factors:
\begin{enumerate}
\item \textbf{Approximation error:} how well $v^\ast$ can be represented by $\mathcal{V}$.
\item \textbf{Statistical error:} approximating population expectations by finite-sample averages.
\item \textbf{Optimization error:} stopping stochastic optimization after finite iterations.
\end{enumerate}
This is not unique to NNs: it is the standard regression error decomposition carried over to flow learning.

\paragraph{Why does ``likelihood-free learning'' reduce to regression?}
The probability path design ensures that the learning distribution $Z=(t,X_t)$ is \emph{samplable},
and the teacher signal $Y$ can be computed from the same randomness.
Therefore, we do not need to evaluate intermediate densities or normalizing constants as in likelihood methods.
Flow matching is designed so that learning can be performed as regression rather than density estimation.

\paragraph{Relation to score matching: a shared philosophy of avoiding likelihoods}
The claim ``flow matching learns by regression without density evaluation'' aligns closely with score matching.
Both follow a likelihood-free design:
\begin{itemize}
\item In score matching, one targets $s_0(x)=\nabla\log p_0(x)$ by squared loss but eliminates the unknown $p_0$
via a Stein identity (integration by parts), yielding an objective such as
\[
\mathbb{E}_{p_0}\!\left[\|s_\theta(X)\|^2+2\nabla\cdot s_\theta(X)\right].
\]
\item In flow matching, rather than evaluating $\rho_t$, one designs a (conditional) probability path and constructs
$(t,X_t)$ and $Y=u_t(X_t\,|\,X_1)$ by sampling, reducing to
\[
\min_{v\in\mathcal{V}}\ \mathbb{E}\bigl[\|v(t,X_t)-u_t(X_t\,|\,X_1)\|^2\bigr].
\]
\end{itemize}
The analogy is:
\[
\text{avoid likelihood evaluation}
\Longleftrightarrow
\text{(Score) remove density by integration by parts}
\Longleftrightarrow
\text{(Flow) construct teacher signals and regress}.
\]
A key difference is that score matching estimates a \emph{gradient field} describing local shape of a fixed distribution,
whereas flow matching estimates a \emph{velocity field} that directly implements distribution transport and sample generation.

\paragraph{What if we restrict to a linear velocity field?}
Consider restricting the nonparametric regression problem \eqref{eq:npreg-risk} to a linear class.
Let the velocity field be a time-dependent affine function:
\begin{equation}
v(t,x)=a+bt+Cx,
\label{eq:lin-velocity}
\end{equation}
where \(a\in\mathbb{R}^d\), \(b\in\mathbb{R}^d\), and \(C\in\mathbb{R}^{d\times d}\).
Then \eqref{eq:npreg-risk} becomes
\[
\min_{a,b,C}\ \mathbb{E}\Bigl[\|Y-(a+bt+CX_t)\|^2\Bigr],
\]
and the solution is the $L^2$ projection (ordinary least squares) of $Y$ onto this linear space.
If we collect parameters as
\[
\beta=(a^\top,b^\top,\mathrm{vec}(C)^\top)^\top,
\]
then the normal equations take the form
\[
\mathbb{E}\!\bigl[\Phi(Z)\Phi(Z)^\top\bigr]\beta
=
\mathbb{E}\!\bigl[\Phi(Z)\otimes Y\bigr],
\]
where $\Phi(Z)$ denotes the associated feature map.
In other words, the unconstrained optimum $v^\ast(z)=\mathbb{E}[Y\,|\,Z=z]$ is projected onto the linear family \eqref{eq:lin-velocity}.

Although restrictive, this linear viewpoint is useful pedagogically: it clarifies
(1) the regression nature of learning (least-squares projection),
(2) the role of stability (error amplification) through the induced ODE flow, and
(3) the connection to ODE solvers.
We keep the nonparametric form \eqref{eq:npreg-risk} as the general description,
while using \eqref{eq:lin-velocity} as a simple anchor for intuition.

\section{Convergence rates for the learned estimator}
\index{convergence rate}
\label{subsec:rate-only}

This section summarizes how the quality of the learned estimator improves with sample size $n$
from the viewpoint of nonparametric estimation rates.
A key point is that flow matching does not estimate a density $p$ directly.
Instead, it estimates a velocity field $v_t(x)$ by regression, and then obtains the final distribution $\widehat p$
as the pushforward of an ODE flow.\index{pushforward}
Accordingly, rate discussions naturally split into:
\begin{enumerate}
\item the rate for estimating the velocity field (regression),
\item how velocity-field error propagates to the error of the final distribution.
\end{enumerate}

\paragraph{Notation and basic setup}
Write the regression data as
\[
Z=(t,X_t)\in[0,1]\times\mathbb{R}^d,\qquad
Y=u_t(X_t\,|\,X_1)\in\mathbb{R}^d,
\]
so the input dimension is $D=d+1$.
Let $\mathcal{V}_n$ be a candidate class (e.g., NNs) that may depend on $n$ and define
\[
\widehat{v}_n\in\argmin_{v\in\mathcal{V}_n}\ \widehat{R}_n(v),
\qquad
\widehat{R}_n(v)=\frac1n\sum_{i=1}^n \|v(Z_i)-Y_i\|^2.
\]
The population risk is
\[
R(v)=\mathbb{E}\bigl[\|v(Z)-Y\|^2\bigr],
\]
and the population regression function is
\[
v^\ast(z)=\mathbb{E}[Y\,|\,Z=z].
\]
We view $\widehat v_n$ as an estimator of $v^\ast$.

\paragraph{Decomposing regression error}
A standard ``map'' is to decompose the excess risk
\[
R(\widehat{v}_n)-R(v^\ast)
\]
into estimation and approximation components:
\begin{equation}
R(\widehat{v}_n)-R(v^\ast)
\le
\underbrace{\Bigl\{R(\widehat{v}_n)-\inf_{v\in\mathcal{V}_n}R(v)\Bigr\}}_{\text{estimation error}}
+
\underbrace{\Bigl\{\inf_{v\in\mathcal{V}_n}R(v)-R(v^\ast)\Bigr\}}_{\text{approximation error}}.
\label{eq:rate-decomp}
\end{equation}
Approximation error depends on expressiveness (width, depth, smoothness or norm constraints),
while estimation error depends on $n$ and the complexity of the class.
Choosing model size and regularization is precisely balancing these two terms.

\paragraph{A baseline nonparametric rate}
As a classical reference point, assume $v^\ast$ has smoothness $s>0$ over $D$-dimensional inputs.
For instance, for $0<s\le1$ this corresponds to H\"older continuity:
\[
\sup_{x\neq y}\frac{\|v^\ast(x)-v^\ast(y)\|}{\|x-y\|^{s}}<\infty.
\]
Then a typical $L^2$ error rate is
\begin{equation}
\|\widehat{v}_n-v^\ast\|_{L^2}^2
= O_P\!\left(n^{-2s/(2s+D)}\right) .
\label{eq:np-rate}
\end{equation}
When $D=d+1$ is large, the rate becomes slow---the curse of dimensionality.
Flow matching cannot avoid this difficulty without additional structure assumptions.

However, in flow matching the target is not an arbitrary regression function,
but a velocity field defined on a learning distribution induced by a designed probability path.
This yields two important possibilities:
\begin{itemize}
\item Probability-path design can control \emph{where} in the input space learning matters (the distribution of $Z$).
\item If $v^\ast$ has compositional or low-dimensional structure, it may behave as if the effective dimension were smaller than $D$.
\end{itemize}
The second point is not automatic; it depends on genuine structure.

\paragraph{Neural networks as sieve approximations}
Expanding $\mathcal{V}_n$ with $n$ (larger networks for larger data) corresponds to the classical sieve idea in statistics.\citep{ShenWong1994,Chen2007}
That is, one chooses
\[
\mathcal{V}_1\subset \mathcal{V}_2\subset \cdots
\]
so that approximation error decreases while estimation error remains controlled.
From this viewpoint, NNs are not theoretically ``mysterious''; they are one sieve family for nonparametric regression.\citep{ChenShen1998}

\paragraph{Order of estimation error}
To bound estimation error one needs a complexity measure for the class, such as
parameter dimension, norm constraints, or metric entropy.
Without pursuing detailed theorems, a prototypical form is that
\[
R(\widehat{v}_n)-\inf_{v\in\mathcal{V}_n}R(v)
\]
is controlled by
\begin{equation}
C\,\frac{\mathrm{Comp}(\mathcal{V}_n)}{n}
\quad \text{or}\quad
C\,\sqrt{\frac{\mathrm{Comp}(\mathcal{V}_n)}{n}},
\label{eq:est-order}
\end{equation}
depending on assumptions and techniques.
Combining \eqref{eq:rate-decomp} and \eqref{eq:est-order} quantifies the basic trade-off:
larger models reduce approximation error but tend to increase estimation error.\citep{vdVWellner1996,BartlettMendelson2002}

\paragraph{From velocity-field error to distribution error: propagation via ODE stability}
Next we describe how error in $\widehat v_t$ affects the final distribution $\widehat p$.
Starting from the same initial distribution $X_0\sim\pi$, solve
\[
\dot{X}_t=v_t(X_t),\qquad \dot{\widehat{X}}_t=\widehat{v}_t(\widehat{X}_t)
\]
over $t\in[0,1]$ to obtain $X_1$ and $\widehat X_1$.
Under standard assumptions such as Lipschitz continuity of $v_t$ in $x$ with constant $L$,\index{Lipschitz}
a Gr\"onwall-type bound yields\citep{CoddingtonLevinson1955,AmbrosioGigliSavare2008}
\begin{equation}
\mathbb{E}\|X_1-\widehat{X}_1\|
\le
C_L\int_0^1 e^{L(1-t)}\,\mathbb{E}\|v_t(X_t)-\widehat{v}_t(X_t)\|\,dt,
\label{eq:propagation-new}
\end{equation}
where $C_L>0$ depends on $L$ (and the time interval) but not on $n$.

Using the Wasserstein distance as a distribution metric,\index{Wasserstein distance}
one likewise obtains a bound of the form\citep{Villani2009,AmbrosioGigliSavare2008}
\begin{equation}
W_2(p,\widehat{p})
\le
C_L\left(\int_0^1 \mathbb{E}\|v_t(X_t)-\widehat{v}_t(X_t)\|^2\,dt\right)^{1/2},
\label{eq:w2-prop}
\end{equation}
under appropriate regularity assumptions.
Thus, small $L^2$ regression error implies small final distribution error.
This is a major advantage of viewing flow learning as regression rather than density estimation.

\paragraph{Why is it difficult to replace NNs by classical smoothers in flow matching?}
Although flow matching is regression,
practical and theoretical bottlenecks arise if one attempts to replace NNs by classical nonparametric estimators.
We summarize the main points.

\medskip
\noindent\textbf{(i) Expressiveness in high dimensions.}
The input $Z=(t,X_t)$ is high-dimensional and the velocity field $v(t,x)$ must represent complex local structure
(nonlinear, asymmetric, and conditional behavior).
Kernel regression or spline smoothing is theoretically clean but typically suffers from exploding computation and memory in high dimensions,
and it becomes impractical at the scale of images or language embeddings.
Linear models or low-dimensional bases are computationally light but often leave large approximation bias.
NNs are a representative choice that relaxes this trade-off in practice.

\medskip
\noindent\textbf{(ii) After training, we must \emph{solve an ODE}.}
Flow matching generates samples by numerically integrating
\[
\dot x_t = v(t,x_t).
\]
At this stage, not only regression error but also \emph{ODE error amplification} matters.
If $x\mapsto v(t,x)$ is uniformly $K$-Lipschitz, then
\[
\|x_t-y_t\|\le e^{Kt}\|x_0-y_0\|,
\]
so a large Lipschitz constant amplifies numerical and approximation errors.
Hence an estimator must be stable enough to be integrated; ``fit'' alone is not sufficient.
NNs allow implementable control of $\mathrm{Lip}(v)$ through spectral normalization or gradient penalties,
whereas local-fit smoothers can easily produce steep local behavior that destabilizes ODE integration.

\medskip
\noindent\textbf{(iii) Conditional velocity fields.}
In conditional flow matching one learns $v(t,x\,|\,x_1)$, so the input becomes $(t,x,x_1)$ and the field can vary substantially with $x_1$.
One needs to share structure across conditions while allowing local adaptation.
NN parameter sharing achieves this naturally.
In contrast, conditional kernel methods scale poorly as the conditioning dimension increases.

\medskip
\noindent\textbf{(iv) Computation for learning and sampling.}
Flow matching is implemented as large-scale mini-batch optimization:
each iteration samples $(x_0,x_1)$, constructs intermediate points $x_t$ and teacher signals $u_t$,
computes the gradient of $\|v_\theta(t,x_t)-u_t\|^2$ via automatic differentiation, and updates parameters.
This is a tensorized computation pattern that maps well to GPUs and distributed data-parallel training.
Sampling also reduces to repeated forward evaluations of $v_\theta$ inside an ODE solver,
again amenable to batch parallelism and modern acceleration (mixed precision, compilation).

In contrast, many classical estimators require global iterations over all data,
large matrix factorizations, or constrained optimization,
which may not translate cleanly to SIMD-style GPU acceleration at large scale.

\medskip
\noindent\textbf{(v) Beyond ``regression'': implementing stable transport.}
Ultimately, flow matching is ``a regression problem whose goal is to implement transport stably.''\index{flow matching}
Replacing NNs therefore requires simultaneously satisfying
\[
\text{(expressiveness)}+\text{(high-dimensional computability)}+\text{(ODE stability)}.
\]
NNs currently offer the most standard compromise.
In low-dimensional settings, or when strong structure assumptions are justified,
using linear (or low-dimensional basis) velocity fields can substantially simplify both theory and computation.

\section{Numerical experiment: Lipschitzness and stability of generation}
\index{Lipschitz}
\index{stability}

This section empirically examines how smoothness of the velocity field (in particular Lipschitz control)
affects numerical stability of the generation process and robustness to outliers.\index{velocity field}

\begin{figure}[htbp]
  \centering
  \includegraphics[width=0.9\linewidth]{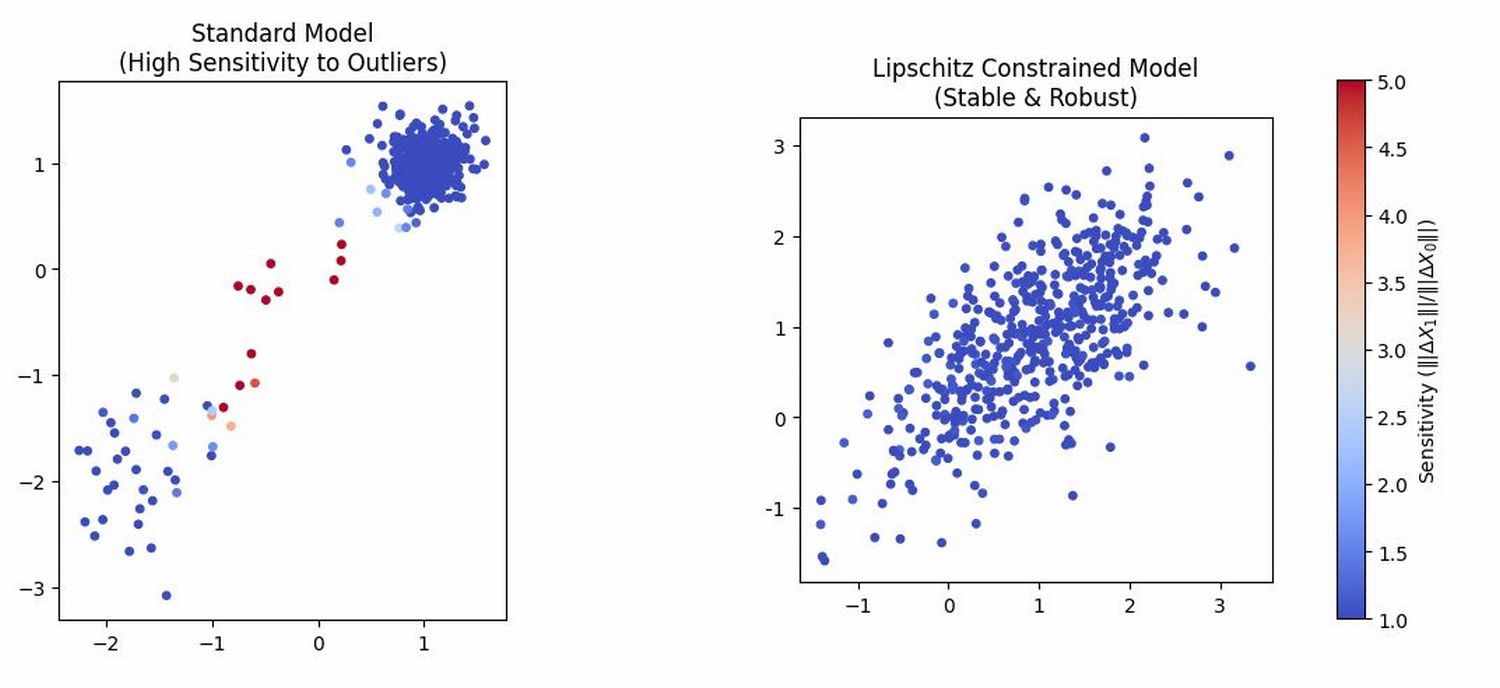}
  \caption{Stabilization and robustness induced by Lipschitz constraints on the velocity field.
The color intensity visualizes the sensitivity ratio ${\|\Delta X_1\|}/{\|\Delta X_0\|}$ of the final state to a small perturbation in the initial state;
red indicates instability.
(Left) standard training.
(Right) training with a Lipschitz constraint via spectral normalization.}
  \label{fig:stability_comparison}
\end{figure}

\paragraph{Mathematical meaning of the sensitivity index: the Jacobian of the flow}
Let $\Phi_t$ denote the ODE flow, so that $X_t=\Phi_t(X_0)$.
For a small perturbation of the initial condition,
\[
\Delta X_t \approx D\Phi_t(X_0)\,\Delta X_0,
\]
so the sensitivity ratio $\|\Delta X_1\|/\|\Delta X_0\|$ approximates the operator norm $\|D\Phi_1(X_0)\|_{\mathrm{op}}$.
If $v_t$ is Lipschitz in $x$ with constant $L$, then the variational equation and a Gr\"onwall-type argument imply
\[
\sup_x \|D\Phi_1(x)\|_{\mathrm{op}}\le e^{L}.
\]
Figure~\ref{fig:stability_comparison} visualizes this sensitivity index for a dataset containing outliers.

The statistical and geometric insights from this experiment are as follows.

\paragraph{1. Vulnerability of the standard model and amplification of error}
Under standard training (Figure~\ref{fig:stability_comparison}, left),
the sensitivity sharply increases near regions containing outliers, producing unstable red zones.
This indicates that the learned velocity field $v_t$ developed steep local changes (large gradients)
in an attempt to fit outliers.
In some locations the sensitivity exceeds $4.0$, meaning that tiny randomness in the initial state is exponentially amplified during generation.
This visually confirms the warning of Gr\"onwall-type bounds: error amplification is governed by the Lipschitz constant $L$.

\paragraph{2. Robust transport under Lipschitz constraints}
With spectral normalization (Figure~\ref{fig:stability_comparison}, right),
the sensitivity remains uniformly low (blue), indicating a stable flow.
By explicitly suppressing the Lipschitz constant of $v_t$, the model learns a smooth ``main flow'' without being dragged by outliers.
This suggests that controlling Lipschitzness acts as an effective regularization for nonparametric regression of vector fields,
and can be decisive when generation is used as a precise computational primitive (e.g., counterfactual generation).

\paragraph{Theoretical consequence: controlling error propagation}
As the Gr\"onwall bound \eqref{eq:propagation-new} indicates,
the Lipschitz constant $L$ governs the amplification factor by which estimation error propagates to the final samples.
Explicitly controlling $L$ is therefore not only a statistical regularizer but also a geometric guarantee of numerical stability for ODE solvers.

Here our goal is not a fully rigorous theorem, but a ``map'' showing which quantities dominate error propagation.
Rigorous asymptotic expansions (e.g., Donsker-type conditions) are discussed later in Section~\ref{subsec:ddml}.

\paragraph{Asymptotic theory vs finite-sample bounds}
Statistics and machine learning often target different forms of guarantees.
In statistics, asymptotic theory ($n\to\infty$) is used to approximate distributions of estimators (e.g., asymptotic normality)
and design confidence intervals and tests.
In machine learning, one often derives finite-sample generalization bounds (holding with probability $1-\delta$),
for example via Rademacher complexity.
These are not opposing viewpoints but complementary tools:
finite-sample bounds provide conservative safety guarantees but may have large constants,
whereas asymptotic theory provides accurate approximations for typical behavior but requires regularity and local assumptions.

In this chapter we emphasized nonparametric regression rates and the mechanism by which velocity-field error propagates to distribution error.
Discussions where $\sqrt{n}$ inference is essential---most notably DDML and orthogonalization---are developed in Section~\ref{subsec:ddml}.

\section{Overfitting and regularization}
\label{subsec:overfit-regularize-geom}
\index{regularization}
\index{overfitting}

We revisit overfitting in flow matching from a geometric viewpoint:
\emph{what geometry does the learned velocity field $v_t(x)$ induce},
and what does regularization suppress or protect?
In statistics, overfitting means an unnecessarily complex estimator with poor generalization.
In flow matching, complexity manifests not only in the shape of $v_t$ as a function,
but also in the geometry of the induced ODE flow $T_\theta$.
Small local perturbations in the velocity field can be nonlinearly amplified into large changes of the final distribution transformation,
creating a flow-specific source of instability.

\paragraph{Overfitting appears as local oscillations of the velocity field}
Since flow matching is regression, a typical overfitting picture is that
to reduce training loss near data points,
the velocity field develops fine-scale wiggles (high-frequency components).
Then $v_t$ becomes strongly pulled toward training points $Z_i=(t_i,X_{t_i,i})$
and may change abruptly in unnatural directions away from them.

In fluid metaphors, this resembles many small vortices that stretch tiny initial differences into large discrepancies.
Statistically:
\begin{itemize}
\item the estimation error becomes locally large, and
\item the local error accumulates through time integration of the ODE and propagates to the error of the generated distribution.
\end{itemize}

\paragraph{Geometric viewpoint: the velocity field generates a transformation map}
Generation by ODE solves
\[
\dot{X}_t=v_t(X_t),\qquad X_0\sim\pi,
\]
and outputs $X_1=T(X_0)$.
Thus $v_t$ is not merely a regression function; it is a \emph{generator} of a space-deforming map $T$.
An overfitted $v_t$ can locally fold or twist $T$,
creating unnaturally sharp pushforwards or excessive stretching of low-density regions.
Then samples may look ``fit'' near training data but become unstable.

As seen earlier, a direct approach to suppress overfitting is to control the Lipschitz constant $L$ of the velocity field.
Even if training loss looks small, strong local oscillations in $\nabla_x v_t$ can be nonlinearly amplified by ODE integration.
Hence regularization in flow matching serves \emph{both} generalization and stabilization of generation.
Large $L$ leads to exponential amplification of small errors, revealing instability and the effect of overfitting.
Conversely, Lipschitz control is a geometric condition that can achieve
\emph{overfitting suppression and numerical stability simultaneously}.

\paragraph{Weight decay: suppressing excessive bending of the induced map}
When representing $v_\theta$ by an NN, a standard regularizer is weight decay (e.g., an $\|\theta\|^2$ penalty):
\begin{equation}
\min_\theta\ \widehat{R}_n(v_\theta)+\lambda\|\theta\|^2.
\label{eq:weight-decay}
\end{equation}
This is the ridge regularization analogue in statistics.
Geometrically, weight decay indirectly controls the amplitude and curvature of the velocity field,
tending to prevent the induced flow map $T_\theta$ from twisting excessively.
While not identical to Lipschitz control, both share the goal of suppressing local amplification.

\paragraph{Smoothing regularization: suppressing high-frequency components}
From a regression viewpoint it is also natural to penalize spatial derivatives of the velocity field, e.g.,
\[
\int_0^1\int \|\nabla_x v_t(x)\|^2\,\rho_t(x)\,dx\,dt.
\]
This is a vector-field analogue of spline roughness penalties.
It directly suppresses high-frequency components and reduces ``vortices'' in the flow.
We do not go into implementation details, but this is perhaps the most geometrically transparent regularization.

\paragraph{Smoothness of the generated geometry}
Many data are believed to concentrate near low-dimensional manifolds (the manifold hypothesis),
a viewpoint often invoked in generative modeling.\index{generative model}
An ODE flow defines a map on the ambient space, but what matters is how the high-density region of $\pi$
is pushed forward into the data-concentrated region.
An overfitted velocity field can locally fold this pushforward and create sharp geometry of generated samples.
Regularization suppresses such folding and tends to preserve smoothness of the generated geometry.
We use ``manifold'' here informally to mean ``geometry induced by a smooth transformation.''

\paragraph{Implication for robustness: regularization prevents overreaction to outliers}
From a robust statistics perspective, overfitting can be seen as overreaction to outliers or biased regions.
Lipschitz control and weight decay make it harder to create locally huge velocities or steep gradients,
thereby suppressing unnatural flows driven by outliers.
Loss-function modifications (e.g., Huberization) can also help,
but in flow matching, regularization that makes the flow geometry gentler often yields practical robustness.
For mitigating learning in low-density ``void'' regions of high-dimensional data space,
a robust flow matching method based on density-weighted Stein operators has been proposed.\citep{eguchi2025implicit}

\paragraph{Practical guidelines: what to monitor and tune}
We summarize practical heuristics in the spirit of a statistician:
\begin{itemize}
\item If generation is unstable (samples ``blow up'' or concentrate locally), suspect poor Lipschitz control and strengthen regularization (weight decay, gradient penalties, spectral normalization).
\item If training loss decreases but sample quality does not improve, consider high-frequency overfitting; use early stopping or smoothing regularization.
\item Use probability-path design to smooth the learning distribution and make the regression problem itself easier.
\end{itemize}
These are not novel prescriptions; they extend classical nonparametric ideas of controlling function roughness.

\paragraph{Takeaway}
Regularization in flow matching is not merely a trick to reduce generalization error.
It controls the geometry of the induced flow, suppresses propagation of estimation error,
and stabilizes the distribution transformation produced by generation.
This viewpoint is particularly well-justified in Wasserstein geometry and gradient-flow frameworks.\citep{AmbrosioGigliSavare2008}

\section{Calibration of semiparametric models}
\index{calibration}
In this section we adopt the stance that ``model misspecification'' is not merely a finite-dimensional parameter shift,
but an infinite-dimensional deviation that distorts the distribution.
If we incorporate such deviations naively into inference, the target becomes ambiguous
and interpretable quantities (main effects, causal effects) can be obscured.
We therefore represent distributional deviations explicitly as an \emph{additional transformation}
and learn that transformation to calibrate the model, yielding a semiparametric decomposition.

\paragraph{A parametric model and a calibration transform}
Start with a parametric model with an interpretable finite-dimensional parameter $\theta\in\Theta$:
\[
{\cal M}=\{p_\theta(x):\theta\in\Theta\}.
\]
Here $p_\theta$ encodes an interpretable structure (e.g., a main effect or a causal skeleton),
but the true distribution $q(x)$ typically does not belong to ${\cal M}$.
Let
\[
X^{(0)}\sim p_\theta
\]
be a pseudo-sample generated from the baseline model, and apply an additional transformation $T_\eta$ to represent the observed distribution:
\begin{equation}
X=T_\eta(X^{(0)}),\qquad X^{(0)}\sim p_\theta .
\label{eq:calib-push}
\end{equation}
The variable $X^{(0)}$ is an ``idealized'' sample under the baseline model,
and $T_\eta$ absorbs residual distortions not explained by the baseline.
The observed distribution can be written as a pushforward\index{pushforward}
\[
q_{\theta,\eta} = (T_\eta)_\# p_\theta.
\]
Even if $T_\eta$ is invertible and smooth (so one could write density change-of-variables),
our focus is not density evaluation but \emph{whether we can sample from the calibrated distribution},
so \eqref{eq:calib-push} is taken as the basic representation.

\paragraph{Meaning as a semiparametric decomposition}
Equation \eqref{eq:calib-push} decomposes misspecification as
\[
\text{finite-dimensional (baseline)} \quad+\quad \text{infinite-dimensional (calibration)}.
\]
The baseline $p_\theta$ carries interpretability, while the calibration $T_\eta$ carries flexibility.
This clarifies three inferential requirements:
\begin{enumerate}
\item[(i)] retain interpretable finite-dimensional targets (main effects, causal effects, etc.),
\item[(ii)] absorb distributional distortions with a sufficiently flexible component,
\item[(iii)] ensure that flexibility in (ii) does \emph{not} break inference for (i) (asymptotic normality, standard errors, etc.).
\end{enumerate}
Requirement (iii) is crucial: an overly powerful $T_\eta$ can absorb effects the baseline is meant to represent.
Avoiding this cannot be achieved by ``learning a flexible generator'' alone; it must be addressed at the level of inferential design.

\paragraph{``Calibration'' rather than ``absorption.''}
The goal is not to discard the baseline model and move to fully nonparametric estimation.
Rather, we aim to preserve the interpretable structure while making residual deviations explicit as an additional transformation.
Thus $T_\eta$ should not be ``as flexible as possible'' indiscriminately; it should be introduced so as
\emph{not to unnecessarily perturb the inferential target}.
This naturally leads to the role split
\[
\theta\ \text{as the target};\quad \eta\ \text{as a nuisance}.
\]
\index{nuisance}

\paragraph{Grounding the problem in inference}
Often the inferential target is not $\theta$ itself but a functional such as a main effect $\beta$ or a causal effect $\psi$,
written generically as $\theta=\theta(P)$.
The calibrated representation gives a distribution approximation $P\approx P_{\theta,\eta}$,
but the focus of inference is
\[
\theta_0=\theta(P_0),
\]
where $P_0$ is the true distribution.
The nuisance $\eta$ is a high-flexibility component needed to approximate $P_0$ well.
The challenge is that error in estimating $\eta$ by machine learning can have a first-order effect on the distribution of $\widehat\theta$
if one uses naive plug-in estimation.
Designing estimators that suppress this ``first-order leakage'' is precisely the core of DDML, discussed next.

\paragraph{Foreshadowing DDML: why orthogonalization is needed}
Under the calibration decomposition \eqref{eq:calib-push}, $\eta$ may be high-dimensional and nonlinear (often an NN).\index{neural network}
Thus its estimation error cannot be ignored.
A naive plug-in approach (simply substituting $\widehat\eta$) may destroy the asymptotic distribution of the target.
DDML addresses this by designing a moment condition of the form
\[
\mathbb E\{\Psi(O;\theta,\eta)\}=0,
\]
where $O$ denotes the observed data (often specified as $(X,Y,A)$ in applications),
and constructing $\Psi$ so that nuisance estimation error does not enter at first order (orthogonalization).
Orthogonalization combined with cross-fitting\index{cross-fitting}
allows flexible calibration $T_\eta$ while maintaining inference on interpretable targets.

\paragraph{Role of $X^{(0)}$.}
The variable $X^{(0)}$ in \eqref{eq:calib-push} is not observed; it is the baseline pseudo-sample generated from $p_\theta$.
The calibration map $T_\eta$ transforms this pseudo-sample into the observed distribution.
Thus $X^{(0)}$ represents the interpretable baseline component, while $T_\eta$ represents residual distortions.

\section{Double machine learning (DDML)}
\index{double machine learning}\index{DDML}
\label{subsec:ddml}

\paragraph{Positioning: $\sqrt{n}$ inference with high-flexibility nuisance estimation.}
In the previous section we viewed the target as a functional $\theta_0=\theta(P_0)$
and wrote the high-flexibility components needed for identification and estimation as a nuisance $\eta_0=\eta(P_0)$.
Flows and conditional generators are powerful estimators for approximating $\eta_0$ in a samplable form,
but naive plug-in substitution can leave a first-order regularization bias.
Suppressing this first-order leakage is the core of DDML (Double/Debiased Machine Learning),
which combines orthogonalization (Neyman orthogonality)\index{Neyman orthogonality}
and cross-fitting.\citep{chernozhukov2018ddml}
We formulate it in generic notation not restricted to causal inference.

Let the target be $\theta(P)$ and the nuisance be $\eta(P)$.
Suppose $\theta$ is defined by a (Fisher-score type) estimating equation
\[
\mathbb{E}\bigl[\Psi(O;\theta,\eta)\bigr]=0,
\]
where $\Psi$ is an influence-function type moment.\index{influence function}
We say that $\Psi$ is Neyman-orthogonal with respect to $\eta$ if, under the true distribution $P_0$,
\begin{equation}
\left.\frac{d}{d\varepsilon}\ \mathbb{E}_{P_0}\!\bigl[\Psi\bigl(O;\theta_0,\eta_0+\varepsilon h\bigr)\bigr]\right|_{\varepsilon=0}=0
\qquad (\text{for any perturbation }h).
\label{eq:neyman_orth}
\end{equation}
This definition means that nuisance error does not enter at first order:
even if $\widehat\eta-\eta_0$ converges relatively slowly,
the leading behavior of $\widehat\theta$ may still be governed by a central limit theorem.

\paragraph{Functional derivatives (G\^{a}teaux and Fr\'echet).}
Equation \eqref{eq:neyman_orth} is a functional derivative statement.
Define
\[
M(\eta)\;=\;\mathbb{E}_{P_0}\!\bigl[\Psi(O;\theta_0,\eta)\bigr].
\]
If the directional derivative exists along any direction $h$,
\[
D M(\eta_0)[h]
=
\left.\frac{d}{dr}M(\eta_0+r h)\right|_{r=0}
=
\lim_{r\to 0}\frac{M(\eta_0+r h)-M(\eta_0)}{r},
\]
then $M$ is G\^{a}teaux differentiable at $\eta_0$.
Moreover, $M$ is Fr\'echet differentiable if there exists a bounded linear functional $\dot M_{\eta_0}$,
with respect to a norm $\|\cdot\|$ (e.g., an $L^2(P_0)$ norm), such that
\begin{equation}
M(\eta_0+h)
=
M(\eta_0)+\dot M_{\eta_0}(h)+o\bigl(\|h\|\bigr)
\qquad (\|h\|\to 0).
\label{eq:frechet_def}
\end{equation}
Fr\'echet differentiability implies G\^{a}teaux differentiability and $\dot M_{\eta_0}(h)=D M(\eta_0)[h]$.
Orthogonality in this book is at least in the directional (G\^{a}teaux) sense,
while stronger uniform arguments may require Fr\'echet differentiability.

\paragraph{Bartlett's identity and the meaning of orthogonality}
To build intuition, consider a smooth nuisance submodel $P_r$ passing through $P_0$,
and let its Fisher score (tangent) be
\[
s_h(O)
=
\left.\frac{d}{dr}\log p_r(O)\right|_{r=0},
\]
where $h$ indexes the nuisance direction.
Bartlett's first identity yields
\[
\mathbb{E}_{P_0}[s_h(O)]=0.
\]
If $g(O)$ is regular enough, then the derivative of an expectation satisfies
\begin{equation}
\left.\frac{d}{dr}\mathbb{E}_{P_r}[g(O)]\right|_{r=0}
=
\mathbb{E}_{P_0}\!\bigl[g(O)\,s_h(O)\bigr].
\label{eq:bartlett_cov}
\end{equation}
This is Bartlett's covariance identity: ``derivative of an expectation equals covariance with the score.''

From this viewpoint, Neyman orthogonality means that the moment function $\Psi$ is $L^2(P_0)$-orthogonal to the nuisance score directions:
\[
\mathbb{E}_{P_0}\!\bigl[\Psi(O;\theta_0,\eta_0)\,s_h(O)\bigr]=0\; 
\qquad (\text{for all allowed nuisance directions }h).
\]
See Figure~\ref{fig:neyman}.

\begin{figure}[!t]
  \centering
  \includegraphics[width=0.65\linewidth]{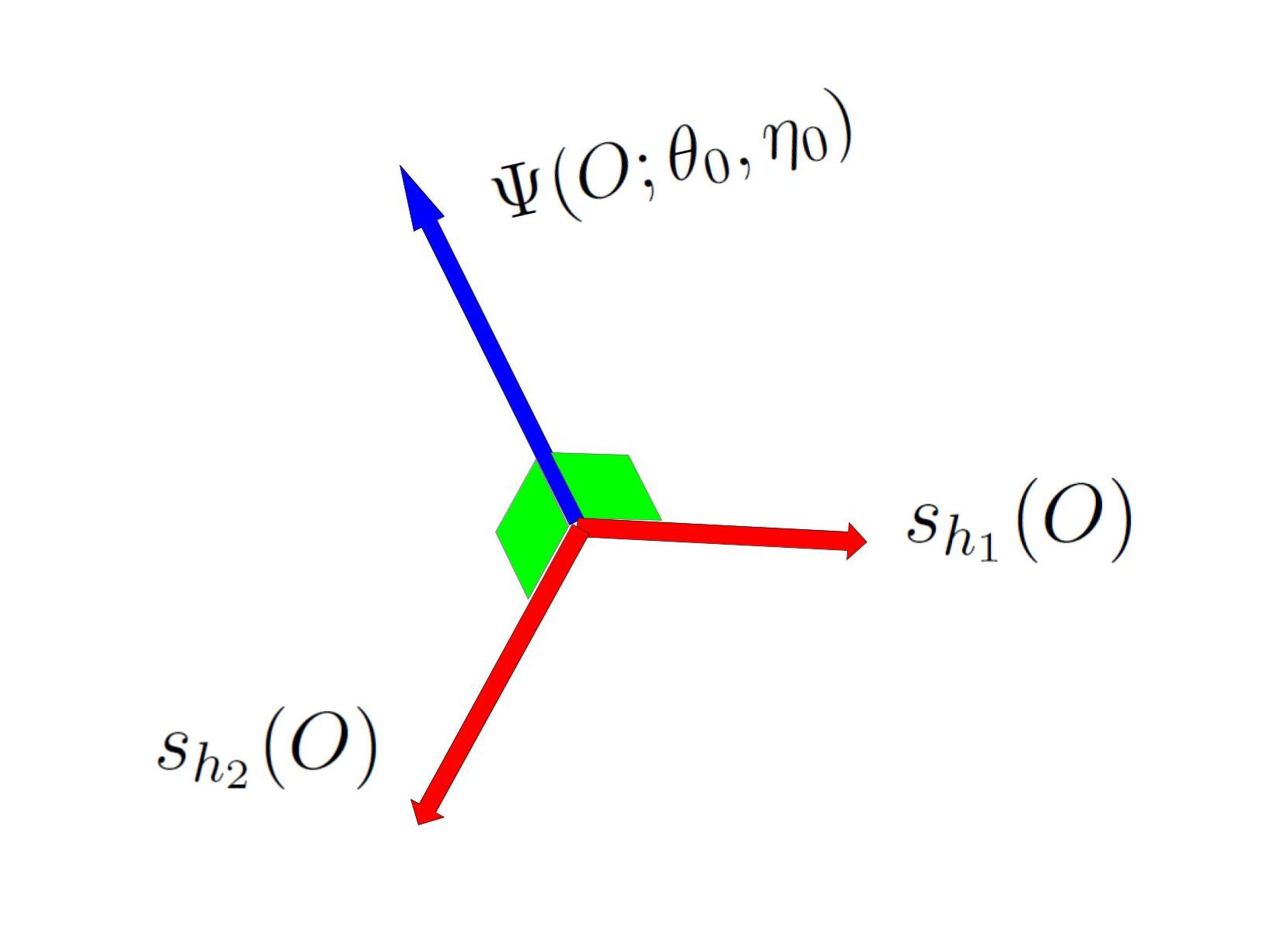}
  \caption{Neyman orthogonality: parameter score vs nuisance score directions.}
  \label{fig:neyman}
\end{figure}

Geometrically, the collection $\{s_h\}$ spans the nuisance tangent space,\index{nuisance tangent space}
and designing $\Psi$ to lie in its orthogonal complement ensures that
small shifts along nuisance directions do not change $\mathbb{E}[\Psi]$ at first order.
This is the intuitive meaning of \eqref{eq:neyman_orth}, and explains why DDML can preserve inference even with complex nuisance learners.

As an aside, Bartlett's second identity
\[
-\mathbb{E}\!\Bigl[\frac{\partial}{\partial\eta}S_\eta(O)\Bigr]
=
\mathbb{E}\!\bigl[S_\eta(O)S_\eta(O)^\top\bigr]
\]
states that the score variance equals the Fisher information.
In information-geometric terms, orthogonalization can be interpreted as projecting away nuisance components.

\paragraph{Bridge to asymptotic theory}
If $\Psi(O;\theta,\eta)$ is Neyman-orthogonal in $\eta$,
then the first-order asymptotic expansion of $\widehat\theta$ does not involve $\widehat\eta-\eta_0$.
Hence, if $\widehat\eta$ achieves roughly an $n^{-1/4}$ rate,
$\widehat\theta$ can still admit $\sqrt{n}$ inference.
We state this as a theorem and then discuss local robustness.

\begin{theorem}\label{thm:asymptotic_normality}
Assume \eqref{eq:neyman_orth} holds and the nuisance estimator satisfies
$$\|\widehat{\eta} - \eta_0\|_2 = o_p(n^{-1/4}).$$
Then any solution $\widehat{\theta}$ to the score equation
$\Psi_n(\widehat{\theta}, \widehat{\eta}) = 0$
satisfies
$$
\sqrt{n}(\widehat{\theta} - \theta_0) \xrightarrow{d} N(0, J^{-1} V J^{-1}),
$$
where $J =\mathbb E[ -\nabla_\theta \Psi(\theta_0, \eta_0)]$
and $V = \mathbb{E}[\Psi(O;\theta_0,\eta_0)\Psi(O;\theta_0,\eta_0)^\top]$.
\end{theorem}

\begin{proof}
Let $\Psi_n(\theta, \eta) = \frac{1}{n}\sum_{i=1}^n \Psi(O_i; \theta, \eta)$.
Expand the defining equation $\Psi_n(\widehat{\theta}, \widehat{\eta}) = 0$ around $(\theta_0,\eta_0)$:
$$
0 = \Psi_n(\theta_0, \eta_0) + (\nabla_\theta \Psi_n)(\widehat{\theta} - \theta_0) + (\nabla_\eta \Psi_n)[\widehat{\eta} - \eta_0] + R_n,
$$
where $R_n$ is a higher-order remainder.
Rearranging,
$$
\sqrt{n}(\widehat{\theta} - \theta_0)
=
-(\nabla_\theta \Psi_n)^{-1}
\left(
\sqrt{n}\Psi_n(\theta_0, \eta_0)
+
\sqrt{n}(\nabla_\eta \Psi_n)[\widehat{\eta} - \eta_0]
\right)
+
o_p(1).
$$
We evaluate each term:
\begin{enumerate}
\item \emph{Leading term.}
By the central limit theorem,
$\sqrt{n}\Psi_n(\theta_0, \eta_0)=\frac{1}{\sqrt{n}}\sum_i \Psi(O_i;\theta_0,\eta_0)$ converges to $N(0,V)$.
\item \emph{Jacobian.}
By the law of large numbers, $\nabla_\theta \Psi_n \xrightarrow{p} \mathbb{E}[\nabla_\theta \Psi] = -J$.
\item \emph{Effect of the nuisance.}
Let $T_n=\sqrt{n}(\nabla_\eta \Psi_n)[\widehat{\eta} - \eta_0]$.
Under stochastic equicontinuity of the empirical process, $\nabla_\eta \Psi_n$ may be replaced by its expectation.
By Neyman orthogonality, $\nabla_\eta \Psi(\theta_0,\eta_0)=0$ in the relevant (directional) sense,
so the nuisance enters only at second order, of size $O_p(\sqrt{n}\|\widehat{\eta}-\eta_0\|^2)$.
Under $\|\widehat{\eta}-\eta_0\|=o_p(n^{-1/4})$, we have $T_n=o_p(1)$.
\end{enumerate}
Therefore $\sqrt{n}(\widehat{\theta} - \theta_0)$ converges to $N(0,J^{-1}VJ^{-1})$.
\end{proof}

Theorem~\ref{thm:asymptotic_normality} implies asymptotic linearity of $\widehat\theta$ and that $V$ can be estimated in sandwich form,
enabling Wald-type confidence intervals and tests.
The key point is that nuisance estimation does not appear in the first-order limit distribution---this is the practical meaning of orthogonalization.
A fully rigorous proof requires regularity conditions (e.g., uniform boundedness of moments in a semiparametric model).
For detailed treatments, see \citet{vanDerVaart1998,Tsiatis2006}.

\paragraph{Implementation via cross-fitting}
A generic cross-fitted DDML procedure is summarized in Algorithm~\ref{alg:cox_dml}.

\begin{algorithm}[!tbp]
\caption{Cross-fitted DDML with an orthogonal estimating equation}
\label{alg:cox_dml}
\begin{algorithmic}[1]
\Require Data $\{O_i\}_{i=1}^n$ (e.g., $O_i=(X_i,A_i,Y_i)$), number of folds $K$,
estimating function $\Psi(O;\theta,\eta)$, nuisance learner $\mathcal{A}_\eta$
\Ensure $\widehat\theta$ and (optionally) an asymptotic variance estimate $\widehat V$
\State Split indices $\{1,\ldots,n\}$ into $K$ disjoint folds $I_1,\ldots,I_K$.
\For{$k=1,\ldots,K$}
  \State Set training indices $I_{-k}\gets \{1,\ldots,n\}\setminus I_k$.
  \State \textbf{(Nuisance fit)} Train $\widehat\eta^{(-k)}\gets \mathcal{A}_\eta(\{O_i\}_{i\in I_{-k}})$.
\EndFor
\State Define the cross-fitted estimating equation
\[
\widehat\Psi_n(\theta)=\frac{1}{n}\sum_{k=1}^K\sum_{i\in I_k}
\Psi\bigl(O_i;\theta,\widehat\eta^{(-k)}\bigr).
\]
\State \textbf{(Solve)} Numerically find $\widehat\theta$ such that $\widehat\Psi_n(\widehat\theta)=0$
(e.g., Newton or quasi-Newton; in one dimension, bisection).
\If{variance estimation is requested}
  \State Compute
  $
  \widehat A\gets -\frac{1}{n}\sum_{k=1}^K\sum_{i\in I_k}
  \bigl(\frac{\partial}{\partial\theta}\bigr)^\top \Psi(O_i;\widehat\theta,\widehat\eta^{(-k)}).
  $
  \State Compute
  $
  \widehat B\gets \frac{1}{n}\sum_{k=1}^K\sum_{i\in I_k}
  \Psi(O_i;\widehat\theta,\widehat\eta^{(-k)})\Psi(O_i;\widehat\theta,\widehat\eta^{(-k)})^\top.
  $
  \State Set $\widehat V\gets \widehat A^{-1}\widehat B\,\widehat A^{-\top}$.
\EndIf
\State \Return $\widehat\theta$ (and optionally $\widehat V$).
\end{algorithmic}
\end{algorithm}

\paragraph{Connection to a local misspecification model}
Finally, we describe robustness to a local misspecification of the baseline model.
Consider the contiguous neighborhood
\begin{align}\label{local}
\eta_{0,n}=\eta_0+n^{-1/2}\Delta.
\end{align}
We evaluate the first-order bias (shift in the limit distribution) under this local perturbation.

\begin{proposition}[Local robustness]
\label{prop:local_robust}
Assume the local model \eqref{local}.
Then the first-order shift appearing in the limit distribution of $\sqrt{n}(\widehat\theta-\theta_0)$
depends only on the component of $\Delta$ that remains after orthogonalization,
and there is no first-order bias for perturbations compatible with the orthogonality constraint.
\end{proposition}

\begin{proof}
Under the local model the true nuisance is $\eta_{0,n}$.
As in the proof of Theorem~\ref{thm:asymptotic_normality},
\[
\sqrt{n}(\widehat\theta-\theta_0)
=
J^{-1}\frac{1}{\sqrt{n}}\sum_{i=1}^n \Psi(O_i;\theta_0,\eta_{0,n})
+o_p(1).
\]
Hence the shift is determined by
\[
b_\Delta
=
J^{-1}\lim_{n\to\infty}\sqrt{n}\,\mathbb{E}\bigl[\Psi(O;\theta_0,\eta_{0,n})\bigr].
\]
By the G\^{a}teaux derivative,
\[
\sqrt{n}\,\mathbb{E}\bigl[\Psi(O;\theta_0,\eta_{0,n})\bigr]
=
\left.
\frac{\partial}{\partial\varepsilon}
\mathbb{E}\bigl[\Psi(O;\theta_0,\eta_0+\varepsilon \Delta)\bigr]
\right|_{\varepsilon=0}
+o(1).
\]
By Neyman orthogonality \eqref{eq:neyman_orth}, this derivative vanishes for directions eliminated by orthogonalization.
Therefore $b_\Delta$ depends only on the projection of $\Delta$ onto the orthogonal complement of the nuisance tangent space.
In particular, if $\Delta$ lies entirely in the eliminated directions, then $b_\Delta=0$.
\end{proof}

\paragraph{Connection to the next chapter}
Chapters~5 and beyond build on the ``regression error $\to$ distribution error'' map
and incorporate flows into concrete statistical models such as copulas, survival models, and missing-data imputation,
clarifying how model misspecification and generation error influence inferential targets.

\chapter{Extensions to Statistical Models: Dependence, Time, and Missingness}
\index{copula}
\index{survival analysis}
\index{missing data}
\index{multiple imputation}
\index{MAR}
\index{MNAR}
\label{sec:stat-extensions}

This chapter connects flow matching (and more broadly generative modeling) to three classical statistical domains:
(1) dependence modeling via copulas,
(2) event-time data with censoring in survival analysis, and
(3) missing-data inference via multiple imputation (MI).
We start from linear regression as a warm-up, to clarify the common design pattern:
\begin{quote}
Flows are often more useful as \emph{conditional samplers} than as generic density estimators.
They serve as modular components for completion, imputation, and counterfactual generation.
\end{quote}

All three examples can be organized by the same blueprint.
We first fix an interpretable \emph{base model} (marginals, hazards, mean structure, etc.).
Then we absorb only the remaining complexity (nonlinearity, dependence, incomplete observation)
as a \emph{residual} learned by a flow (or by conditional flow matching).
More concretely:
\begin{enumerate}
\item \textbf{Choose a base model.} Model the part that should remain interpretable (marginals, hazards, mean structure, etc.).
\item \textbf{Choose coordinates for residuals.} Map the unexplained component into a convenient coordinate system
(e.g., uniformization via the probability integral transform, or latent-variable representations).
\item \textbf{Learn conditional generation.} Learn the conditional distribution needed for completion or imputation as a regression problem
along a designed probability path, and use it as a sampler.
\end{enumerate}
A practical reason this perspective works so well is the availability of reliable ODE solvers:
once a vector field has been learned, sampling is straightforward by solving an ODE
(see Section~\ref{subsec:ode-sampling}).\index{ODE}

\section{Linear regression model and scores}
We begin with a simple but instructive case: linear regression.
We keep the regression parameter (the main inferential target) interpretable,
while relaxing only the Gaussianity of the error distribution by introducing a nuisance density.
This clarifies how (i) orthogonalized scores in DDML and (ii) nuisance learning by flow matching
fit into a common semiparametric template.

For $i=1,\dots,n$, consider
\[
Y_i=X_i^\top\beta+\sigma\epsilon_i,\qquad \epsilon_i\sim N(0,1).
\]
This imposes a strong normality assumption on the error.

\index{nuisance}
More generally, let
\[
Y_i=X_i^\top\beta+\sigma \varepsilon_i,\qquad \varepsilon_i\sim p_\eta,
\]
and impose identifying constraints $\mathbb{E}[\varepsilon_i]=0$ and $\mathbb{E}[\varepsilon_i^2]=1$.
With the standardized residual
\[
e_i=e_i(\beta,\sigma):=\frac{Y_i-X_i^\top\beta}{\sigma},
\]
the conditional density is
\[
f(y_i\,|\,X_i;\beta,\sigma,p_\eta)=\frac{1}{\sigma}p_\eta(e_i),
\]
and the log-likelihood becomes
\begin{equation}
\ell_n(\beta,\sigma,p_\eta)
=
\sum_{i=1}^n \Bigl\{\log p_\eta(e_i)-\log\sigma\Bigr\}
=
\sum_{i=1}^n \log p_\eta(e_i)-\frac{n}{2}\log\sigma^2.
\label{eq:loglik-locscale}
\end{equation}
If $p_\eta=\varphi$ is the standard normal density, then
\[
\ell_n(\beta,\sigma^2,\varphi)
=
-\frac{n}{2}\log(2\pi\sigma^2)-\frac{1}{2\sigma^2}\sum_{i=1}^n (Y_i-X_i^\top\beta)^2.
\] 

Define the (one-dimensional) Stein score
\[
s_\eta(e):=\partial_e\log p_\eta(e).
\]
Differentiating \eqref{eq:loglik-locscale}, the per-observation scores for $O_i=(Y_i,X_i)$ are
\begin{align}
\psi_{\beta}(O_i;\beta,\sigma,\eta)
&:=
\partial_\beta \ell_i
=
-\frac{1}{\sigma}X_i\, s_\eta(e_i),
\label{eq:score-beta-raw}\\
\psi_{\sigma^2}(O_i;\beta,\sigma,\eta)
&:=
\partial_{\sigma^2}\ell_i
=
-\frac{1}{2\sigma^2}\Bigl\{1+e_i s_\eta(e_i)\Bigr\}.
\label{eq:score-sig2-raw}
\end{align}
In the Gaussian case, $s_\eta(e)=-e$, so these reduce to the usual normal linear-regression scores.

\paragraph{Nuisance tangent space and orthogonalization by projection}
Perturbations of $p_\eta$ that preserve mean $0$ and variance $1$ generate the nuisance tangent space in $L^2(p_\eta)$,
\[
\mathcal{T}_\eta
=
\Bigl\{h(e):\ \mathbb{E}[h(\varepsilon)]=0,\ \mathbb{E}[\varepsilon h(\varepsilon)]=0,\ 
\mathbb{E}[(\varepsilon^2-1)h(\varepsilon)]=0\Bigr\}.
\]
Hence the orthogonal complement is
\[
\mathcal{T}_\eta^\perp=\mathrm{span}\{1,\ e,\ e^2-1\}.
\]
The efficient score is obtained by projecting the parametric score
\eqref{eq:score-beta-raw}--\eqref{eq:score-sig2-raw} onto $\mathcal{T}_\eta^\perp$.
A key object is the projected score function
\[
\tilde s_\eta(e):=\Pi_{\mathcal{T}_\eta^\perp}\,s_\eta(e)
=
b_\eta\, e+c_\eta\,(e^2-1),
\]
where the constant term vanishes because $\mathbb{E}[s_\eta(\varepsilon)]=0$.
Let
\[
\mu_3:=\mathbb{E}[\varepsilon^3],\qquad \mu_4:=\mathbb{E}[\varepsilon^4].
\]
Solving the normal equations for the $L^2(p_\eta)$ projection yields
\begin{equation}
c_\eta=\frac{\mu_3}{\mu_4-1-\mu_3^2},\qquad
b_\eta=-\frac{\mu_4-1}{\mu_4-1-\mu_3^2}.
\label{eq:bc-coef}
\end{equation}
Therefore the efficient scores can be written as
\begin{align}
\psi_{\beta}^{\mathrm{eff}}(O_i;\beta,\sigma^2,\eta)
&=
-\frac{1}{\sigma}X_i\,\tilde s_\eta(e_i),
\label{eq:score-beta-eff}\\
\psi_{\sigma^2}^{\mathrm{eff}}(O_i;\beta,\sigma^2,\eta)
&=
-\frac{1}{2\sigma^2}\Bigl\{1+e_i\tilde s_\eta(e_i)\Bigr\}.
\label{eq:score-sig2-eff}
\end{align}
If the error distribution is symmetric ($\mu_3=0$), then $c_\eta=0$ and $b_\eta=-1$,
so $\tilde s_\eta(e)=-e$ and \eqref{eq:score-beta-eff} coincides with the OLS score.
This is a concrete instance of Neyman orthogonality:\index{orthogonality}
the score has been orthogonalized against nuisance directions, so nuisance estimation error
does not impact the target estimator to first order.

\paragraph{Learning the nuisance distribution by flow matching}\index{flow matching}
Here the nuisance parameter $\eta$ is the error density $p_\eta$.
Given standardized residual samples $\{e_i\}$, we can learn a transport from $Z\sim N(0,1)$
to $p_\eta$ via flow matching, thereby obtaining a \emph{generator} for the errors.
Concretely, we learn a velocity field $v_\theta(t,z)$ so that particles evolve as
\[
\frac{d}{dt}Z_t=v_\theta(t,Z_t),\qquad t\in[0,1],
\]
and the terminal distribution of $Z_1$ matches the empirical residual distribution.
In the simplest FM setup, we pair $e_1$ (a residual sample) with $e_0\sim N(0,1)$, draw $t\sim\mathrm{Unif}(0,1)$,
use linear interpolation $e_t=(1-t)e_0+te_1$, and fit $v_\theta$ by
\[
\min_\theta\ 
\mathbb{E}\Bigl[\bigl(v_\theta(t,e_t)-(e_1-e_0)\bigr)^2\Bigr].
\]
After training, solving the ODE with $Z_0\sim N(0,1)$ produces pseudo-samples from $p_\eta$
without requiring an explicit density formula.
The learned generator can be used for (i) additional sampling (diagnostics, bootstrap-like uses) and
(ii) approximating the Stein score $s_\eta$ (e.g., via KDE in one dimension, or via DSM).

\paragraph{Cross-fitting}
Split the sample into $K$ folds.
For fold $k$, learn $\widehat\eta^{(-k)}$ (e.g., a flow-based generator/score for $p_\eta$) using the complement of fold $k$.
Then in fold $k$ solve estimating equations such as
\[
\sum_{i\in k}\psi_\beta(O_i;\beta,\sigma^2,\widehat\eta^{(-k)})=0,\qquad
\sum_{i\in k}\psi_{\sigma^2}(O_i;\beta,\sigma^2,\widehat\eta^{(-k)})=0,
\]
or use them for a one-step update from an initial estimator.
Averaging over folds yields $(\widehat\beta,\widehat\sigma^2)$.
This stabilizes inference under flexible nuisance learning.

The importance of residualization/orthogonalization in high-dimensional linear models is emphasized in
\citet{belloni2014pds}. Related ``debiased'' constructions for low-dimensional inference under regularization
include \citet{zhang2014jssb} and \citet{vandegeer2014aos}, which can be viewed as score constructions via projection.

\section{Copulas, censoring, and MAR: a roadmap}
\index{copula}
\index{censoring}
\index{MAR}
\index{MNAR}
\index{multiple imputation}
We now connect flows to (1) dependence modeling (copulas), (2) survival analysis with censoring, and (3) missing data via MI.
These topics use different terminology, so this section provides only the minimum ``map'' needed for this chapter.

\subsection*{Copulas: separating marginals and dependence}
\paragraph{Goal}
A copula separates a multivariate distribution into marginal distributions and a dependence component.
This is useful when we want to flexibly enrich dependence while keeping marginals interpretable, or vice versa.

\paragraph{Basic form}
Let $X=(X_1,\dots,X_d)$ be continuous with marginal CDFs $F_j(x_j)$.
With $U_j=F_j(X_j)$ we have $U_j\in(0,1)$ and the joint distribution of $U=(U_1,\dots,U_d)$ has uniform marginals.
If
\[
P(X_1\le x_1,\dots,X_d\le x_d)
=
C(F_1(x_1),\dots,F_d(x_d)),
\]
then $C$ is called a copula.
This provides a natural semiparametric decomposition: keep the marginals (or a simple base dependence model)
parametric and use a flow to calibrate the residual dependence.

\subsection*{Survival analysis: censoring, hazards, proportional hazards}
Event times $T$ (death, relapse, failure, etc.) are often not fully observed due to limited follow-up.
With a censoring time $C$, the observed data are typically
\[
Y=\min(T,C),\qquad \Delta=\mathbb{I}(T\le C),
\]
where $\Delta=0$ indicates right censoring.

Let $S(t)=P(T>t)$ and density $f(t)$.
The hazard is
\[
\lambda(t)=\lim_{\Delta t\downarrow 0}\frac{P(t\le T<t+\Delta t\,|\,T\ge t)}{\Delta t}
=\frac{f(t)}{S(t)}.
\]
With cumulative hazard $\Lambda(t)=\int_0^t \lambda(s)\,ds$, we have $S(t)=\exp\{-\Lambda(t)\}$.

\paragraph{Cox proportional hazards model}
Cox PH models
\[
\lambda(t\,|\,x)=\lambda_0(t)\exp(\beta^\top x),
\]
estimating the covariate effect $\beta$ as a finite-dimensional target.
In this book we keep $\beta$ interpretable while absorbing departures from PH or more complex time dependence via a flow-based correction.
Censoring makes conditional distributions essential, which is precisely where conditional generation becomes valuable.

\subsection*{Missing data: MCAR/MAR/MNAR and multiple imputation}
\paragraph{Missingness mechanisms}
Write $X=(X_{\mathrm{obs}},X_{\mathrm{mis}})$ and let $R$ indicate observation (1) vs missingness (0).
The missingness mechanism is $P(R\,|\,X)$.
Common assumptions are:
\begin{itemize}
\item \textbf{MCAR} (Missing Completely at Random): $R\perp X$.
\item \textbf{MAR} (Missing at Random):
\[
P(R\,|\,X_{\mathrm{obs}},X_{\mathrm{mis}})=P(R\,|\,X_{\mathrm{obs}}).
\]
\item \textbf{MNAR} (Missing Not at Random): missingness may depend on $X_{\mathrm{mis}}$; MAR fails.
\end{itemize}

\paragraph{Rubin's multiple imputation}
MI creates $K$ completed data sets, computes an estimate $\widehat\theta^{(k)}$ and variance estimate $\widehat V^{(k)}$
from each, then combines them via
\[
\bar\theta=\frac{1}{K}\sum_{k=1}^K \widehat\theta^{(k)},
\quad
\bar V=\frac{1}{K}\sum_{k=1}^K \widehat V^{(k)},
\quad
B=\frac{1}{K-1}\sum_{k=1}^K
(\widehat\theta^{(k)}-\bar\theta)(\widehat\theta^{(k)}-\bar\theta)^\top,
\]
and total variance
\[
T=\bar V+\left(1+\frac{1}{K}\right)B.
\]
The key statistical requirement is the ability to generate valid samples from
\[
X_{\mathrm{mis}} \,|\,X_{\mathrm{obs}},
\]
which is exactly conditional generation.
Under MNAR, non-identifiability remains and sensitivity analysis (or extra assumptions) is unavoidable.

With this roadmap, we now develop the three connections:
(1) residual dependence calibration via flow copulas,
(2) survival models that preserve interpretability while extending distributional structure,
and (3) imputation as conditional generation in MI.

\section{Flow copulas}
\index{copula}
\index{dependence}
\index{Sklar's theorem}
\label{subsec:flow-copula}

\paragraph{Motivation: keep marginals, flex dependence}
In multivariate modeling we must handle both marginal shapes and dependence.
Statistically, it is often desirable to preserve interpretable marginal modeling
(e.g., regression structures or scales) while flexibly enriching dependence.
Copulas provide the classical coordinate system for this separation, and flows provide a powerful nonparametric dependence module.

\paragraph{Sklar's theorem and the copula factorization}
Let $F_j(x)=P(X_j\le x)$.
For continuous marginals, the probability integral transform gives
\[
U_j=F_j(X_j),\qquad U_j\sim\mathrm{Unif}(0,1).
\]
Sklar's theorem states that the joint CDF can be written as
\[
F(x_1,\dots,x_d)=C(F_1(x_1),\dots,F_d(x_d)).
\]
If a joint density exists, then
\[
p(x_1,\dots,x_d)=c(u)\,\prod_{j=1}^d f_j(x_j),
\qquad u_j=F_j(x_j),
\]
where $c(u)=\partial^d C(u)/(\partial u_1\cdots\partial u_d)$ is the copula density.
Thus we can ``delegate'' marginals to a statistical model and learn only $c$ by a flexible generator.

\paragraph{Pseudo-observations and boundary stabilization}
In practice $F_j$ are unknown, so we use pseudo-observations
\[
\widehat U_{ij}=\widehat F_j(X_{ij}).
\]
Naively plugging in empirical CDFs can be unstable near $0$ and $1$ after logit/probit transforms.
A robust implementation uses rank-based stabilization
\[
\widehat U_{ij}=\frac{R_{ij}-1/2}{n}\qquad (R_{ij}:\ \text{rank of }X_{ij}),
\]
and/or clipping $\widehat U_{ij}\leftarrow \min(1-\varepsilon,\max(\varepsilon,\widehat U_{ij}))$ with $\varepsilon\asymp n^{-1}$.
For inference (standard errors), uncertainty from marginal estimation can propagate into dependence estimation;
bootstrap that re-fits marginals, or cross-fitting, is a principled way to account for this.\index{cross-fitting}

\paragraph{Flow copulas via transforms on $\mathbb{R}^d$.}
The copula domain is $[0,1]^d$, whose boundary can be inconvenient.
We map $u\in(0,1)^d$ into $z\in\mathbb{R}^d$ using, for example, the logit transform
\[
z=\mathrm{logit}(u),\qquad
\mathrm{logit}(u)=\Bigl(\log\frac{u_1}{1-u_1},\dots,\log\frac{u_d}{1-u_d}\Bigr).
\]
Let $q(z)$ be a base density on $\mathbb{R}^d$ (e.g., standard normal) and let $T_\theta:\mathbb{R}^d\to\mathbb{R}^d$ be an invertible map.
Generate
\[
Z=T_\theta(Z_0),\qquad Z_0\sim q,
\]
then map back by $U=\sigma(Z)$ with componentwise sigmoid $\sigma(z)=1/(1+e^{-z})$.
By change of variables,
\[
c_\theta(u)
=
q\bigl(T_\theta^{-1}(z)\bigr)\,
\left|\det \nabla T_\theta^{-1}(z)\right|\,
\prod_{j=1}^d \frac{1}{u_j(1-u_j)},
\qquad z=\mathrm{logit}(u).
\]
A probit map $z_j=\Phi^{-1}(u_j)$ is also common; then
\[
c_\theta^{\mathrm{probit}}(u)
=
q\!\left(T_\theta^{-1}(z)\right)\left|\det\nabla T_\theta^{-1}(z)\right|
\prod_{j=1}^d \frac{1}{\phi(z_j)},
\qquad z_j=\Phi^{-1}(u_j).
\]
Both transforms have diverging Jacobians near $u\to0,1$, so clipping is practically essential.

\paragraph{Learning by flow matching: regression rather than likelihood}\index{flow matching}
Rather than maximizing a copula likelihood, we design a probability path from $q$ to the data distribution in $z$-space
and learn the corresponding velocity field by conditional flow matching:
\begin{enumerate}
\item Construct pseudo-observations $u_i$ (typically by ranks).
\item Transform to $z_i=\mathrm{logit}(u_i)$ (or probit).
\item Design a probability path $\rho_t(\cdot\,|\,z_1)$ and learn the velocity field $v_t(z)$ by conditional flow matching in $z$-space.
\item Sample by solving the learned ODE in $z$-space and mapping back to $u$.
\end{enumerate}
This yields a sampler for the copula (and for conditional distributions induced by the copula)
without requiring evaluation of the normalizing constants in high dimension.

\paragraph{A simple experiment: an S-shaped dependence}
We illustrate with a strongly nonlinear dependence where correlation changes sign across the domain,
which is difficult for simple parametric copulas.
Figure~\ref{fig:copula_fm} shows (left) samples from the target copula and (right) samples from the learned model,
together with KDE-based contour visualization of the copula density and the learned vector field at $t=0.5$.\index{vector field}
The key point is geometric: FM does not directly optimize a density value at each point;
it learns a transport rule that moves simple noise into the target dependence structure,
thereby avoiding normalization issues while capturing complex dependence.

\begin{figure}[htbp]
    \centering
    \includegraphics[width=\textwidth]{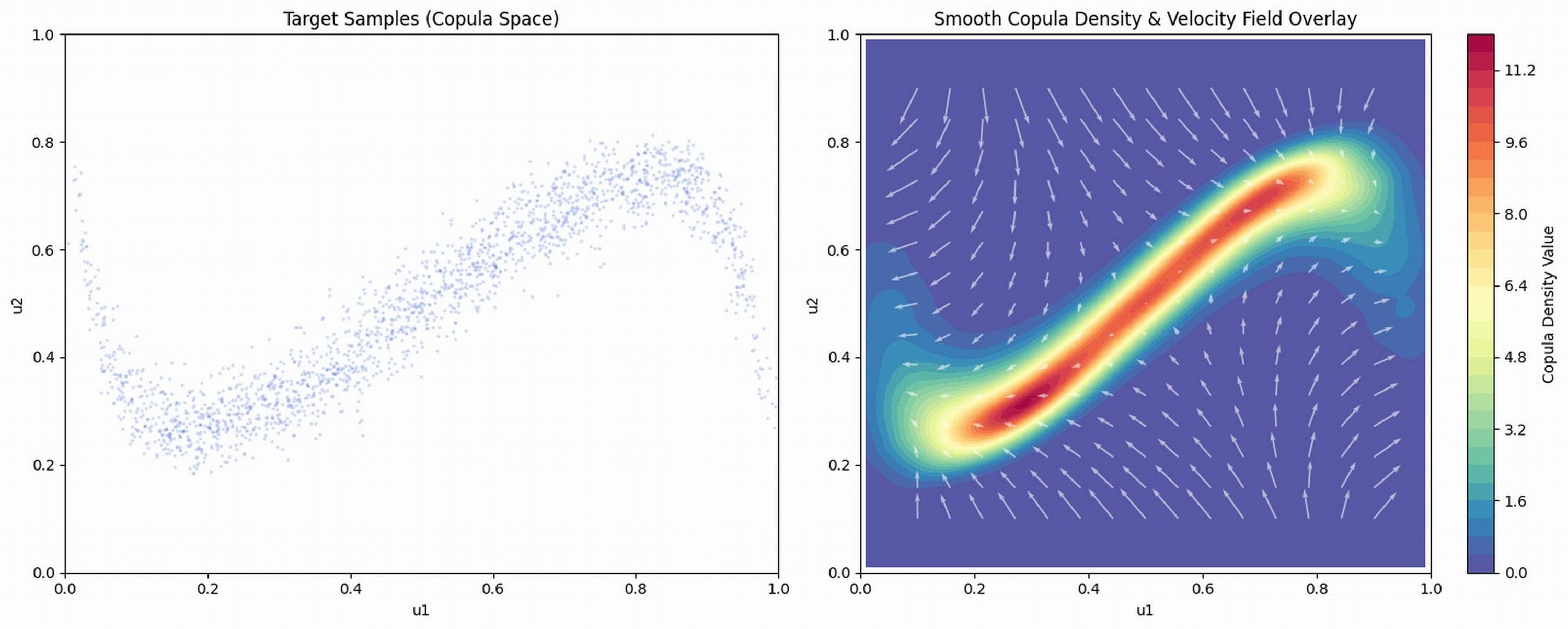}
    \caption{Flow-copula learning by FM. Left: target samples exhibiting an S-shaped dependence. Right: samples from the learned model; KDE-based copula-density contours are shown as filled levels, overlaid with the learned vector field (velocity at $t=0.5$, arrows).}
    \label{fig:copula_fm}
\end{figure}

\paragraph{Semiparametric interpretation}
The construction ``marginals + flow dependence'' is a semiparametric dependence model.
It preserves interpretability of marginal modeling while providing a flexible dependence module.
This is particularly useful when downstream tasks require conditional generation, e.g.,
sampling from $p(x_A\,|\,x_B)$ for completion or imputation.

\subsubsection*{Parametric copulas versus flow copulas}
\label{subsubsec:archimedean-vs-flow}
A classical family is the Archimedean copula, defined as
\begin{equation}
C(u_1,\dots,u_d)=\varphi^{-1}\!\left(\varphi(u_1)+\cdots+\varphi(u_d)\right),
\label{eq:archimedean}
\end{equation}
where $\varphi:[0,1]\to[0,\infty]$ is decreasing and convex with $\varphi(1)=0$ (plus regularity conditions).
Archimedean copulas are attractive because (i) they use low-dimensional parameters,
(ii) dependence summaries such as Kendall's $\tau$ often have clear relationships to parameters,
and (iii) tail dependence properties can be interpreted by family.

However, structural constraints can be strong.
For example, \eqref{eq:archimedean} often implies near-exchangeability, so dependence tends to be symmetric.
In high dimension a single generator must explain all dependence, making anisotropic or locally varying dependence harder to express
without additional constructions (rotations, nesting, vines).

\paragraph{What flow copulas extend}
Flow copulas represent dependence via a map (or vector field) on $\mathbb{R}^d$ rather than a one-dimensional generator,
enabling:
\begin{itemize}
\item pair-specific strength/direction (non-exchangeable dependence),
\item asymmetric tail dependence,
\item local dependence that varies across the domain,
\item dependence that varies with covariates (conditional copulas).
\end{itemize}
In this sense, flows provide a nonparametric dependence module beyond standard parametric copulas.\citep{Joe2014,KamtheAssefaDeisenroth2021,MitskopoulosAmvrosiadisOnken2022,huk2025diffusion}

\paragraph{Why parametric copulas still matter}
From a statistical viewpoint, parametric copulas remain valuable:
they are stable in small samples, interpretable via low-dimensional summaries, and provide explicit structural extrapolation.
Flow copulas require regularization and model selection to control overfitting,
and tail behavior may require additional diagnostics.

\paragraph{Practical recommendation: hybrid designs}
Rather than choosing between parametric and flow copulas, a natural approach is to combine them:
\begin{enumerate}
\item \textbf{Parametric base + flow residual.} Sample $U^{(0)}\sim C_{\alpha}$ (e.g., Archimedean),
then apply an additional deformation $U=\mathcal{T}_\theta(U^{(0)})$.
The parameter $\alpha$ summarizes interpretable dependence, while $\mathcal{T}_\theta$ captures local/asymmetric residual dependence.
\item \textbf{Weak structural constraints on the flow.} If exchangeability or other symmetry is desired, encode it in the parameterization
or add penalties that discourage symmetry violations.
\end{enumerate}
This aims for the trio: stable estimation, interpretable summary, and flexible residual correction.

\paragraph{Model diagnostics}\index{diagnostics}
When comparing flow copulas to parametric baselines, minimum practical diagnostics include:
\begin{itemize}
\item uniform marginals of pseudo-observations and stability near boundaries,
\item agreement in low-dimensional summaries (Kendall's $\tau$, tail dependence),
\item behavior of conditional sampling $p(u_A\,|\,u_B)$, which directly matters for completion/imputation.
\end{itemize}


\section{Survival analysis with censoring}
\index{survival analysis}
\index{AFT model}
\label{subsec:flow-survival}

We next illustrate how flows can be integrated with one of the most central models in statistics:
the Cox proportional hazards model.

\paragraph{Censoring as a missing-data problem for distributions}
In survival analysis, the true event time $T$ is not always observed.
With a censoring time $C$, we observe
\begin{align}\label{transform1}
\tilde{T}=\min(T,C),\qquad \delta=\mathbb{I}(T\le C),
\end{align}
where $\delta=0$ indicates right censoring.
While censoring is usually formulated through likelihood contributions,
from the flow viewpoint it is also a missing-data problem:
we need the conditional distribution of the latent $T$ given the observed truncation event $T>\tilde t$.

\paragraph{Right censoring is conditional generation}\index{censoring}
For censored cases ($\delta=0$), the data provide only $T>\tilde T$.
Completion and prediction therefore rely on sampling from
\[
p(t\,|\,t>\tilde t, X=x)
\ \propto\ 
p(t\,|\,X=x)\,\mathbb{I}(t>\tilde t).
\]
This highlights a natural role for flows:
not necessarily as ``a better density model'' for $T\,|\,X$,
but as a stable \emph{conditional sampler under truncation}.

\paragraph{A monotone conditional time model}
To obtain a coherent one-dimensional conditional distribution, we can model $T\,|\,X=x$ via a monotone transform
\[
T=h_\theta(\epsilon;x),\qquad \epsilon\sim r,
\]
where $\epsilon\mapsto h_\theta(\epsilon;x)$ is increasing.
Then, letting $R$ denote the CDF of $r$,
\begin{align}\notag
F_\theta(t\,|\,x)&=P(T\le t\,|\,x)
=P(\epsilon\le h_\theta^{-1}(t;x))
=R\bigl(h_\theta^{-1}(t;x)\bigr),
\end{align}
and the density is
\[
f_\theta(t\,|\,x)
=
r\bigl(h_\theta^{-1}(t;x)\bigr)\,
\left|\frac{\partial}{\partial t}h_\theta^{-1}(t;x)\right|.
\]
If $h_\theta$ is restricted to simple location/scale forms, this includes accelerated failure time (AFT) models as a special case.

\paragraph{Censoring likelihood}
Under right censoring, the log-likelihood is
\begin{align}\notag
&\ell(\theta)=\sum_{i=1}^n
\left\{
\delta_i \log f_\theta(\tilde{T}_i\,|\,X_i)
+(1-\delta_i)\log S_\theta(\tilde{T}_i\,|\,X_i)
\right\},\\[2mm]\notag
&S_\theta(t\,|\,x)=1-F_\theta(t\,|\,x).
\end{align}
In the monotone-transform representation, evaluating $F_\theta$ and $S_\theta$ reduces to computing the inverse map $h_\theta^{-1}$,
illustrating why invertible flow constructions can be practically attractive.

\paragraph{Cox PH model: counting-process formulation and partial likelihood}
Cox PH is a special case of the transformation model \eqref{transform1}, with hazard
\index{transformation model}
\[
\lambda_\beta(t\,|\,X)=\lambda_0(t)\exp(\beta^\top X),
\qquad
S_\beta(t\,|\,X)=\exp\{-\Lambda_0(t)\exp(\beta^\top X)\},
\]
where $\Lambda_0(t)=\int_0^t\lambda_0(s)\,ds$.
For individual $i$, define the counting process and at-risk indicator
\[
N_i(t)=\mathbb{I}(\tilde T_i\le t,\ \delta_i=1),\qquad
Y_i(t)=\mathbb{I}(\tilde T_i\ge t).
\]
Let $\mathcal{F}_{t-}$ be the natural filtration up to (but not including) time $t$.
With covariate process $X_i(t)$, the Cox model assumes the $\mathcal{F}_{t-}$-intensity
\[
\lambda_i(t\,|\,\mathcal{F}_{t-})
=
Y_i(t)\,\lambda_0(t)\exp\{\beta^\top X_i(t)\},
\]
a semiparametric model for the conditional hazard.\citep{Cox1972,Cox1975,AndersenGill1982,FlemingHarrington1991}

\paragraph{Partial likelihood}
Using only conditional probabilities of ``who fails'' at event times yields the partial likelihood
\[
L_p(\beta)
=
\prod_{i:\delta_i=1}
\frac{\exp\{\beta^\top X_i(\tilde T_i)\}}
{\sum_{j\in\mathcal R_i}\exp\{\beta^\top X_j(\tilde T_i)\}},
\qquad
\mathcal R_i=\{j:\tilde T_j\ge \tilde T_i\},
\]
which eliminates the baseline hazard $\lambda_0(t)$.
This robustness to the nuisance baseline is a key semiparametric strength.

\paragraph{What breaks when PH is violated?}\index{proportional hazards}
If PH fails, the hazard ratio changes over time; a natural representation is
a time-varying coefficient $\beta(t)$ or, more generally, an additional non-PH component $g(t,X)$ in the intensity.
A purely constant-coefficient Cox model can then exhibit systematic miscalibration over time horizons.

\paragraph{Using a flow as a correction module}
We keep $\beta$ as the interpretable target and absorb departures from PH into a nuisance component $g(t,X)$.
A convenient generalized intensity is
\[
\lambda(t\,|\,X)=\lambda_0(t)\exp\{\beta^\top X+g(t,X)\},
\]
where $\eta=(\Lambda_0,g)$ plays the role of nuisance.
To preserve $\sqrt{n}$-valid inference on $\beta$ under flexible estimation of $\eta$,
we use orthogonalized scores and cross-fitting (DDML). The core idea is to construct an estimating equation whose
population moment is \emph{Neyman-orthogonal} with respect to nuisance perturbations.

\paragraph{Orthogonal score (generalized Cox score)}
Write $\Lambda_0(t)=\int_0^t\lambda_0(s)\,ds$ and $\eta=(\Lambda_0,g)$.
For $k=0,1$, define
\[
S_\eta^{(k)}(t;\beta)
=
\mathbb{E}\!\left[
Y(t)\exp\{\beta^\top X+g(t,X)\}\,X^{\otimes k}
\right],
\qquad
\mu_\eta(t;\beta)=\frac{S_\eta^{(1)}(t;\beta)}{S_\eta^{(0)}(t;\beta)}.
\]
Consider the estimating function (Fisher score--type moment)
\begin{align}\nonumber
\Psi(O;\beta,\eta)
=
\int_0^\tau
\Bigl(X-\mu_\eta(t;\beta)\Bigr)
\Bigl[
dN(t)-Y(t)\exp\{\beta^\top X+g(t,X)\}\,d\Lambda_0(t)
\Bigr],
\end{align}
and its population moment $\Psi(\beta,\eta)=\mathbb{E}\{\Psi(O;\beta,\eta)\}$.
We say that $\Psi$ is Neyman-orthogonal in direction $h$ if
\begin{align}\label{result}
\left.
\frac{\partial}{\partial\varepsilon}
\Psi\bigl(\beta_0,\eta_0+\varepsilon h\bigr)
\right|_{\varepsilon=0}=0,
\end{align}
where $\eta_0=(\lambda_0,g_0)$. See Section~\ref{subsec:ddml} for the general DDML discussion.\index{cross-fitting}

\paragraph{Comparison with the PH case ($g\equiv 0$)}
If $g\equiv0$, then
\begin{align}\nonumber
\Psi_{\mathrm{PH}}(O;\beta,\eta)
=
\int_0^\tau
\Bigl(X-\mu_{\mathrm{PH}}(t;\beta)\Bigr)
\Bigl[
dN(t)-Y(t)\exp\{\beta^\top X\}\,d\Lambda_0(t)
\Bigr],
\end{align}
where
\[
\mu_{\mathrm{PH}}(t;\beta)=\frac{\mathbb E[Y(t)\exp(\beta^\top X)X]}{\mathbb E[Y(t)\exp(\beta^\top X)]}.
\]
Thus the difference is entirely the correction term $g(t,X)$.

\begin{assumption}[Basic assumptions (survival analysis)]
\label{ass:surv_basic}
Assume:
\begin{enumerate}
\item i.i.d. sampling: $\{O_i\}_{i=1}^n$ are i.i.d. with $O=(\tilde T,\delta,X)$.
\item Independent censoring: $C\indep T\,|\,X$.
\item Regularity: $\displaystyle\int_0^\tau\lambda_0(t)\,dt<\infty$, and $\exp\{\beta^\top X+g(t,X)\}$ is integrable over $t\in(0,\tau)$.
\item Bounded covariates and positive risk: $\|X\|$ is bounded and there exists $c>0$ such that
\[
\inf_{t\in(0,\tau)}S_{\eta_0}^{(0)}(t;\beta_0)\ge c.
\]
\end{enumerate}
\end{assumption}

\begin{definition}
\label{def:beta_target}
The target parameter $\beta_0$ is defined as the solution to
\[
\Psi(\beta,\eta_0)=0
\]
under the true nuisance $\eta_0=(\lambda_0,g_0)$.
We further assume uniqueness under an identification (orthogonality) constraint.
\end{definition}

\begin{proposition}[Neyman orthogonality (centered score)]
\label{prop:orthogonality}
Under Assumption~\ref{ass:surv_basic}, suppose the estimating equation $\Psi(\beta,\eta)$ is defined with centering by the risk-set mean.
Then, for any perturbation direction $h$,
\[
\left.\frac{d}{dr}\Psi(\beta_0,\eta_0+r h)\right|_{r=0}=0.
\]
That is, $\Psi$ is Neyman-orthogonal with respect to $\eta$.
\end{proposition}
(The proof is technical and may be skipped on a first reading.)
\begin{proof}
Perturb the nuisance parameter in direction $h=(\lambda, g)$:
\[
\lambda_\varepsilon(t) = \lambda_0(t) + \varepsilon \lambda(t), \quad g_\varepsilon(t,x) =g_0(t,x)+  \varepsilon g(t,x).
\]
Let $\Psi(\varepsilon) := \Psi(\beta_0, \eta_\varepsilon) = \mathbb{E}[\Psi(O; \beta_0, \eta_\varepsilon)]$ and differentiate at $\varepsilon=0$.
The nuisance $\eta$ enters through (i) the centering term $\mu_\eta$ and (ii) the compensator involving $\lambda,g$, so the derivative decomposes as
\[
\left.\frac{\partial}{\partial\varepsilon} \Psi(\varepsilon)\right|_{\varepsilon=0} = I_1 + I_2.
\]
The term $I_1$ is the contribution from the variation of $\mu_\eta$:
\[
I_1 = \mathbb{E}\left[ \int_0^\tau \left( - \left.\frac{\partial}{\partial\varepsilon}\mu_{\eta_\varepsilon}(t;\beta_0)\right|_{\varepsilon=0} \right) \{dN(t) - Y(t) e^{\beta_0^\top X + g_0(t,X)} d\Lambda_0(t)\} \right].
\]
The bracketed term is deterministic and the curly bracket is the true martingale increment $dM_0(t)$, so the expectation vanishes.

For $I_2$,
\[
I_2 = -\mathbb{E}\left[ \int_0^\tau (X - \mu_{\eta_0}(t;\beta_0)) \frac{\partial}{\partial\varepsilon} \left( Y(t) e^{\beta_0^\top X + g_\varepsilon(t,X)} d\Lambda_\varepsilon(t) \right)_{\varepsilon=0} \right].
\]
Decompose into the $\lambda$- and $g$-parts.
\begin{itemize}
\item[(a.)] The $\lambda$-part: since $d\Lambda_\varepsilon = \lambda_0 + \varepsilon \lambda$,
\begin{align}\notag
I_{2,\Lambda} &= -\mathbb{E}\left[ \int_0^\tau (X - \mu_{\eta_0}) Y(t) e^{\beta_0^\top X + g_0(t,X)} \lambda(t) dt \right] \\[3mm]
&= -\int_0^\tau \lambda(t) {\mathbb{E}\left[ Y(t) e^{\beta_0^\top X + g_0(t,X)} (X - \mu_{\eta_0}(t;\beta_0)) \right]} dt.\label{expect}
\end{align}
By the definition $\mu_{\eta_0}=S_{\eta_0}^{(1)}/S_{\eta_0}^{(0)}$, the expectation in \eqref{expect} equals
\[
S_{\eta_0}^{(1)}(t;\beta_0)-\mu_{\eta_0}(t;\beta_0)S_{\eta_0}^{(0)}(t;\beta_0)=0,
\]
so $I_{2,\Lambda}=0$.
\item[(b.)] The $g$-part:
\[
\begin{aligned}
I_{2,g} &= -\mathbb{E}\left[ \int_0^\tau (X - \mu_{\eta_0}) Y(t) e^{\beta_0^\top X+ g_0(t,X) } g(t,X) d\Lambda_0(t) \right] \\[3mm]
&= -\int_0^\tau d\Lambda_0(t)\,\mathbb{E}\left[ Y(t) e^{\beta_0^\top X+ g_0(t,X) } (X - \mu_{\eta_0}(t;\beta_0)) g(t,X) \right].
\end{aligned}
\]
This expectation is a weighted covariance between $X$ and $g$ within the risk set.
Under the identification constraint that $g$ lives in a residual space orthogonal to the linear span of $X$,
the integrand vanishes, yielding $I_{2,g}=0$.
\end{itemize}
Hence $I_1=0$ and $I_2=0$, proving \eqref{result} and Neyman orthogonality.
\end{proof}

\paragraph{Local robustness}
The intuition is that centering by $\mu_\eta(t;\beta)$ cancels first-order nuisance perturbations under the risk-set weighted measure.
This generalizes the classical fact that Cox partial likelihood eliminates $\lambda_0$:
the same principle extends to small departures from PH encoded in $g$.
Thus, using Cox as a base model and learning $g$ flexibly (e.g., by a flow) need not destroy $\sqrt{n}$ inference on $\beta$
when the departure is moderate.

\begin{corollary}[Local efficiency (reduction to Cox score)]
\label{cor:efficiency}
If the true model is proportional hazards, $g_0\equiv 0$, then with $\eta$ fixed at the truth,
$\Psi(O;\beta,\eta)$ coincides with the efficient score of classical Cox partial likelihood:
\[
U_{\mathrm{Cox}}(\beta)
=
\int_0^\tau
\left\{
X-\bar X(\beta,t)
\right\}\,dN(t),
\;
\bar X(\beta,t)
=
\frac{\mathbb{E}\!\left[Y(t)\exp\{\beta^\top X\}X\right]}
{\mathbb{E}\!\left[Y(t)\exp\{\beta^\top X\}\right]}.
\]
Therefore $\widehat\beta$ attains the semiparametric efficiency bound under Cox PH.
\end{corollary}

\paragraph{Related work}
Orthogonal-score constructions in survival analysis are developed in several strands of the modern literature
(e.g., time-varying effects, nuisance learning, causal survival forests). We refer the reader to the references in the Japanese draft for a detailed map.

\paragraph{Summary}
Orthogonalized DDML-style estimating equations recover Cox efficiency when PH is true,
remain locally robust when PH approximately holds,
and allow flexible learning of non-PH components $g_0$ while preserving interpretability of $\beta_0$.

\subsection*{Real-data comparison}
\label{sec:realdata_ph_vs_flow}

\paragraph{Goal}
We compare settings where the proportional hazards (PH) assumption is approximately valid and where it is violated.
We evaluate (i) interpretability of the Cox estimate $\widehat\beta$ and
(ii) the impact of a minimal time-varying correction on calibration and prediction error.
We call the extended model ``Cox+TV'': a Cox base model with a minimal time-varying effect to capture PH violations,
designed to shrink back to Cox when PH holds.

\paragraph{Using the \texttt{survival} package in R.}
Our baseline implementation uses the standard R package \texttt{survival} \citep{TherneauGrambsch2000},
which provides Cox fitting (\texttt{coxph}), PH diagnostics (\texttt{cox.zph}), and related tools.

We use three benchmark right-censored data sets included in the package:
\begin{description}
\item[\texttt{lung}:] NCCTG advanced lung cancer data ($n=228$), with 165 events and right censoring.
Covariates include age, sex, and performance-status measures (e.g., \texttt{ph.ecog}, \texttt{pat.karno}).
\item[\texttt{pbc}:] Mayo Clinic primary biliary cirrhosis trial data, with clinical biomarkers and censoring.
\item[\texttt{veteran}:] Veteran's Administration lung cancer trial data, often exhibiting stronger PH violations for certain covariates.
\end{description}

\paragraph{Base model and PH diagnostics}
We fit a Cox model and test PH using \texttt{cox.zph}.
We then choose a single covariate with the strongest evidence of PH violation and add a time-varying effect for that covariate.
This yields the ``minimal'' correction that is easy to interpret.

\paragraph{Extended model (Cox+TV)}
We use a one-degree-of-freedom time-varying effect of the form
$\beta(t)=\beta+\gamma\log t$ for a selected covariate (notation suppressed for brevity).
This can be viewed as a low-dimensional proxy for $g(t,X)$.

\paragraph{Prediction metrics}
For a horizon $t_0$, let $\widehat R(t_0\,|\,x)$ denote predicted risk.
Calibration curves are constructed by binning by predicted risk and estimating observed risk in each bin via Kaplan--Meier.

We compute the inverse probability of censoring weighted (IPCW) Brier score.
Let $Y_i(t_0)=\mathbb{I}\{T_i\le t_0,\delta_i=1\}$ and let $\widehat G$ denote the estimated censoring survival function.
With weights
\[
w_i(t_0)
=
\frac{\mathbb{I}(\tilde T_i\ge t_0)}{\widehat G(t_0\,|\,x_i)}
+
\frac{\mathbb{I}(\tilde T_i<t_0,\ \delta_i=1)}{\widehat G(\tilde T_i\,|\,x_i)},
\]
the Brier score is
\[
\mathrm{Brier}(t_0)
=
\frac{1}{n}\sum_{i=1}^n
w_i(t_0)\Bigl(Y_i(t_0)-\widehat R(t_0\,|\,x_i)\Bigr)^2.
\]

\paragraph{Summary of results}
Table~\ref{tab:ph_flow_summary} summarizes PH diagnostics and prediction comparisons.
For \texttt{lung}, the global PH $p$-value is large, the average risk difference between Cox and Cox+TV is tiny,
and the IPCW Brier score is essentially unchanged, indicating shrinkage back to Cox.
For \texttt{veteran}, where PH is strongly violated, the time-varying correction is non-negligible,
improving both calibration and Brier score.
For \texttt{pbc}, improvement at the chosen horizon is limited; stronger time bases (e.g., splines) may be needed.

\begin{table}[t]
\centering
\caption{PH diagnostics and Cox vs Cox+TV summary (evaluation data).}
\label{tab:ph_flow_summary}
\footnotesize
\setlength{\tabcolsep}{3.2pt}
\renewcommand{\arraystretch}{1.08}
\begin{tabular}{lrrrrrr}
\hline
dataset & $t_0$ & PH $p$ & tv & mean$|\Delta R|$ & Br.(Cox) & Br.(Flow) \\
\hline
lung    & 210.0  & 0.545    & ph.ecog & 0.00348 & 0.21587 & 0.21578 \\
pbc     & 1180.5 & 0.118    & bili    & 0.00341 & 0.06763 & 0.06828 \\
veteran & 93.5   & 2.17e-4  & karno   & 0.01807 & 0.16758 & 0.16476 \\
\hline
\end{tabular}
\end{table}

\begin{figure}[t]
\centering
\includegraphics[width=0.95\textwidth]{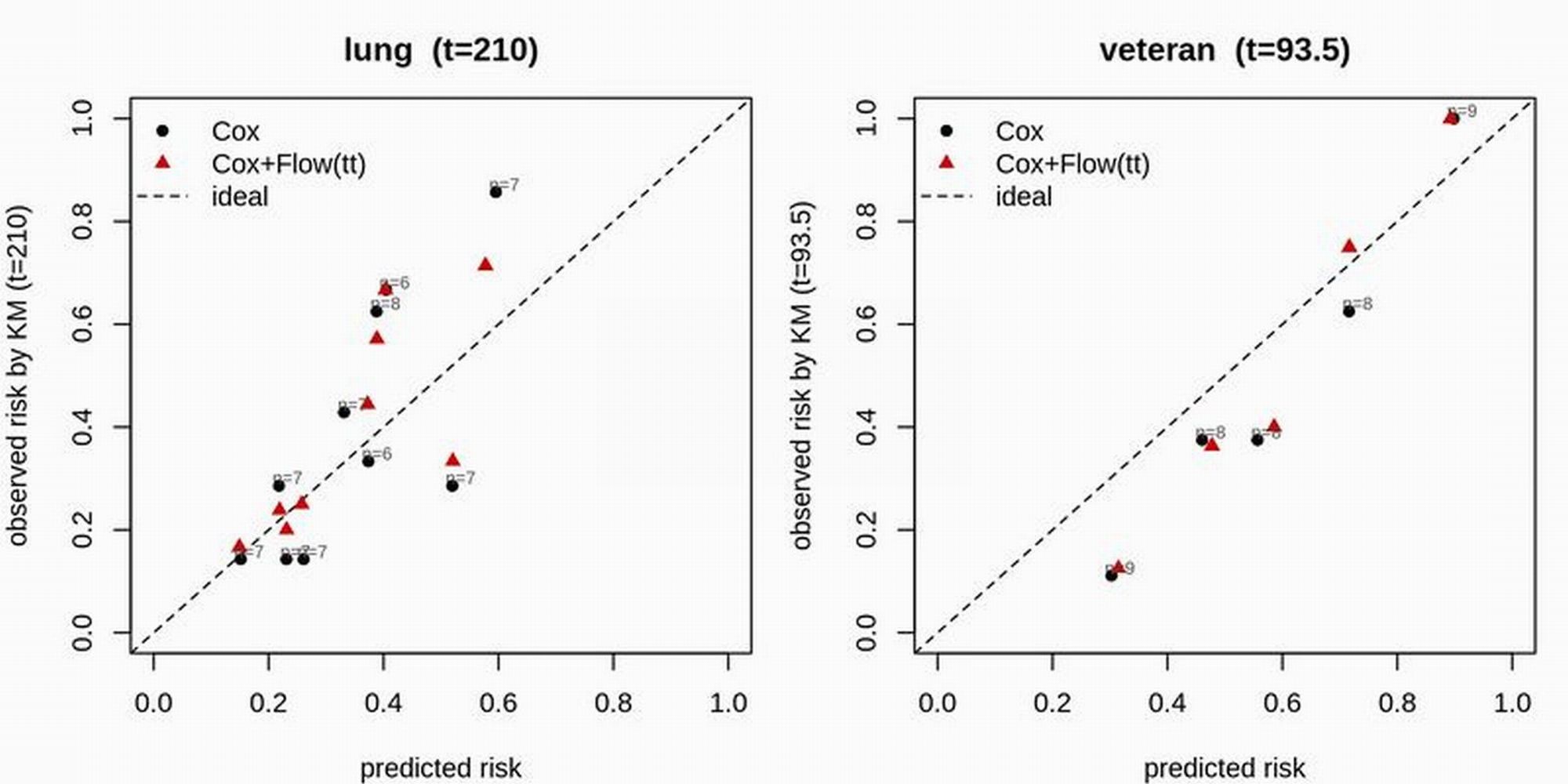}
\caption{
Calibration curves for a PH-good case (\texttt{lung}) and a PH-bad case (\texttt{veteran}).
The horizontal axis is predicted risk at horizon $t_0$; the vertical axis is observed risk estimated by Kaplan--Meier within bins defined by Cox-predicted risk.
Each label indicates the bin sample size $n$.
When PH holds, Cox and Cox+TV nearly coincide; when PH is violated, Cox+TV moves toward the ideal line.
}
\label{fig:calib_good_bad}
\end{figure}

\paragraph{Takeaway for statistical practice}
The main benefit of bringing flows into survival analysis is not ``a new estimator of $\beta$'',
but a reliable conditional sampler for truncated distributions and a modular way to correct misspecification
without sacrificing interpretability.
In particular, flows facilitate:
\begin{itemize}
\item natural completion for censored cases by sampling $T\,|\,(T>\tilde T, X)$,
\item flexible time-dependent structure as a nuisance component,
\item downstream uncertainty quantification via simulation-based procedures.
\end{itemize}


\section{Multiple imputation via conditional flows}
\index{missing data}
\index{MAR}
\index{MNAR}
\index{multiple imputation}
\index{MICE}
\label{subsec:flow-missing-mi}

Missing data are a central statistical problem.
Assume an underlying complete-data vector $X\in\mathbb{R}^d$ and a missingness indicator
$M\in\{0,1\}^d$ (where $M_j=1$ means missing), yielding observed and missing parts
$X_{\mathrm{obs}}$ and $X_{\mathrm{mis}}$.
The inferential goal is to estimate targets (regression coefficients, means, causal effects, etc.) validly under missingness.

What we truly need is not a single ``best guess'' fill-in for $x_{\mathrm{mis}}$,
but the conditional distribution
\[
p(x_{\mathrm{mis}}\,|\,x_{\mathrm{obs}})
\]
(or samples from it), because downstream inference and uncertainty depend on this distribution.
Thus the core of missing-data analysis is reconstructing the \emph{distributional shape} of missing values,
not merely predicting them.

Rubin's framework (MCAR/MAR/MNAR) and MI provide a unified solution.\citep{Rubin1976, Rubin1987, LittleRubin2019}
The widely used MICE algorithm approximates $p(x_{\mathrm{mis}}\,|\,x_{\mathrm{obs}})$ via chained conditional regressions.
However, in high-dimensional, nonlinear, or multi-modal settings,
chained regressions can distort distributional shape (e.g., collapsing a multi-modal conditional distribution to unimodal).
In our perspective, conditional flow matching directly learns a sampler
\[
X_{\mathrm{mis}}\sim p(\cdot\,|\,x_{\mathrm{obs}})
\]
that preserves multimodality, nonlinear dependence, and heavy tails.

\paragraph{Flows for missing-data analysis}
From the flow viewpoint, missingness can be handled in two natural ways:
\begin{enumerate}
\item Model the complete-data distribution $p(x)$, then generate from $p(x_{\mathrm{mis}}\,|\,x_{\mathrm{obs}})$.
\item Condition on the missingness pattern $m$ and directly learn $p(x_{\mathrm{mis}}\,|\,x_{\mathrm{obs}},m)$ (conditional generation).
\end{enumerate}
In practice, the second approach is often more stable and aligns with the ``learn by regression'' philosophy.

\paragraph{Under MAR: conditional distributions are the essence}
Assume MAR:
\[
P(M\,|\,X)=P(M\,|\,X_{\mathrm{obs}}).
\]
Then valid inference fundamentally relies on $p(x_{\mathrm{mis}}\,|\,x_{\mathrm{obs}})$.
MI draws multiple samples from this conditional distribution, creates $K$ completed data sets,
and combines inferential outputs using Rubin's rules (as reviewed in Section~5.2).

\paragraph{Flow-based conditional imputation}
A flow-based imputer can be represented as a conditional generator
\[
X_{\mathrm{mis}} = G_\theta(\epsilon; X_{\mathrm{obs}}, M),\qquad \epsilon\sim \pi,
\]
trained so that the induced conditional distribution matches the data.
In flow-matching terms, we design a probability path from a base noise distribution to the conditional target
and learn the conditional velocity field by regression.
After training, we generate each imputed version by solving the conditional ODE and sampling at $t=1$.

\paragraph{Caution}
As in any flexible conditional generation, regularization and diagnostics are important:
we must check whether generated imputations preserve key low-dimensional summaries and whether they respect data constraints.

\paragraph{MI pipeline: connecting to Rubin's rules}
The flow component replaces only the \emph{imputation engine};
the combination step remains standard:
\begin{enumerate}
\item For $k=1,\dots,K$, generate imputed data $X_{\mathrm{mis}}^{(k)}$ from the conditional flow sampler to obtain a completed data set.
\item Compute $\widehat\theta^{(k)}$ and $\widehat V^{(k)}$ from each completed data set using the intended complete-data analysis.
\item Combine $(\widehat\theta^{(k)},\widehat V^{(k)})$ using Rubin's rules.
\end{enumerate}
This modularity is important: flows do not replace statistical procedures; they provide the conditional sampling component those procedures require.

\paragraph{MNAR and DRO: sensitivity analysis via ambiguity sets}
Under MNAR, identification is fundamentally impossible without additional assumptions.
A modern approach is to formulate uncertainty as an \emph{ambiguity set} and consider worst-case analysis
via distributionally robust optimization (DRO).\citep{Duchi2021,Blanchet2019, Muzellec2020}
One convenient formulation uses a KL-ball around a nominal conditional model $p_\theta(\cdot\,|\,x_{\mathrm{obs}},m)$:
\[
\sup_{q:\ \mathrm{KL}(q\|p_\theta)\le \rho}\ \mathbb{E}_q\!\left[\ell(X_{\mathrm{mis}},x_{\mathrm{obs}},m)\right].
\]
The optimizer takes an exponential-tilting form
\[
q_{\eta}(x_{\mathrm{mis}}\,|\,x_{\mathrm{obs}},m)\ \propto\ 
p_\theta(x_{\mathrm{mis}}\,|\,x_{\mathrm{obs}},m)\exp\{\eta\,\ell(x_{\mathrm{mis}},x_{\mathrm{obs}},m)\},
\]
with normalizer
\[
A(\eta)=\log \mathbb{E}_{p_\theta}\!\left[\exp\{\eta\,\ell(X_{\mathrm{mis}},x_{\mathrm{obs}},m)\}\right].
\]
Then $\mathrm{KL}(q_\eta\|p_\theta)=\eta A'(\eta)-A(\eta)$, and the active-constraint solution $\eta^\ast\ge0$ satisfies
\[
\eta^*A'(\eta^*)-A(\eta^*)=\rho.
\]
In practice, Monte Carlo from the flow-based $p_\theta(\cdot\,|\,x_{\mathrm{obs}},m)$ can approximate $A(\eta)$ and $A'(\eta)$,
and one can solve for $\eta^\ast$ by bisection or Newton's method.

\subsection*{Numerical experiment}
\label{subsec:mi-mice-vs-fm-demo}

\paragraph{Goal}
We visualize that (i) when the conditional distribution is multi-modal, MICE tends to collapse distributional shape,
whereas (ii) conditional flow matching can preserve it in MI, and we examine the impact on regression inference.

\paragraph{Experimental setup}
Let $X=(X_1,X_2,X_3)^\top\in\mathbb{R}^3$ and generate $n=3000$.
We draw
\[
X_1 \sim \mathcal{N}(0,1),\qquad
X_2 \sim \mathcal{N}(0,1),
\]
and let $X_3$ be generated from a nonlinear bimodal conditional mechanism (details follow the Japanese source),
so that $p(x_3\,|\,x_1,x_2)$ has two modes in a substantial region.

\begin{enumerate}
\item{Missingness mechanism (MAR).}
We impose missingness only on $X_3$:
$M_3\in\{0,1\}$ is generated under a MAR mechanism depending on $(X_1,X_2)$.

\item{Two MI engines to compare.}
We compare:
\begin{itemize}
\item {MICE:} chained equations with standard parametric conditional regressions.
\item {FM:} conditional flow matching used as a sampler for $p(x_3\,|\,x_1,x_2,M_3)$.
\end{itemize}

\item{ Covariate fairness.}
To make the comparison meaningful, we include the same set of predictors/auxiliary variables in both approaches.

\item{Evaluation metrics.}
We evaluate (i) distributional reconstruction for missing rows (Wasserstein $W_1$, histogram shape),
(ii) prediction summaries such as Brier score for $\mathbb{I}\{X_3>0\}$, and
(iii) regression inference using Rubin's MI combination.
\end{enumerate}
\paragraph{Results: MICE collapses shape, FM preserves it}
Table~\ref{tab:mi-demo-result} summarizes distributional reconstruction and regression inference.
Figure~\ref{fig:mi-demo-hist} visualizes the pooled imputed distribution of $X_3$ among missing rows:
MICE yields a unimodal shrinkage, while FM preserves bimodality.

\begin{table}[t]
\centering
\caption{Reconstruction of $X_3$ for missing rows and regression inference (Rubin combination).
Since $X$ is standardized before analysis, the ``true'' coefficients refer to standardized values.}
\label{tab:mi-demo-result}
\begin{tabular}{lcc}
\hline
Metric & MICE & FM \\
\hline
Missingness rate $P(M_3=1)$ & \multicolumn{2}{c}{$0.359$} \\
RMSE (mean imputation) & $0.5387$ & $0.4955$ \\
$W_1$ (distribution distance) & $0.2628$ & $0.0553$ \\
Brier ($\mathbb{I}\{X_3>0\}$) & $0.0674$ & $0.0624$ \\
\hline
$\widehat{\beta}_0$ (SE) & $-0.0750\ (0.0200)$ & $-0.0685\ (0.0192)$ \\
$\widehat{\beta}_1$ (SE) & $0.4600\ (0.0242)$  & $0.4452\ (0.0237)$ \\
$\widehat{\beta}_2$ (SE) & $0.7640\ (0.0214)$  & $0.7820\ (0.0224)$ \\
$\widehat{\beta}_3$ (SE) & $1.4617\ (0.0237)$  & $1.4533\ (0.0231)$ \\
True (standardized) & $0,\ 0.4971,\ 0.7892,\ 1.4505$ & --- \\
\hline
\end{tabular}
\end{table}

\begin{figure}[t]
\centering
\includegraphics[width=0.55\textwidth]{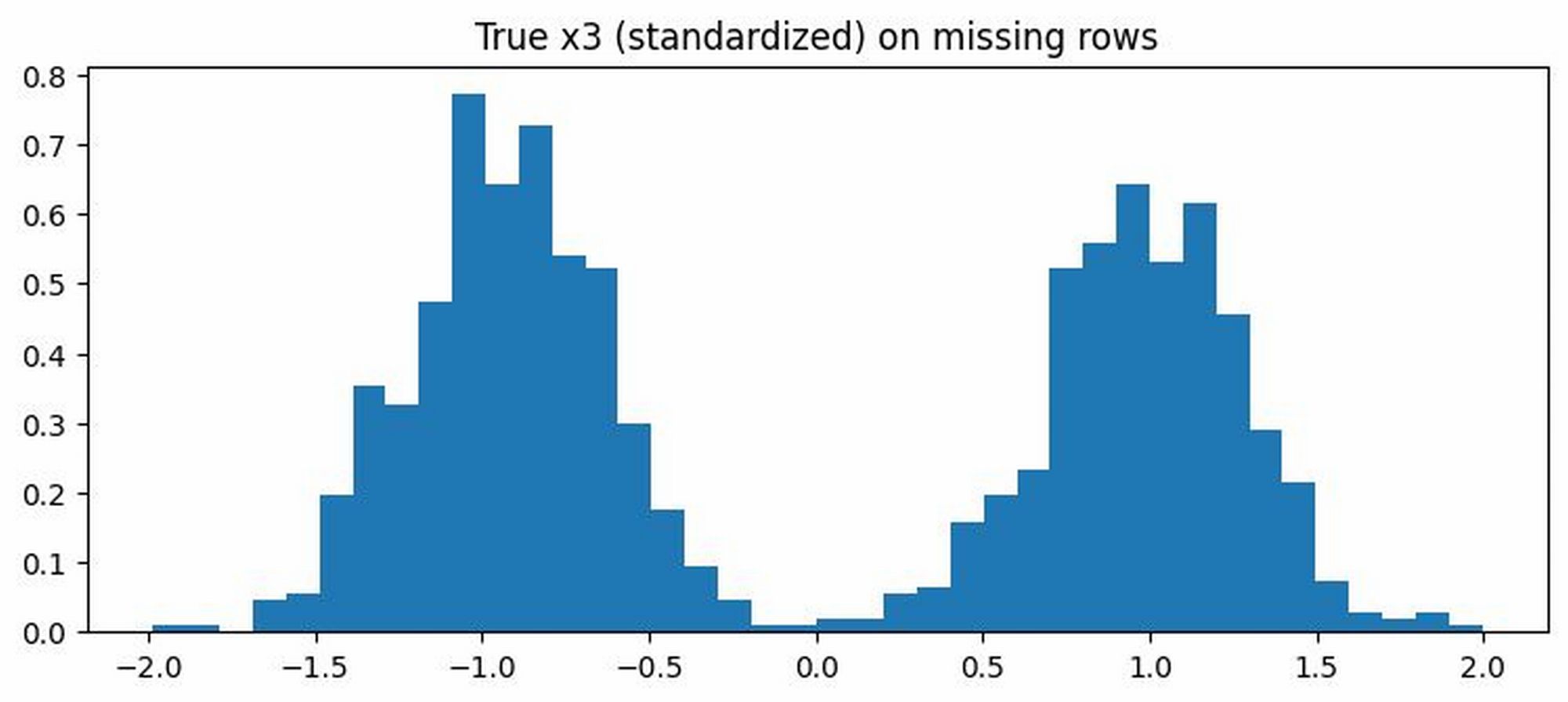}\\[3.5mm]
\includegraphics[width=0.55\textwidth]{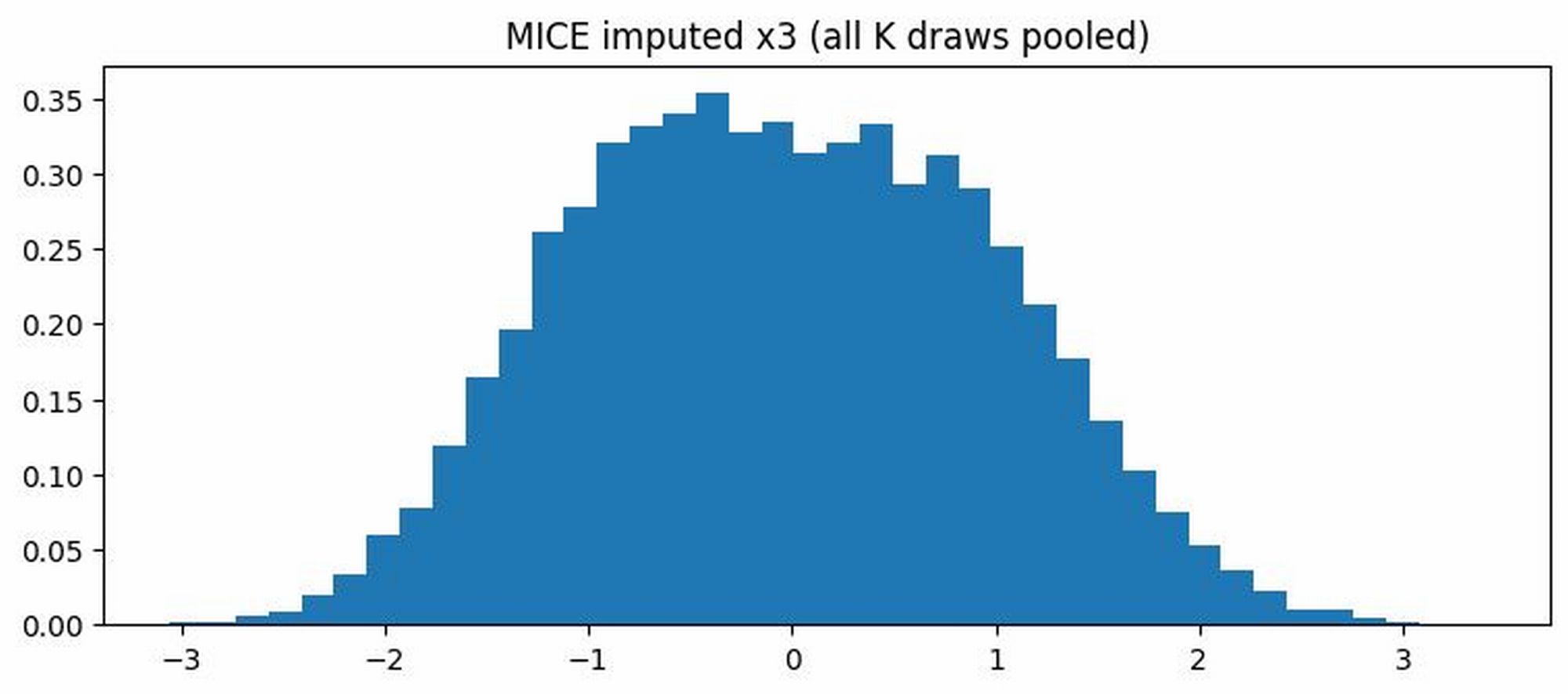}\\[3.5mm]
\includegraphics[width=0.55\textwidth]{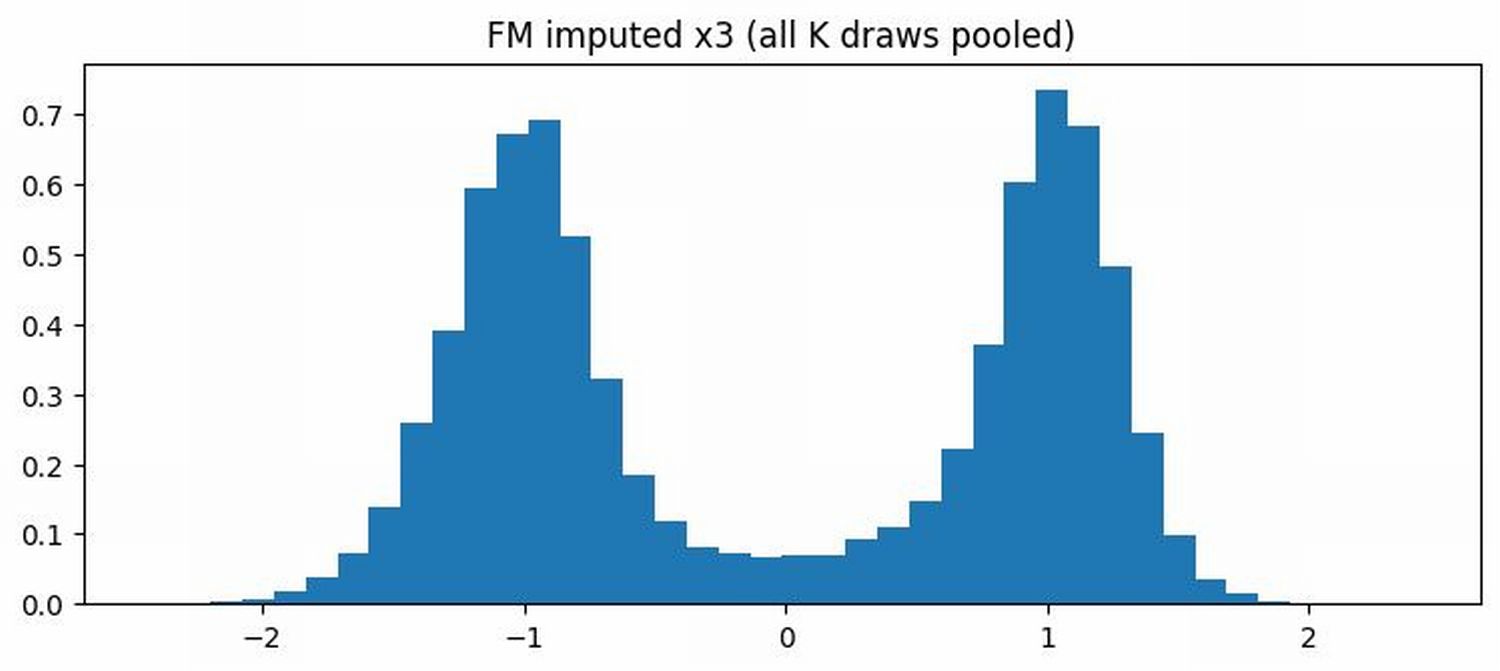}\\[3mm]
\caption{Distribution of $X_3$ among missing rows ($M_3=1$).
Top: true $X_3$ (bimodal). Middle: pooled MI distribution from MICE (collapses toward unimodal).
Bottom: pooled MI distribution from FM-based imputations (bimodality preserved).}
\label{fig:mi-demo-hist}
\end{figure}

\paragraph{Impact on inference}
Even when point estimates of regression coefficients are similar,
distorting the conditional distribution can affect uncertainty quantification and downstream decision-making.
Preserving distributional shape is particularly important when conditional distributions are multi-modal.

\paragraph{Conclusion}
This experiment illustrates that FM can function as a sampler for
\[
p(x_{\mathrm{mis}}\,|\,x_{\mathrm{obs}},M),
\]
and can outperform chained-equation regressions in regimes where the conditional distribution is complex (e.g., multimodal).
Because Rubin's combination remains unchanged, FM naturally integrates into standard MI pipelines
as a drop-in conditional-generation component.


\chapter{Causal Inference: Counterfactual Distributions and Double Robustness}
\index{causal inference}
\index{counterfactual}
\index{double machine learning}
\index{CATE}
\index{Neyman orthogonality}

\section{Potential Outcomes and Structural Causal Models}
\index{potential outcomes}
\index{structural causal model}
\index{SUTVA}
\label{subsec:po_scm}

This chapter follows a single three-step storyline:
(i) \emph{identification} (express interventional distributions in terms of the observational distribution),
(ii) \emph{distribution transformation} (implement the identified counterfactual distribution in a sampleable form),
and (iii) \emph{inference} (orthogonalize nuisance estimation error to enable $\sqrt{n}$-valid inference).
We first review identification using potential outcomes and SCMs,
then position \emph{causal transport} as a ``generator of counterfactual distributions,''
and finally connect to stable inference via DDML.

The goal is to move one step beyond associations in observational data and to formalize what would happen
under an intervention.
We introduce Rubin's potential-outcome framework and Pearl's structural causal model (SCM),
and make their correspondence explicit.
Subsequent sections connect this identification view to transport/generative models (including flow matching)
and to orthogonalized semiparametric inference (DDML).

\paragraph{Observational distribution and intervention}
Let $A\in\{0,1\}$ be a treatment (intervention), $Y\in\mathbb{R}$ an outcome, and $X\in\mathcal{X}$ covariates.
We observe
\[
(X_i,A_i,Y_i),\qquad i=1,\dots,n,
\]
generated from an observational distribution $p(x,a,y)$.
Causal inference is not only about describing observational conditionals such as $p(y\,|\,a,x)$ or marginals such as
$p(y\,|\,a)$, but about estimating the distribution induced by forcing $A$ to a fixed value $a$,
\[
p(y\,|\,\mathrm{do}(A=a)).
\]
Here $\mathrm{do}(A=a)$ is the intervention operator: it severs the assignment mechanism and sets $A$ externally to $a$.
One may also consider covariate-conditional interventional distributions $p(y\,|\,\mathrm{do}(A=a),x)$.

\paragraph{Potential outcomes}
In Rubin's framework, the outcome under intervention $A=a$ is defined as a potential outcome $Y(a)$.
With binary treatment, each unit has two conceptual values $Y(0)$ and $Y(1)$, but only one is observed.
We assume
\begin{equation}
Y = Y(A)
\label{eq:consistency}
\end{equation}
(\emph{consistency}).
This typically comes together with SUTVA (Stable Unit Treatment Value Assumption),
which rules out interference across units.
When interference exists (networks, infectious diseases, spatial spillovers), extensions are needed.

The average treatment effect (ATE) is
\index{ATE}
\begin{equation}
\tau = \mathbb{E}\bigl[Y(1)-Y(0)\bigr],
\label{eq:ate_def}
\end{equation}
and the conditional average treatment effect (CATE) is
\begin{equation}
\tau(x) = \mathbb{E}\bigl[Y(1)-Y(0)\,|\,X=x\bigr].
\label{eq:cate_def}
\end{equation}
Identification from observational data is nontrivial because $Y(0)$ and $Y(1)$ are never jointly observed.
Assumptions on the relationship between assignment and potential outcomes are required.

\paragraph{Identification assumptions: exchangeability and overlap}
A standard pair of assumptions is as follows.

\noindent{(A1). Conditional independence (unconfoundedness)}
\begin{equation}
\bigl(Y(0),Y(1)\bigr)\ \indep\ A \,|\,X.
\label{eq:unconf}
\end{equation}
This asserts that, conditional on $X$, treatment assignment is independent of potential outcomes;
informally, all confounding is captured in $X$.

\noindent{(A.2). Overlap (positivity)}
\begin{equation}
0 < P(A=1\,|\,X=x) < 1\qquad (\text{for almost every }x).
\label{eq:positivity}
\end{equation}
This requires that both treatment arms occur with positive probability in every covariate stratum.

These assumptions are generally not testable from the observational distribution alone;
their plausibility must be assessed using study design, subject-matter knowledge, and sensitivity analyses.

\paragraph{Consequence}
Under \eqref{eq:consistency}--\eqref{eq:positivity}, the mean of the potential outcome is identified by
\begin{align}
\mathbb{E}\bigl[Y(a)\bigr]
&=
\mathbb{E}\bigl[\mathbb{E}[Y(a)\,|\,X]\bigr]
=
\mathbb{E}\bigl[\mathbb{E}[Y(a)\,|\,A=a,X]\bigr] \notag\\[3mm]
&=
\mathbb{E}\bigl[\mathbb{E}[Y\,|\,A=a,X]\bigr],
\label{eq:gcomp_mean}
\end{align}
so that
\begin{equation}
\tau
=
\mathbb{E}\bigl[\mu_1(X)-\mu_0(X)\bigr],
\qquad
\mu_a(x)=\mathbb{E}[Y\,|\,A=a,X=x].
\nonumber
\end{equation}
This is the regression (outcome-regression) form, in which $\mu_a$ appears as a nuisance function.
\index{nuisance}

Let the propensity score be
\[
e(x)=P(A=1\,|\,X=x).
\]
Then the IPW (inverse-probability weighting) identity is
\begin{equation}
\mathbb{E}[Y(1)]
=
\mathbb{E}\biggl[\frac{A Y}{e(X)}\biggr],
\qquad
\mathbb{E}[Y(0)]
=
\mathbb{E}\biggl[\frac{(1-A) Y}{1-e(X)}\biggr].
\nonumber
\end{equation}
This representation depends on $e$ but not on $\mu_a$.
Double robustness arises by combining both $\mu_a$ and $e$, as we discuss later.

\paragraph{Structural causal models}
Pearl's framework describes the data-generating mechanism via structural equations.
For simplicity, introduce exogenous variables $U_X,U_A,U_Y$ and write
\begin{equation}
X = f_X(U_X),\qquad
A = f_A(X,U_A),\qquad
Y = f_Y(A,X,U_Y).
\label{eq:scm_basic}
\end{equation}
Typically the exogenous variables are assumed independent (or to have a known dependence structure),
and causal directions are encoded by the functional inputs.
Equation \eqref{eq:scm_basic} induces a directed acyclic graph (DAG),
with edges such as $X\to A\to Y$ and $X\to Y$.

\paragraph{Definition of intervention}
An intervention $\mathrm{do}(A=a)$ is defined as replacing the structural equation for $A$ by the constant assignment,
\[
A=f_A(X,U_A)
\quad \text{replaced by}\quad
A\equiv a,
\]
and defining $p(y\,|\,\mathrm{do}(A=a))$ as the distribution of $Y$ in the modified model.
This makes precise the idea of ``cutting'' the assignment mechanism.

\paragraph{Definition of counterfactual}
The counterfactual random variable $Y_a$ in an SCM (the value of $Y$ under $A\equiv a$)
is the $Y$ generated by the intervened model.
In this sense,
\begin{equation}
Y(a)\ \equiv\ Y_a,
\label{eq:po_scm_link}
\end{equation}
so potential outcomes can be identified with SCM counterfactual variables.
Rubin's and Pearl's languages are therefore complementary rather than contradictory:
they differ mainly in what they take as primitive objects.

\paragraph{Backdoor adjustment and the $g$-formula}
DAG-based identification yields formulas expressing interventional distributions in terms of observational ones.
In the simplest case, when a set of covariates $X$ blocks all backdoor paths from $A$ to $Y$,
\begin{equation}
p(y\,|\,\mathrm{do}(A=a))
=
\int p(y\,|\,A=a,X=x)\,p(x)\,dx
\label{eq:backdoor}
\end{equation}
holds.
This is the distributional $g$-formula; \eqref{eq:gcomp_mean} is its expectation form.
This viewpoint is crucial below:
causal inference is about estimating the distribution transformation induced by $\mathrm{do}$,
and transport/generative models can be used to implement that transformation.

\paragraph{Assumption violations and modeling implications}
In real applications, \eqref{eq:unconf} may fail.
With unmeasured confounding, $p(y\,|\,\mathrm{do}(A=a))$ is not uniquely determined from observational data alone;
additional information such as sensitivity models, structural assumptions on exogenous variables,
or instrumental variables is needed.
If \eqref{eq:positivity} fails, weighting estimators become unstable and extrapolation dominates.
If SUTVA fails, potential outcomes must be indexed by joint assignments (e.g., $Y_i(a_1,\dots,a_n)$),
and the meaning of intervention changes.
We focus on the basic setting, and later discuss extensions through the lens of generative modeling and orthogonal inference.

\paragraph{Nuisance functions and counterfactual distributions}
The identification formula \eqref{eq:backdoor} becomes a statistical estimation problem once we choose how to estimate
its components.
Typically, we learn nuisance functions such as the outcome regressions $\mu_a(x)$ and the propensity score $e(x)$,
and use them to estimate $\tau$ or $\tau(x)$.
In high-dimensional or complex settings, naive plug-in estimators can carry regularization bias
that invalidates asymptotic inference.
DDML addresses this via orthogonalization and cross-fitting.
\index{cross-fitting}
Moreover, when the goal is to estimate (and sample from) $p(y\,|\,\mathrm{do}(A=a),x)$ itself,
transport and generative models (including flow matching) appear naturally as \emph{counterfactual generators}.
\index{generative model}

\begin{figure}[t]
\centering
\includegraphics[width=0.65\textwidth]{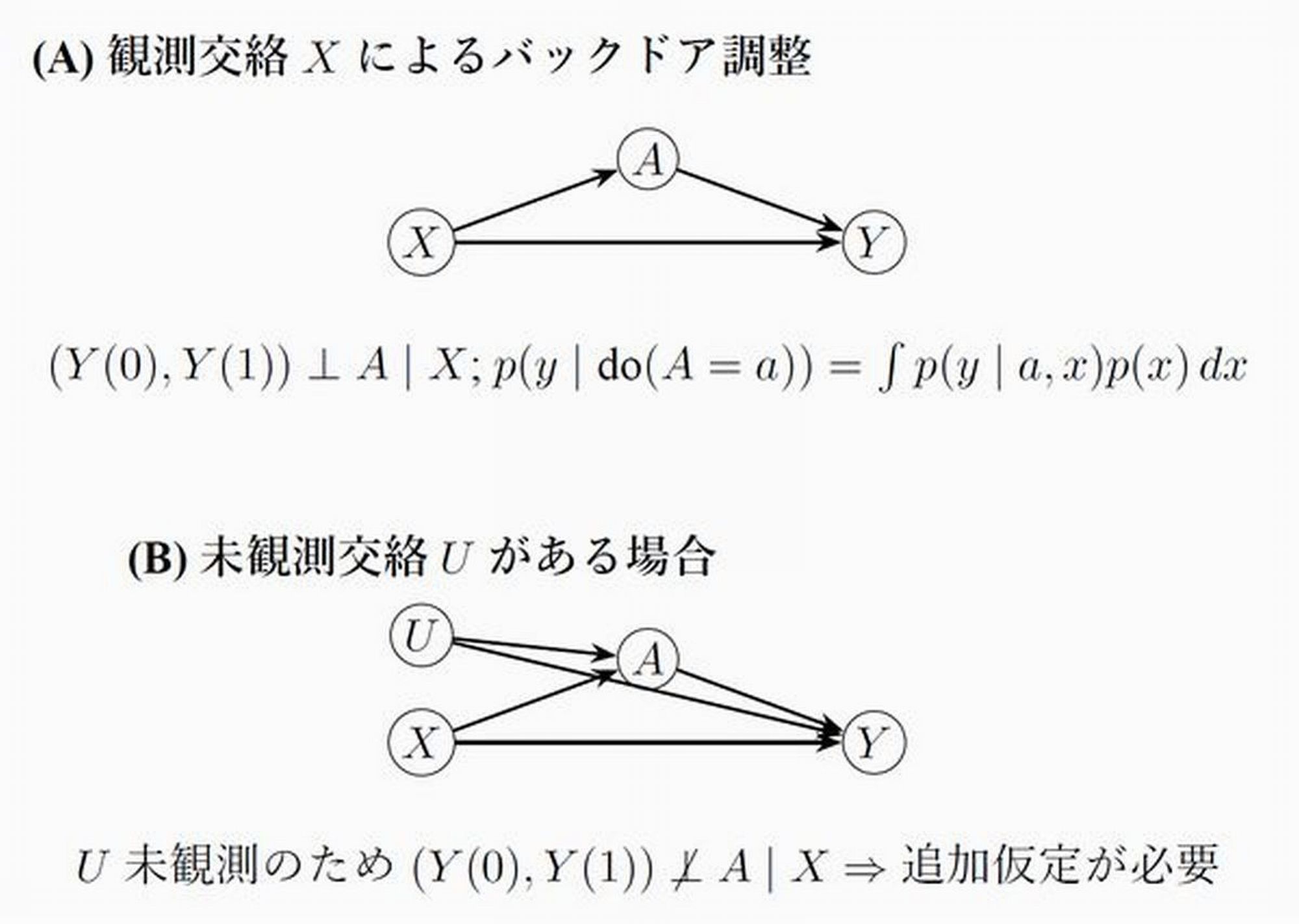}
\caption{A basic causal DAG example.
(A) Conditioning on observed confounders $X$ closes the backdoor path, and the interventional distribution is identified by the $g$-formula.
(B) If an unobserved confounder $U$ is a common cause of $A$ and $Y$, conditioning on $X$ does not close the backdoor path,
and $p(y\,|\,\mathrm{do}(A=a))$ is not uniquely identified from the observational distribution alone.}
\label{fig:dag_backdoor}
\end{figure}

\section{Causal Optimal Transport}
\index{optimal transport}
\label{subsec:causal_ot}

This section formulates covariate shift between treated and control groups as a \emph{transport} problem between probability distributions.
By learning a transport map, we can implement counterfactual generation in a sampleable way.
Optimal transport provides a geometric language for balancing: it turns ``make the covariate distributions comparable''
into a distribution transformation problem.
Modern generative models (flow matching, diffusion models, continuous normalizing flows, etc.)
can be viewed as scalable implementations of such transformations in high dimensions.

\paragraph{Covariate shift and balancing}
Let $A\in\{0,1\}$ and covariates $X$.
In observational data,
\[
p_1(x)=p(x\,|\,A=1),\qquad p_0(x)=p(x\,|\,A=0)
\]
typically differ.
This mismatch is one manifestation of confounding: the naive contrast
$\mathbb{E}[Y\,|\,A=1]-\mathbb{E}[Y\,|\,A=0]$
is generally biased.
Classically, balancing is achieved by reweighting using the propensity score $e(x)=P(A=1\,|\,X=x)$ (IPW),
which creates a ``pseudo-population.''
Here we interpret this operation as transporting $p_1$ to $p_0$ (or vice versa).

\paragraph{OT basics}
We reviewed Monge--Kantorovich OT in Section~\ref{subsec:ot-monge-kantorovich}.
If there exists a transport map $T$ such that $(T)_\# p_1=p_0$,
\index{pushforward}
then we can align the treated covariate distribution to the control one:
\[
(X\,|\,A=1)\sim p_1,\qquad Z=T(X)\sim p_0.
\]
This is a geometric expression of balancing.

\paragraph{Weighting and transport: two representations of the same distribution change}
Balancing transforms $p_1$ into $p_0$.
This transformation can be represented either by a density ratio (weights) or by a transport map.
If $w(x)=p_0(x)/p_1(x)$ exists, then for any integrable $f$,
\[
\mathbb{E}_{p_0}[f(X)]=\mathbb{E}_{p_1}\bigl[w(X)f(X)\bigr].
\]
If a map $T$ satisfies $(T)_\# p_1=p_0$, then
\[
\mathbb{E}_{p_0}[f(X)]=\mathbb{E}_{p_1}\bigl[f(T(X))\bigr].
\]
The former corresponds to IPW, the latter to OT-style balancing.
Both ``make distributions match,'' but in different coordinates.

\paragraph{Link to counterfactuals (joint handling of covariates and outcomes)}
The goal is not only to balance $X$ but to estimate the distribution of counterfactual outcomes $Y(0)$ and $Y(1)$.
A natural idea is to consider the joint variable $U=(X,Y)$ and compare
\[
p_1(u)=p(x,y\,|\,A=1),\qquad p_0(u)=p(x,y\,|\,A=0).
\]
However, since $Y$ changes under intervention, transporting $(X,Y)$ naively can lack causal meaning.
In causal inference, it is important to separate roles:
balancing $X$ corrects differences in assignment mechanisms,
whereas transforming $Y$ corresponds to counterfactual generation itself.

A simple two-stage design is:
\begin{enumerate}
\item Learn a covariate transport $T_{1\to 0}$ that maps treated covariates to the control covariate distribution.
\item Use an outcome regression $\mu_0(x)=\mathbb{E}[Y\,|\,A=0,X=x]$ and set
$\widehat Y(0)=\mu_0\!\bigl(T_{1\to 0}(X)\bigr)$ for treated units.
\end{enumerate}
This is intuitive, but estimation errors propagate, and individual-level counterfactuals remain unidentified.

\paragraph{Identification formula (distributional $g$-formula)}
To justify ``generating a counterfactual distribution,'' recall the distributional identification:
under unconfoundedness $Y(a)\indep A\,|\,X$ and overlap,
\begin{equation}
p\{y\,|\,\mathrm{do}(A=a)\}
=
\int p(y\,|\,A=a,X=x)\,p(x)\,dx
\label{eq:g_formula}
\end{equation}
holds.
Thus the target is not to ``guess'' each unit's $Y_i(a)$, but to estimate the components
$p(y\,|\,A=a,X=x)$ and $p(x)$ (or an equivalent distribution transformation) and thereby approximate
the interventional distribution $p(y\,|\,\mathrm{do}(A=a))$ in a \emph{sampleable} form.
From this angle, a generative model is not merely a device that outputs plausible samples,
but a computational principle that implements an identified distribution transformation.
More realistically, one aims to directly estimate a counterfactual generator for
$p(y\,|\,\mathrm{do}(A=a),x)$ under the required identification assumptions.

\paragraph{Causal transport: designing the objective}
In causal OT, transport is not only geometry but is constrained to preserve causal identification.
Two design principles are central.

\paragraph{Balancing constraints}
One may enforce exact pushforward constraints
\[
(T_{1\to 0})_\#\,p_1=p_0,
\]
or, more practically, moment matching of features:
\[
\mathbb{E}\bigl[\phi(T_{1\to 0}(X))\,|\,A=1\bigr]
=
\mathbb{E}\bigl[\phi(X)\,|\,A=0\bigr],
\]
where $\phi$ is a feature map.

\paragraph{Structure preservation (causal coherence)}
Transport is intended to correct the assignment mechanism, not to arbitrarily rewrite the outcome mechanism.
A natural division of labor is therefore:
transport acts on the covariate space $X$,
and counterfactual outcome distributions are induced through $p(y\,|\,a,x)$.
Choosing what to transport and what to learn is the core of causal OT.

\paragraph{Learning transport maps via flow matching}
Computing an optimal transport map explicitly is generally difficult in high dimensions.
Instead, represent transport as a continuous-time probability flow and learn its velocity field.
\index{velocity field}
Consider a path of densities $(\rho_t)_{t\in[0,1]}$ with
\(
\rho_0=p_0,\ \rho_1=p_1.
\)
If a velocity field $v_t$ satisfies the continuity equation,
\begin{equation}\notag
\partial_t \rho_t(x)+\nabla\cdot\bigl(\rho_t(x)\,v_t(x)\bigr)=0,
\end{equation}
then the ODE
\begin{equation}
\frac{d}{dt}X_t=v_t(X_t),\qquad X_0\sim p_0
\label{eq:ode_transport}
\end{equation}
induces (ideally) $X_1\sim p_1$.
Flow matching approximates $v_t$ by a neural network $v_\theta(t,x)$ and learns it via a squared-loss regression
based on a designed conditional path.
An implementation advantage is that learning can be sample-based and does not require explicit density evaluation
or Jacobian computations during training (in continuous time, the divergence integral would appear in likelihood-based flows).

\paragraph{Conditional transport}
In applications, transport may depend on additional conditions (subpopulations, regions, time, strata).
Flow matching extends naturally to conditional velocity fields $v_\theta(t,x\,|\,c)$,
allowing one to learn a family of transport maps indexed by $c$.
This is useful when treating stratification or CATE as a distribution-transformation problem.
\index{CATE}

\paragraph{Generating counterfactual distributions}
Once a transport model is available, counterfactual generation can be written as follows.
For treated units with covariates $X$, generate
\[
\widetilde X=T_{1\to 0}(X),
\]
and then use the control-side conditional outcome distribution (or its generator) to sample
\[
\widetilde Y(0)\sim p(y\,|\,A=0,\widetilde X).
\]
This yields samples from the counterfactual distribution ``what would have happened without treatment.''
Similarly, learning $T_{0\to 1}$ enables generation of $Y(1)$ from controls.
Individual-level counterfactuals are not identified in general; the generated quantities should be interpreted as samples
from an identified \emph{counterfactual distribution} under the chosen assumptions.
In this sense, causal OT and generative modeling provide tools not only for point estimation of ATE/CATE
but also for presenting distributional effects and uncertainty.
\index{uncertainty}

\paragraph{Connection to DDML}
Transport maps and counterfactual generators learned by flow matching act as nuisance components in causal-effect estimation.
Naive plug-in may retain regularization bias, especially with high-capacity models.
The next section therefore introduces orthogonalization and cross-fitting (DDML) as the theoretical device
that permits $\sqrt{n}$-valid inference even when complex generative models are used for nuisance estimation.
\index{double machine learning}

\paragraph{Note: Monte Carlo error from generative models}
If one estimates $\mu_a(x)=\mathbb{E}[Y\,|\,A=a,X=x]$ from a conditional generator by an average of $M$ samples,
the additional approximation error is $O_p(M^{-1/2})$.
To match DDML asymptotics ($n\to\infty$), one should take $M$ sufficiently large,
or combine the generator with an analytic conditional-mean head when possible.

\section{Neyman Orthogonality and Cross-Fitting in Causal Inference}
\index{influence function}
\label{subsec:orth_tutorial}

To estimate counterfactual distributions and interventional effects,
we often need to estimate high-capacity nuisance objects simultaneously:
the propensity score $e(x)=P(A=1\,|\,X=x)$, outcome regressions $\mu_a(x)=\mathbb{E}[Y\,|\,A=a,X=x)$,
and sometimes conditional samplers for $p(y\,|\,a,x)$.
If such nuisance estimates are plugged in directly, first-order regularization bias can remain,
and $\sqrt{n}$ inference can break down.

The general principles of orthogonalization (Neyman orthogonality) and cross-fitting were summarized in
Section~\ref{subsec:ddml}.
Here we focus on
(i) ATE/CATE estimation via orthogonal scores of AIPW type, and
(ii) policy evaluation and dynamic treatment regimes built on counterfactual generation.

As a canonical example, consider the ATE orthogonal moment in AIPW form.
Let $\mu_a(x)=\mathbb{E}[Y\,|\,A=a,X=x]$ and $e(x)=P(A=1\,|\,X=x)$.
Then
\begin{align}\notag
\Psi(O;\psi,\eta)
&=
\Bigl\{\mu_1(X)-\mu_0(X)\Bigr\}
+
\frac{A}{e(X)}\Bigl\{Y-\mu_1(X)\Bigr\}\\[3mm]
&-
\frac{1-A}{1-e(X)}\Bigl\{Y-\mu_0(X)\Bigr\}
-\psi
\label{eq:aipw_score}
\end{align}
is orthogonal in $\eta=(\mu_0,\mu_1,e)$ and is \emph{doubly robust}:
consistency holds if either $\mu_a$ is correctly estimated or $e$ is correctly estimated.
In our context, $\mu_a$ and $e$ may be learned not only by standard regression/classification learners
but also via a conditional generator for $p(y\,|\,a,x)$ (e.g., flow matching),
from which one can compute $\mu_a(x)=\mathbb{E}[Y\,|\,a,x]$ or other functionals (quantiles, distributional effects).
Such generators enter \eqref{eq:aipw_score} as nuisance estimators.

\paragraph{Implementation note: exploding weights and clipping}
When overlap is weak, $\widehat e(X)$ can be close to $0$ or $1$, and AIPW weights may explode.
A common numerical stabilization is clipping:
\[
\widehat e(X)\leftarrow \min(1-\varepsilon,\max(\varepsilon,\widehat e(X))),
\]
with $\varepsilon\asymp 10^{-2}$, for example.
This effectively changes the estimand to a trimmed population; if used, its interpretation should be stated explicitly.
\index{stability}

\paragraph{Cross-fitting}
Even with orthogonal scores, learning $\widehat\eta$ and evaluating $\widehat\psi$ on the same data can leave
dependence between learning error and evaluation noise, which complicates empirical-process arguments,
especially with flexible learners.
DDML addresses this by splitting the sample into $K$ folds.
For each fold $k$, learn nuisance estimators $\widehat\eta^{(-k)}$ using the data outside fold $k$,
evaluate the score on observations in fold $k$, and then average across folds:
\begin{align}
\widehat\psi
&=
\frac1n\sum_{i=1}^n
\Bigl[
\widehat\mu^{(-k(i))}_1(X_i)-\widehat\mu^{(-k(i))}_0(X_i)\notag\\[3mm]
&+
\frac{A_i}{\widehat e^{(-k(i))}(X_i)}\{Y_i-\widehat\mu^{(-k(i))}_1(X_i)\}\notag\\[3mm]
&-
\frac{1-A_i}{1-\widehat e^{(-k(i))}(X_i)}\{Y_i-\widehat\mu^{(-k(i))}_0(X_i)\}
\Bigr],
\label{eq:ddml_crossfit}
\end{align}
where $k(i)$ is the fold index of observation $i$.
This weakens the dependence between $O_i$ and $\widehat\eta^{(-k(i))}$ and makes asymptotic expansions tractable.

\subparagraph{Key rate condition for $\sqrt{n}$ inference}
Under orthogonality \eqref{eq:neyman_orth} and cross-fitting \eqref{eq:ddml_crossfit},
standard DDML theory shows that if nuisance estimation errors are sufficiently small, then
\[
\sqrt{n}\bigl(\widehat\psi-\psi_0\bigr)\ \Rightarrow\ \mathcal{N}(0,V).
\]
The key point is that orthogonality removes first-order nuisance error, so the remainder appears as products such as
\[
\|\widehat\mu_1-\mu_{1,0}\|\cdot \|\widehat e-e_0\|,
\qquad
\|\widehat\mu_0-\mu_{0,0}\|\cdot \|\widehat e-e_0\|,
\]
which only need to be $o_p(n^{-1/2})$.
Thus, nuisance estimators need not be $\sqrt{n}$-consistent individually:
even nonparametric learners or generative models can be used, as long as their learning errors satisfy suitable bounds.
In this division of labor, a high-capacity model such as flow matching is used for counterfactual distribution approximation
(nuisance estimation), while DDML debiases the target causal-effect estimation.
\citep{chernozhukov2018ddml,chernozhukov2022locally}

\subsection*{Numerical experiment}
\label{subsec:causal-interventional-distribution-demo}

\paragraph{Goal}
Causal inference is not limited to the mean effect (ATE); in many tasks the target is the \emph{interventional distribution}
$p(y\,|\,\mathrm{do}(A=a))$ itself.
We compare (i) random forests, which are strong for mean regression, with
(ii) flow matching, which learns a conditional sampler.
Even when ATE estimates are similar, differences can emerge in the distribution tails.

\paragraph{Data generation}
Let $X\in\mathbb{R}^d$ with $X\sim\mathcal{N}(0,I)$.
Treatment is assigned by a logistic propensity $e(X)=P(A=1\,|\,X)$ (with sufficient overlap).
Potential outcomes are generated as
\[
Y^{(a)} = f(X) + \tau(X)a + s(X,a)\,\varepsilon(X,a),
\]
where $\varepsilon$ is generated as a mixture of a Gaussian component and a skewed exponential component.
Thus treatment affects not only the mean but also the \emph{tail shape} (skewness and mixture proportion).
We observe $Y=Y^{(A)}$.

\paragraph{Concrete components (implementation)}
In the accompanying code we set $d=10$ and define the components as follows,
where $\sigma(u)=(1+e^{-u})^{-1}$ is the logistic function.

\begin{itemize}
\item \textbf{Baseline function $f(X)$.}
As a nonlinear mean structure (with interactions), we use
\begin{align*}
f(X)
&=
1.0\sin(1.2X_1)
+0.6(X_2^2-1.0)
+0.5X_3X_4\\[2mm]
&\;+0.3\cos(X_5X_6)
-0.2X_7
+0.15\sin(X_8+X_9).
\end{align*}
This creates a nontrivial regression function with sine/cosine terms, squares, and interactions,
inspired by settings in \citet{kang2007demystifying}.

\item \textbf{Scale function $s(X,a)$.}
To introduce heteroskedasticity and treatment-dependent variance, we set
\[
s(X,a)
=
\Bigl\{
0.7
+0.25\,\sigma(0.9X_2-0.6X_3)
+0.15|X_4|
\Bigr\}\,
\exp(0.25\kappa a),
\]
where $\kappa>0$ controls the strength of distributional effects; we use $\kappa=1.6$.

\item \textbf{Error $\varepsilon(X,a)$ (treatment-dependent tail shape).}
We control the tail via a mixture of a normal core and a right-skewed exponential component.
Define
\[
p(X,a)=\sigma\Bigl(0.9X_1-0.7X_2+0.3\sin(X_3)+1.4\kappa a\Bigr),
\]
generate $Z\,|\,(X,a)\sim\mathrm{Bernoulli}\bigl(p(X,a)\bigr)$,
and independently generate $\xi\sim\mathcal{N}(0,1)$ and centered exponential noise $U\sim \mathrm{Exp}(1)-1$.
Set
\[
\varepsilon(X,a)=0.75\,\xi+\bigl(0.15+0.9Z\bigr)U.
\]
Here $U$ is the heavy right-tail component; when $Z=1$ its coefficient increases from $0.15$ to $1.05$,
which amplifies the upper tail.
Because $p(X,a)$ depends on $(X,a)$, treatment affects tail shape through skewness and mixture weight.
\end{itemize}

\paragraph{Estimators}
\begin{enumerate}
\item \textbf{Random forest (regression + residual resampling).}
For each arm, learn the mean regression $\mu_a(x)\approx \mathbb{E}[Y\,|\,A=a,X=x]$ using a random forest.
Estimate a scale function $s_a(x)$ by regressing $\log\{(Y-\mu_a(X))^2\}$.
From pooled standardized residuals $\widehat\varepsilon=(Y-\mu_a(X))/s_a(X)$, resample
\[
\widehat Y^{(a)}=\mu_a(X)+s_a(X)\,\widehat\varepsilon
\]
to approximate the interventional distribution.

\item \textbf{Flow matching (conditional distribution sampler).}
Learn a conditional sampler for $p(y\,|\,a,x)$ and estimate $p(y\,|\,\mathrm{do}(A=a))$
by sampling $X\sim p(x)$ and then generating $Y^{(a)}\sim p_\theta(y\,|\,a,X)$ directly.
\index{flow matching}
\end{enumerate}

\paragraph{Evaluation metrics}
We evaluate ATE and quantile treatment effects (QTE),
\[
\mathrm{QTE}_\alpha = Q_\alpha\!\bigl(Y^{(1)}\bigr)-Q_\alpha\!\bigl(Y^{(0)}\bigr),
\]
and measure distributional discrepancy by the one-dimensional Wasserstein distance $W_1$
between the true and estimated interventional distributions for $a=0,1$.
Figure~\ref{fig:qq-do} visualizes tail reproduction via QQ plots.
\index{Wasserstein distance}

\paragraph{Results}
The true values were ATE$_{\mathrm{true}}=1.358$,
QTE$_{0.1}=0.307$, QTE$_{0.5}=1.329$, and QTE$_{0.9}=2.452$.
Random forest achieved a small ATE error ($1.350$), confirming its strength for mean effects.
However, distributional metrics showed a clear difference, especially under do$(A=1)$:
RF had $W_1=0.362$, whereas FM achieved $W_1=0.066$.
In the QQ plots (Figure~\ref{fig:qq-do}), RF systematically underestimates the upper tail,
while FM reproduces the tail shape much more accurately.

\paragraph{Takeaway}
Mean effects (ATE) are often well estimated by regression learners such as random forests.
But when the target involves the full interventional distribution
(quantile effects, tail risk, counterfactual generation, policy evaluation),
simplified residual models can miss distributional shifts.
Flow matching provides a conditional sampler, and differences invisible at the mean level become apparent
in distributional evaluations ($W_1$, QQ).

\begin{figure}[!t]
\centering
\includegraphics[width=0.95\textwidth]{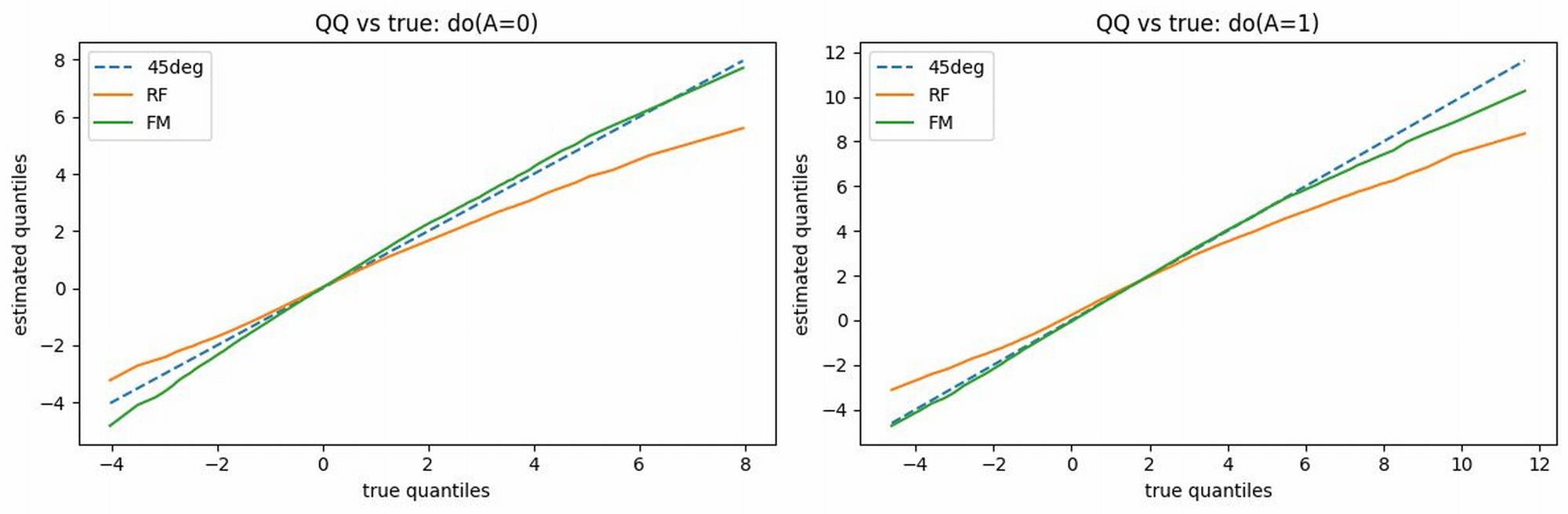}
\caption{QQ plots for interventional distributions (truth vs estimate).
Left: do$(A=0)$; right: do$(A=1)$.
RF (regression + pooled standardized residuals) underestimates the upper tail, with a larger discrepancy for do$(1)$.
FM (conditional sampler) reproduces the distributional shape including the tail more faithfully.}
\label{fig:qq-do}
\end{figure}

\begin{algorithm}[t]
\caption{Estimating interventional distributions (simple comparison: RF vs FM)}
\label{alg:causal-do-demo-simple}
\begin{algorithmic}[1]
\Require $n,d$ (sample size and dimension), number of generated samples $L$, quantile set $\mathcal{A}$, RF learner, FM learner
\Ensure ATE, QTE, $W_1$, QQ (truth vs estimate)
\State \textbf{Generate data:} simulate $(X_i,A_i,Y_i)_{i=1}^n$ and obtain dataset $\mathcal{D}$.
\State \textbf{Truth (benchmark):}
\For{$a\in\{0,1\}$}
  \State Generate $X^{(\ell)}\sim p(x)$ and produce $Y_{\mathrm{true}}^{(a,\ell)}$ for $\ell=1,\dots,L$ using the true mechanism.
\EndFor
\State \textbf{RF (regression + residual resampling):}
\For{$a\in\{0,1\}$}
  \State Fit mean regression $\widehat\mu_a(x)\approx \mathbb{E}[Y\,|\,A=a,X=x]$.
  \State Fit scale $\widehat s_a(x)$ from residuals and pool standardized residuals $\widehat\varepsilon$.
  \For{$\ell=1,\dots,L$}
    \State Draw $X^{(\ell)}\sim p(x)$ and resample $\widehat\varepsilon^{(\ell)}$, then set
    \[
    Y_{\mathrm{rf}}^{(a,\ell)}\gets \widehat\mu_a(X^{(\ell)})+\widehat s_a(X^{(\ell)})\,\widehat\varepsilon^{(\ell)}.
    \]
  \EndFor
\EndFor
\State \textbf{FM (direct conditional sampling):}
\State Train $\widehat p_\theta(y\,|\,a,x)$ on $\mathcal{D}$.
\For{$a\in\{0,1\}$}
  \For{$\ell=1,\dots,L$}
    \State Draw $X^{(\ell)}\sim p(x)$ and sample $Y_{\mathrm{fm}}^{(a,\ell)}\sim \widehat p_\theta(\cdot\,|\,a,X^{(\ell)})$.
  \EndFor
\EndFor
\State \textbf{Evaluation:}
\State Compute ATE and QTE$_\alpha$ ($\alpha\in\mathcal{A}$) for truth, RF, and FM.
\State Compute $W_1$ for $a=0,1$ and compare tails via QQ plots.
\end{algorithmic}
\end{algorithm}

\section{From ATE/CATE to Dynamic Treatment Regimes}
\index{dynamic treatment regime}
\label{subsec:cate_to_dtr}

This section generalizes from ATE to heterogeneous effects (CATE),
and further to dynamic treatment regimes (DTR), where interventions are chosen sequentially over time.
The key perspective is that the target is not only a point estimand but a \emph{counterfactual distribution induced by a policy}.
In this setting, orthogonalization and cross-fitting are the foundation that allows $\sqrt{n}$ inference
while using complex nuisance estimators (regressors and generative models).
\citep{chernozhukov2018ddml,hernanrobins2020whatif}
Finally, we note that for continuous-time interventions, differential-equation-based flow models provide a natural generator.

\paragraph{CATE: identification and estimation}
CATE is defined by
\[
\tau(x)=\mathbb{E}[Y(1)-Y(0)\,|\,X=x].
\]
Under unconfoundedness and overlap,
\[
\tau(x)=\mu_1(x)-\mu_0(x),\qquad \mu_a(x)=\mathbb{E}[Y\,|\,A=a,X=x],
\]
but a naive plug-in $\widehat\tau(x)=\widehat\mu_1(x)-\widehat\mu_0(x)$ can be hard to analyze due to regularization bias.
DDML instead constructs an orthogonal pseudo-outcome and regresses its conditional mean to estimate $\tau(x)$.
\citep{kunzel2019metalearners,nie2021quasioracle}

A representative pseudo-outcome is the AIPW form.
For the true nuisances $\eta=(\mu_0,\mu_1,e)$, define
\begin{equation}
\widetilde Y
=
\Bigl\{\mu_1(X)-\mu_0(X)\Bigr\}
+
\frac{A}{e(X)}\Bigl\{Y-\mu_1(X)\Bigr\}
-
\frac{1-A}{1-e(X)}\Bigl\{Y-\mu_0(X)\Bigr\}.
\label{eq:pseudo_outcome_cate}
\end{equation}
Then
\[
\mathbb{E}[\widetilde Y\,|\,X=x]=\tau(x).
\]
With cross-fitting, learn $\widehat\eta^{(-k)}$, compute pseudo-outcomes
$\widetilde Y_i=\widetilde Y(O_i;\widehat\eta^{(-k(i))})$, and finally regress
\[
\widehat\tau(\cdot)\approx \mathbb{E}[\widetilde Y\,|\,X=\cdot]
\]
using any learner.
This decomposes CATE estimation into two stages:
a flexible first stage for nuisance estimation (including conditional generative models),
and a second stage stabilized by orthogonality.
In practice, meta-learners (S/T/X-learners) \citep{kunzel2019metalearners},
R-learners \citep{nie2021quasioracle}, and generalized random forests \citep{athey2019grf} are widely used.

\paragraph{From CATE to policy learning}
CATE connects directly to policy design: which units should be treated?
For a binary policy $\pi:\mathcal{X}\to\{0,1\}$, define its value
\[
V(\pi)=\mathbb{E}\bigl[Y(\pi(X))\bigr].
\]
Under identification assumptions,
\[
V(\pi)=\mathbb{E}\bigl[\mu_{\pi(X)}(X)\bigr]
=
\mathbb{E}\bigl[\mu_0(X)+\pi(X)\tau(X)\bigr].
\]
Heuristically, the optimal policy treats where $\tau(x)$ is positive.
In practice, estimation error and weak overlap (extreme propensity scores) can make the naive rule
``treat if $\widehat\tau(x)>0$'' unstable.
Orthogonal pseudo-outcomes such as \eqref{eq:pseudo_outcome_cate} are also useful for policy evaluation and optimization,
and policy learning can be cast as cost-sensitive classification.
Again, cross-fitting is important for debiasing.

\paragraph{Dynamic treatment regimes (DTR) as multi-stage interventions}
Consider repeated interventions at times $t=1,\dots,T$.
Let the history up to time $t$ be
\[
H_t=(X_1,A_1,\dots,X_{t-1},A_{t-1},X_t),
\]
and define a dynamic regime (treatment plan)
\[
g=(g_1,\dots,g_T),\qquad A_t=g_t(H_t).
\]
\citep{murphy2003dtr,chakraborty2013dtr}
Let $Y(g)$ denote the potential outcome under regime $g$.
The goal is to estimate an optimal regime
\[
g^\ast \in \argmax_g \mathbb{E}[Y(g)].
\]
This is a sequential decision problem at the intersection of causal inference and dynamic programming.
\citep{murphy2003dtr,suttonbarto2018rl}
\index{causal inference}

Identification typically requires sequential ignorability and positivity:
for each $t$,
\[
\{Y(g):g\}\ \indep\ A_t\,|\,H_t,
\qquad
0<P(A_t=a\,|\,H_t=h)<1.
\]
Under these, interventional distributions under $g$ can be expressed via the longitudinal $g$-formula.
\citep{hernanrobins2020whatif,robins1986hws}
When time-varying confounders are themselves affected by past treatment,
naive regression adjustment can be biased; marginal structural models (MSM) and IPTW address this.
\citep{robins2000msm}
There are many approaches (MSM, structural nested models, $g$-estimation, sequential regression, etc.),
but the essence is always: estimate the counterfactual distribution induced by a policy.

\paragraph{Connection to reinforcement learning: Q-learning and off-policy evaluation}
In reinforcement-learning (RL) language, DTR corresponds to policy optimization.
Under a Markov structure, value functions and action-value functions ($Q$-functions) characterize optimal policies.
\citep{suttonbarto2018rl}
Classical Q-learning learns the fixed point of the Bellman equation,
\[
Q^\ast(s,a)
=
\mathbb{E}\bigl[R_{t+1}+\gamma\max_{a'}Q^\ast(S_{t+1},a')\,|\,S_t=s,A_t=a\bigr],
\]
\citep{watkins1992qlearning}
and in medical applications it can be interpreted as a DTR estimator based on sequential regression.
Observational data, however, are not generated by exploratory policies, so \emph{off-policy} evaluation is required.
Here too, the same stability principles as IPTW/AIPW apply: orthogonalization and cross-fitting are key.

\paragraph{Continuous-time interventions and a flow-model outlook}
Finally, consider continuous time.
Let $X_t$ be a state process and $a_t$ a control (dose, stimulation, etc.).
A natural model is a (stochastic) differential equation such as
\begin{equation}
\frac{d}{dt}X_t = b(t,X_t,a_t) + \varepsilon_t,
\label{eq:ct_controlled_dynamics}
\end{equation}
where $\varepsilon_t$ represents noise (or one may use an SDE formulation).
A policy $\pi$ specifies $a_t=\pi(t,\mathcal{H}_t)$, where $\mathcal{H}_t$ is the continuous-time history,
and induces a counterfactual trajectory $X_t^{(\pi)}$ and outcome $Y^{(\pi)}$.
In this view, counterfactuals are trajectories under alternative controls, and the problem is fundamentally a control problem.

Differential-equation-based generative models enter naturally.
Neural ODEs represent the vector field $b_\theta$ by a neural network and learn continuous dynamics via numerical solvers.
\citep{chen2018neuralode}
\index{ODE}
From the flow viewpoint, distributions deform continuously along $t$ via a velocity field,
so conditional distribution transformations can generate counterfactual distributions under interventions.
Conceptually, if one can learn transitions such as
\[
p(x_{t+\Delta}\,|\,x_t,a_{\,[\,t,\;t+\Delta\,]}),
\]
then fixing a policy $\pi$ and generating sample paths yields an approximation of the counterfactual distribution under $\pi$.
From the DDML viewpoint, such generators are nuisance components, and targets such as $V(\pi)$ or policy-improvement functionals
should be estimated and debiased via orthogonal moments.

In summary:
(i) ATE, CATE, and DTR are unified as targets defined by counterfactual distributions induced by policies;
(ii) generative models (flows, ODEs, transport) can serve as flexible counterfactual generators on the nuisance side; and
(iii) inference is stabilized by orthogonalization and cross-fitting, which connects these flexible models to $\sqrt{n}$ inference.
Open problems include making identification assumptions explicit in continuous-time settings
(overlap, unmeasured confounding, partial observability) and developing theory that propagates
counterfactual uncertainty into policy evaluation (local robustness, sensitivity analysis).

\paragraph{Chapter summary}
We organized interventional-effect estimation into three steps:
``identification $\to$ transport $\to$ inference.''
Key points are:
\begin{itemize}
\item \textbf{Identification.}
Under unconfoundedness, overlap, and SUTVA, the interventional distribution $p(y\,|\,\mathrm{do}(A=a))$ is identified from the observational distribution.
Individual counterfactuals $Y_i(a)$ are not observed, but distribution-level counterfactuals (interventional distributions) are well-defined estimands.
\item \textbf{Distribution transformation.}
Balancing aligns covariate distributions between treated and control groups.
Weighting by density ratios (IPW) and transforming by transport maps (OT) implement the same goal in different coordinates.
Weighting acts as a change of measure for expectation evaluation,
whereas transport pushes distributions forward via maps on the sample space.
This distinction determines whether counterfactual \emph{sample generation} can be implemented directly.
\item \textbf{Inference.}
DDML uses Neyman orthogonality and cross-fitting so that high-dimensional nuisance estimation error does not enter first-order asymptotics,
enabling estimation and uncertainty quantification for ATE/CATE.
Flows can be incorporated as flexible nuisance estimators (or conditional samplers),
but weak overlap (extreme weights) remains a practical instability risk.
\item \textbf{Dynamic regimes.}
The same blueprint extends to dynamic treatment regimes, where interventions are chosen sequentially given a history $H_t$
and the goal is to learn an optimal policy.
\end{itemize}

\chapter{Inference-Aware Generation: Uncertainty, Diagnostics, and Score Learning}
\index{diagnostics}
\index{uncertainty}
\index{KSD (kernel Stein discrepancy)}
\index{DSM (denoising score matching)}

In Chapter~2 we introduced the Stein identity and score-based ideas.
In this chapter we reposition them not as \emph{training objectives} but as practical tools for
\emph{model diagnostics} and \emph{uncertainty assessment}.
Our goal is to narrow the gap between ``we can generate samples'' and ``we can conduct statistical inference''.
From a statistical viewpoint, a generative model is not trustworthy unless we can answer:
\emph{Is it accurate? Where does it miss? How do such mismatches affect downstream inference?}

We proceed in four steps:
(i) we organize diagnostics into unconditional two-sample checks and conditional checks;
(ii) we introduce the \emph{kernel Stein discrepancy} (KSD) as a density-free goodness-of-fit measure;
(iii) we explain \emph{denoising score matching} (DSM) as a computational device to avoid expensive divergence terms;
and (iv) we clarify the layers of uncertainty that are specific to inference with generators.

\section{The diagnostic landscape: two-sample tests and conditional checks}
Evaluating a generative model often degenerates into subjective judgments such as ``it looks plausible''.
For statistical use, however, it is essential to distinguish at least the following two types of diagnostics.
Let $Z$ denote the observed vector.
For unconditional generation (e.g.\ copula-type sampling) we take $Z=X$,
while for regression, missing-data imputation, and causal inference it is natural to consider
$Z=(X,Y)$ or $Z=(X,A,Y)$.

\begin{enumerate}
\item \textbf{Unconditional (two-sample) diagnostics:} whether the generated distribution is close to the true distribution \emph{as a whole}.
\item \textbf{Conditional diagnostics:} whether conditional structure such as $Y\,|\,X$ is reproduced correctly (regression as a conditional generator).
\end{enumerate}

These are complementary: relying on only one of them can miss important failures.
For instance, in regression problems the marginal distributions of $X$ and $Y$ may look well matched,
while the conditional shape $Y\mid X$ (nonlinearity, multimodality, heteroskedasticity) is distorted,
which can severely mislead prediction intervals or threshold-based decisions.

\paragraph{Unconditional diagnostics: what do they verify?}
Two-sample diagnostics compare an observed sample $\{Z_i\}_{i=1}^n\sim P$ with a generated sample
$\{\widetilde Z_j\}_{j=1}^m\sim Q_\theta$ and check whether $P\approx Q_\theta$.
Typically one estimates a discrepancy $D(P,Q_\theta)$ from samples and verifies that
$D$ is sufficiently small (or that a test does not reject).
A main advantage is that (i) no conditional structure needs to be specified, and
(ii) many tools remain applicable in high dimensions as long as sampling is possible.

\paragraph{A common blind spot (unconditional): tail mismatches}
Even if a generator reproduces the high-density region well,
small discrepancies in tail mass (frequency of rare events) or tail dependence can cause large differences
in risk assessment and counterfactual analyses.
Many two-sample statistics respond primarily to \emph{average} discrepancies and may have low power
when tail differences occur with small probability.
Thus, in practice it is important to supplement standard distances/tests by
\emph{tail-focused summaries}, such as rare-event frequencies and comparisons of maxima or upper quantiles.

\paragraph{Conditional diagnostics: what do they verify?}
Conditional diagnostics treat the generator as a conditional regressor and assess the shape of $Y\,|\,X=x$.
Concretely, if the model supports
\[
Y = G_\theta(X,\varepsilon),\qquad \varepsilon\sim\pi,
\]
then for a fixed $x$ we generate $\{Y^{(b)}(x)\}_{b=1}^B$ and examine summaries of the predictive distribution
(means, quantiles, intervals) and their calibration.
The key advantage is that one can evaluate \emph{distribution-level} accuracy,
not only point prediction.

\paragraph{A common blind spot (conditional): collapsing multimodality and asymmetry}
A typical failure mode is that a genuinely multimodal $Y\,|\,X=x$ is learned as unimodal due to averaging.
In such cases unconditional two-sample checks may still suggest an acceptable fit, while
the distribution of intervention effects, the distribution produced by imputation,
or the coverage of predictive intervals may break down.
Therefore conditional diagnostics should use shape-sensitive tools beyond mean/variance,
such as quantile curves, PIT (probability integral transform) diagnostics, and coverage checks.

\paragraph{Two-sample and conditional diagnostics are complementary}
In short, unconditional checks test ``global distributional agreement'',
whereas conditional checks test ``regression-level agreement''.
When a generator is used as a component of inference, either alone is insufficient.
One should combine diagnostics according to what matters for the task
(tails, multimodality, conditional structure).
In the next section we introduce KSD as a powerful density-free tool for two-sample diagnostics.

\section{Kernel Stein discrepancy (KSD)}\label{U-stat}
\index{KSD (kernel Stein discrepancy)}
\index{Stein identity}
KSD measures goodness of fit using the \emph{score function} and does not require the normalization constant of a density.
Although many generative models do not provide an explicit density,
there are many situations where the score is evaluable (or can be approximated),
which makes KSD practically attractive.

\paragraph{Density-free (no normalization constant)}
Goodness-of-fit measures based on density comparison are often obstructed in high dimensions by
unknown normalizing constants and intractable likelihood evaluations.
KSD uses the Stein identity to construct a quantity whose expectation is zero under the \emph{correct} distribution,
and estimates its deviation from zero by a sample average.
Hence computation relies on the score $s_q(x)$ and a kernel $K(x,x')$, rather than on the density value itself.
In particular, when a normalizing constant is unknown but the score can be computed (or estimated),
KSD provides a likelihood-free diagnostic option.

As discussed in Chapter~2, the Stein identity can be understood as:
if $X\sim q$, then for a Stein operator ${\cal T}_q$,
\[
\mathbb{E}_q[{\cal T}_q f(X)]=0
\quad\text{for all sufficiently smooth } f.
\]
This suggests evaluating the same quantity under another distribution $p$ and measuring how far it departs from zero.
KSD formalizes and optimizes this idea in the framework of reproducing kernel Hilbert spaces (RKHS),
leading to a computable discrepancy.
For details, see \citet{gorham2015measuring,liu2016kernelized,chwialkowski2016kernel,liu2016stein}.

\paragraph{Definition}
Let $K:\mathbb{R}^d\times\mathbb{R}^d\to\mathbb{R}$ be a symmetric positive definite kernel and
let $\mathcal{H}_K$ be the RKHS with reproducing kernel $K$ and inner product
$\langle\cdot,\cdot\rangle_{\mathcal{H}_K}$.
For each $x$, the evaluation functional $g\mapsto g(x)$ is continuous and the reproducing property holds:
\[
g(x)=\langle g, K(x,\cdot)\rangle_{\mathcal{H}_K}.
\]
A vector-valued RKHS is defined by the direct sum
\(
\mathcal{H}_K^d=\mathcal{H}_K\oplus\cdots\oplus\mathcal{H}_K.
\)
For $f=(f_1,\dots,f_d)$ we set
\[
\langle f,g\rangle_{\mathcal{H}_K^d}=\sum_{j=1}^d \langle f_j,g_j\rangle_{\mathcal{H}_K},
\qquad
\|f\|_{\mathcal{H}_K^d}^2=\sum_{j=1}^d\|f_j\|_{\mathcal{H}_K}^2.
\]
Thus $\|f\|_{\mathcal{H}_K^d}\le 1$ enforces a joint smoothness constraint on all components.

Let $s_q(x)=\nabla \log q(x)$ be the score of $q$ and define the Stein operator
\[
(\mathcal{T}_q f)(x)=s_q(x)^\top f(x)+\nabla\cdot f(x).
\]
For the unit ball
\[
\mathcal{B}_K=\{f\in\mathcal{H}_K^d:\|f\|_{\mathcal{H}_K^d}\le 1\},
\]
the kernel Stein discrepancy is
\[
\mathcal{S}_K(p\|q)
=
\sup_{f\in\mathcal{B}_K}
\Bigl(\mathbb{E}_{X\sim p}[(\mathcal{T}_q f)(X)]\Bigr)^2.
\]

\paragraph{Closed form and estimation}
By the reproducing property and the Riesz representation theorem,
\[
\mathcal{S}_K(p\|q)
=
\mathbb{E}_{(X,X')\sim p^{\otimes 2}}\!\bigl[u_{q,K}(X,X')\bigr],
\]
where the Stein kernel $u_{q,K}$ is
\begin{align*}
u_{q,K}(x,x')
&=
s_q(x)^\top K(x,x')\,s_q(x')
+s_q(x)^\top\nabla_{x'}K(x,x')\\
&\quad
+s_q(x')^\top\nabla_{x}K(x,x')
+\mathrm{tr}\!\bigl(\nabla^2_{x,x'}K(x,x')\bigr).
\end{align*}
If $K$ is sufficiently characteristic and regularity conditions hold, then
$\mathcal{S}_K(p\|q)=0\Rightarrow p=q$.

Given $X_1,\dots,X_n\sim p$, an unbiased estimator is the U-statistic
\[
\widehat{\mathcal{S}}_K(p\|q)
=
\frac{1}{n(n-1)}\sum_{i\neq j} u_{q,K}(X_i,X_j).
\]
In goodness-of-fit testing one often approximates the null distribution by (wild) bootstrap and computes a $p$-value.

\paragraph{Practical notes: kernel choice, bandwidth, and high dimensions}
KSD is useful, but the following points matter in practice.
\begin{itemize}
\item \textbf{Kernel and bandwidth.}
With Gaussian kernels, sensitivity can change substantially with the bandwidth.
Too large a bandwidth averages away differences; too small a bandwidth can overreact to noise.
In practice, one may use multiple bandwidths, or initialize by the median heuristic and then check sensitivity.
\item \textbf{Loss of sensitivity in high dimensions.}
In high dimensions distances concentrate and kernel methods may lose power.
Possible remedies include computing KSD on informative low-dimensional representations (feature maps)
or combining KSD with conditional diagnostics to focus on task-relevant directions.
\item \textbf{Tail robustness.}
As noted above, any statistic that mainly responds to average discrepancies can miss tail differences.
KSD is not a universal detector, so it is safer to combine it with tail summaries and extreme-value diagnostics.
\end{itemize}
A minimal practical protocol is:
(i) start from the median heuristic but always check sensitivity to bandwidth;
(ii) approximate the null distribution by bootstrap (e.g.\ wild bootstrap);
and (iii) in high dimensions, complement KSD by feature design and conditional diagnostics.

In this book we position KSD not as a universal test but as a core density-free two-sample diagnostic,
to be used together with conditional checks.
The next section discusses DSM, which avoids a computational bottleneck in KSD and score learning:
divergence terms (second derivatives).

\section{Denoising score matching (DSM)}\label{app:trace-dsm}
\index{DSM (denoising score matching)}
\index{score matching}
In score learning and score matching, divergence terms such as $\nabla\cdot s_\theta(x)$ (second derivatives)
often appear and can be computationally heavy in high dimensions.
Denoising score matching (DSM) smooths the objective by adding noise and enables score learning without
explicit second-derivative calculations \citep{vincent2011connection,song2021scorebased}.
Here we emphasize DSM not as ``a way to train a generator'' but as a computational technique to obtain
scores stably for diagnostics and inference.

\paragraph{Why second derivatives are costly: divergence as a bottleneck}
Standard score matching objectives often combine a squared error term with a divergence term.
Although automatic differentiation can compute such divergences, the cost typically grows with dimension $d$,
making the divergence a practical bottleneck.
Thus, approximations and reformulations that avoid explicit second derivatives are important in practice.

\paragraph{Hutchinson-type trace estimation}
Since $\nabla_x\cdot s_\theta(x)=\mathrm{tr}\{\nabla_x s_\theta(x)\}$,
let $\epsilon\in\mathbb{R}^d$ satisfy $\mathbb{E}[\epsilon\epsilon^\top]=I$, e.g.\ a Rademacher vector
(each component takes $\pm1$ with probability $1/2$, independently).
Then
\[
\mathrm{tr}(A)=\mathbb{E}_\epsilon[\epsilon^\top A\,\epsilon],
\]
and hence
\[
\nabla_x\cdot s_\theta(x)
=
\mathbb{E}_\epsilon\bigl[\epsilon^\top (\nabla_x s_\theta(x))\,\epsilon\bigr]
\]
can be approximated stochastically.
In automatic differentiation, $\epsilon^\top (\nabla_x s)\epsilon$ can be computed without explicitly forming the Jacobian,
which is advantageous in high dimensions.
In implementations one typically uses
\[
\nabla_x\cdot s_\theta(x)
\approx
\frac{1}{M}\sum_{m=1}^M \epsilon_m^\top (\nabla_x s_\theta(x))\,\epsilon_m,
\qquad \epsilon_m\sim\text{Rademacher}.
\]

Another viewpoint is to add Gaussian noise $\xi\sim N(0,\sigma^2 I)$ and consider $Y=X+\xi$,
and then learn the score of the smoothed density $p_\sigma(y)$.
This yields training rules that avoid explicit second derivatives and
places diffusion-model score learning on the same conceptual line.
For our purposes, it suffices to view DSM as a representative way to reduce the cost of divergence computations.

\paragraph{Implementation tips}
Key practical points for DSM include:
\begin{itemize}
\item \textbf{Choosing the noise level $\sigma$.}
If $\sigma$ is too small, smoothing is weak and the regression remains difficult;
if $\sigma$ is too large, information is washed out.
Multi-scale training that mixes multiple noise levels is often used in practice.
\item \textbf{Be explicit about the goal.}
DSM is a means to make score learning computationally feasible.
Design choices differ depending on whether the learned score is used for generation, for diagnostics, or as a regularizer.
\end{itemize}
DSM is a computational backbone for the ``score-without-density'' theme of this book.
Combined with KSD, it supports practical model checking and calibration of generative distributions.

\section{Validating uncertainty}
When a generative model is used as a component of inference, uncertainty has multiple layers.
If these layers are conflated, one cannot tell whether instability is due to
a misspecified model, insufficient data, or insufficient Monte Carlo draws.
Here we separate uncertainty into:
(i) approximation error of the generator (modeling error),
(ii) estimation error due to finite training data,
and (iii) Monte Carlo (MC) error due to finitely many generated samples.
Approximation error depends on the model class and training objective,
estimation error is governed by data size and regularization/cross-fitting,
and MC error is controlled by the number of generated draws.
Thus, when results ``do not match'', one should diagnose which layer dominates before deciding what to increase.

\paragraph{Approximation error of the generator}
First, any generator has a limited representational class.
Depending on network capacity, regularization, and the learning objective,
the learned distribution $Q_{\widehat\theta}$ may fail to approximate the true distribution $P_0$:
\[
Q_{\widehat\theta}\not\approx P_0.
\]
This is a modeling error and does not vanish by increasing $n$.
The basic remedy is diagnostics: unconditional and conditional checks,
with particular focus on task-relevant structures such as tails and multimodality.

\paragraph{Estimation error (finite training data)}
Second, because the training sample size $n$ is finite, $\widehat\theta$ fluctuates across datasets,
and so does the induced distribution $Q_{\widehat\theta}$.
This corresponds to the sampling variability of an estimator.
When the estimand is an expectation or a causal effect, DDML orthogonalization
suppresses the first-order impact of nuisance estimation error.
If one wants to visualize the uncertainty of the generator itself,
bootstrap (resample the training data and re-train) is a natural approach.

\paragraph{MC error (finite number of generated draws)}
Third, Monte Carlo approximation induces additional variability.
For example, when estimating a functional $E\{h(X)\}$ by
\[
\frac{1}{B}\sum_{b=1}^B h(\widetilde X^{(b)}),\qquad \widetilde X^{(b)}\sim Q_{\widehat\theta},
\]
MC error remains as long as $B<\infty$.
Typically it decreases on the order $B^{-1/2}$, so $B$ should be chosen according to the desired precision.
In practice, the required $B$ differs substantially between stable quantities (means/medians) and tail quantities
(upper quantiles or rare-event probabilities), which demand larger $B$.

\paragraph{Implementation notes}
A minimal practical recommendation is:
\begin{itemize}
\item \textbf{How many draws?}
For stable targets such as means or medians, moderate $B$ may suffice.
For upper quantiles and rare-event probabilities, $B$ must be much larger.
A basic check is to increase $B$ and verify numerical stabilization of the estimated quantity (MC convergence).
\index{stability}
\item \textbf{A bootstrap entry point.}
To reflect estimation error of the generator (finite training data),
one may resample the training data, re-train the generator, and compare the distribution of the downstream estimand.
Although computationally costly, this helps disentangle modeling error, estimation error, and MC error.
\end{itemize}

The diagnostics and uncertainty bookkeeping in this chapter are minimal safety devices
for using generative models as components of statistical inference.
In concrete applications, one should tailor the diagnostic menu to the problem at hand
and use it to decide which uncertainty source to prioritize.

\chapter*{Conclusion}
\addcontentsline{toc}{chapter}{Conclusion}

Inferring statistical parameters is a more difficult endeavor than we often imagine.
If the model is correct, statistical inference proceeds in an orderly fashion:
estimators have clear meaning, asymptotic theory provides the backbone,
and standard errors quantify uncertainty.
In this sense, mathematical statistics builds a self-contained world \citep{inagaki2003}.
Yet, in reality, \emph{the model is wrong}.
What is more troublesome is that it is itself difficult to grasp what this ``wrongness'' looks like.
Model misspecification does not appear as a finite-dimensional deviation;
rather, it manifests as a distortion of a distribution with infinitely many degrees of freedom.
What we overlook is not a small discrepancy in coefficients,
but asymmetry of the distribution, heavy tails,
or multimodality hidden in conditional distributions.
Such departures can slip through the mesh of classical tests
and quietly erode our confidence in inference.

This is why this book has focused on Flow Matching: to confront this ``infinite-dimensional misspecification.''
Flow Matching is not merely a black-box generator.
Behind it lies the \textbf{continuity equation}, the most basic law of mass conservation in physics.
Moreover, the brilliant idea of Stein's identity
provides a route to calibrating distributions while bypassing normalization constants.
Supported by these two pillars, Flow Matching can encompass a broad class of
plausible misspecification neighborhoods and offers the possibility of strongly calibrating
the discrepancy between model and data.
Our goal has been to translate this possibility into the language of statistics
(estimands, influence functions, orthogonalization, efficiency),
and to treat inference and generation not as opponents but as complementary components.

As a concrete illustration, we first examined the Cox model,
often regarded as a jewel of statistics.
Its appeal is that, even in the presence of an infinite-dimensional nuisance component
(the baseline hazard), one can estimate the main effect $\beta$ in an interpretable form.
However, the proportional hazards assumption is often slightly violated in practice.
If we ignore this ``small violation,'' interpretation may remain beautiful,
but inference becomes quietly distorted.
We therefore proposed to fix the Cox model as a baseline
and to separate only the violation component as a residual that is absorbed by a flow (a vector field).
When the skeleton is correct, the procedure collapses back to the baseline model;
corrections appear only when violations exist.
This modesty is the key to gaining flexibility while preserving inferential quality.

Next, we addressed missing data, an issue that statistics cannot avoid.
Whether the missingness mechanism (MAR) holds is often an untestable, perpetual mystery. \index{mar@MAR}
The practical difficulty is not to impute a single value,
but to faithfully reproduce the conditional distribution
$p(x_{\mathrm{mis}}\,|\,x_{\mathrm{obs}})$ of the missing part given the observed part.
Under complex missingness mechanisms, this conditional distribution can be multimodal and highly nonlinear,
and classical regression-style chains may fail to capture its shape.
Because a flow can learn a sampler for this conditional distribution,
imputation can be implemented as ``estimating a distribution'' rather than ``estimating a mean.''
When non-identifiability is unavoidable, one can further present identification intervals as a sensitivity analysis,
thereby visualizing which assumptions inference depends on.

Finally, we turned to causal inference. \index{causal inference}
We wish to estimate intervention effects without bias,
and at the same time to generate counterfactual data correctly.
In some settings the average treatment effect is sufficient,
but in real decision-making we need quantile effects, tail risks,
and comparisons of entire outcome distributions.
Here again, a flow acts as a sampler of counterfactual distributions
and directly visualizes the interventional distribution $p(y\,|\,\mathrm{do}(a))$.
On the other hand, estimation accuracy must be protected by the principle of orthogonalization,
because the greater the freedom of generation, the more fragile inference can become.
What we have emphasized throughout this book is the stance that
``generate flexibly'' and ``estimate without bias'' should not be separated,
but should be designed simultaneously within the framework of influence functions and orthogonalization.

The model is wrong.
Yet, precisely because it is wrong, statistics is beautiful.
We sharpen a baseline model, separate the violation, make uncertainty explicit,
and nevertheless push inference forward.
Flow Matching can become a new toolkit for this purpose.
Supported by the continuity equation and guided by Stein's identity,
it equips us with the power to handle distributions ``as shapes.''
What this book has presented is only an entrance to a methodology that still needs to be completed.
We hope that readers will confront their own data and problem settings,
choose what to keep as a baseline model and what to absorb into a flow,
and take control of both inference and generation with their own hands,
breathing new life into open problems.
Beyond that lies a future in which \textbf{generative AI and statistics can be discussed on the same horizon}.

\appendix
\chapter{Appendix}
The material below supplements parts that could not be fully covered in the main text.

\section{GGM Numerical Experiment}
\index{gaussian graphical model@Gaussian graphical model}
\label{app:ggm}

In this appendix, for the GGM in Example~\ref{ex:ggm-sm}, we summarize an experimental setup that compares
the regularized maximum likelihood estimator (Graphical Lasso) and a regularized score matching estimator.
\index{score matching@score matching}
The aim is to illustrate, through a concrete example, that optimization involving $\log\det$ can become computationally heavy in high dimensions.

\paragraph{Setup}
We set the dimension $d=200$ and the sample size $n=120$.
The true precision matrix $K^\ast$ is generated by randomly placing sparse off-diagonal entries and then applying a diagonal adjustment to ensure $K^\ast\succ 0$.
The data are generated from
\[
X_1,\dots,X_n \stackrel{\mathrm{i.i.d.}}{\sim} \mathcal{N}\!\bigl(0,(K^\ast)^{-1}\bigr),
\]
and we use the centered sample covariance
\[
S=\frac{1}{n}\sum_{i=1}^n (X_i-\bar X)(X_i-\bar X)^\top.
\]

\paragraph{Regularized MLE (Graphical Lasso)}
We numerically solve \citep{friedman2008sparse}
\[
\widehat K_{\mathrm{MLE}}
=\argmin_{K\succ 0}
\Bigl\{
\mathrm{tr}(SK)-\log\det K+\alpha \sum_{i\neq j}|K_{ij}|
\Bigr\}.
\]

\paragraph{Regularized score matching (ridge + $\ell_1$)}
Since $S$ can be singular when $n<d$, we add a ridge term for stabilization:
\begin{equation}
\widehat J_{\lambda,\rho}(K)
=
\tfrac12\,\mathrm{tr}(KSK)-\mathrm{tr}(K)
+\frac{\rho}{2}\|K\|_F^2
+\lambda\sum_{i\neq j}|K_{ij}|,
\qquad K=K^\top.
\label{eq:sm_ridge_l1}
\end{equation}
The gradient of the smooth part is
\[
\nabla f(K)=\tfrac12(SK+KS)-I+\rho K,
\]
and the nonsmooth part can be handled by a proximal update using soft-thresholding on the off-diagonal entries.

\begin{algorithm}[t]
\caption{Ridge-$\ell_1$ score matching for a GGM (proximal gradient method)}
\label{alg:sm_prox}
\begin{algorithmic}[1]
\Require Data $x_1,\dots,x_n\in\mathbb{R}^d$, regularization parameters $\lambda>0,\rho>0$, maximum iterations $T$
\Ensure Estimated precision matrix $\widehat K_{\mathrm{SM}}$
\State Center the data and compute $S\leftarrow \frac1n\sum_{i=1}^n (x_i-\bar x)(x_i-\bar x)^\top$
\State $\eta \leftarrow 0.9/(\|S\|_2+\rho)$, \ $K^{(0)}\leftarrow I_d$
\For{$t=0,1,\dots,T-1$}
  \State $G^{(t)} \leftarrow \tfrac12(SK^{(t)}+K^{(t)}S)-I_d+\rho K^{(t)}$
  \State $\widetilde K \leftarrow K^{(t)}-\eta G^{(t)}$
  \State $K^{(t+1)} \leftarrow \mathrm{Soft}_{\eta\lambda}(\widetilde K)$ \Comment{threshold off-diagonal entries only}
  \State $K^{(t+1)} \leftarrow \tfrac12\bigl(K^{(t+1)}+(K^{(t+1)})^\top\bigr)$ \Comment{symmetrize}
\EndFor
\State To guard against numerical error, apply a small shift so that the minimum eigenvalue becomes positive, and output $\widehat K_{\mathrm{SM}}$
\end{algorithmic}
\end{algorithm}

\paragraph{Evaluation metrics}
We measure the elementwise RMSE by
\[
\mathrm{RMSE}(\widehat K)
=
\Bigl(\frac1{d^2}\sum_{i=1}^d\sum_{j=1}^d(\widehat K_{ij}-K^\ast_{ij})^2\Bigr)^{1/2}.
\]
Computation time is measured by wall-clock time (with synchronization when using a GPU).

\begin{table}[t]
\centering
\caption{(Example) Monte Carlo mean and standard deviation of estimation accuracy (RMSE) and computation time (CT). $d=200,n=120$.}
\label{tab:ggm_sum}
\begin{tabular}{lcccc}
\toprule
 & RMSE\_MLE & CT\_MLE (sec) & RMSE\_SM & CT\_SM (sec)\\
\midrule
mean & 0.182256 & 1.595136 & 0.051682 & 0.127432\\
std  & 0.001820 & 1.511921 & 0.001156 & 0.016656\\
\bottomrule
\end{tabular}
\end{table}

\paragraph{Comment}
The relative relationship between RMSE and CT in Table~\ref{tab:ggm_sum} can depend on the regularization parameters and stopping criteria.
Nevertheless, in optimization problems involving $\log\det$, each iteration typically requires computing $K^{-1}$ (or a Cholesky factorization),
and this cost can become dominant in high dimensions, which is the point emphasized here.

\chapter*{List of Symbols}
\label{symbol}
\addcontentsline{toc}{chapter}{List of Symbols}

The principal symbols and definitions used in this book are summarized below.

\subsection*{Probability, statistics, and sets}
\begin{longtable}{p{3.2cm} p{11cm}}
    $\mathbb{R}^d$ & The $d$-dimensional Euclidean space. \\
    $P, Q$ & Probability measures. \\
    $p(x), q(x),\rho(x)$ & Probability density functions. \\
    $X \sim P$ & A random variable $X$ follows distribution $P$. \\
    $\mathbb{E}[f(X)]$ & Expectation. When the underlying distribution $P$ is explicit, we write $\mathbb{E}_{X \sim P}[f(X)]$ or $\mathbb{E}_P[f]$. \\
    $\widehat{P}_n$ & Empirical measure: $\frac{1}{n}\sum_{i=1}^n \delta_{X_i}$. \\
    $\mathcal{N}(\mu, \Sigma)$ & Gaussian distribution with mean $\mu$ and covariance matrix $\Sigma$. \\
    $\mathrm{Unif}(0, 1)$ & Uniform distribution on $(0,1)$. \\
    $W_p(\mu, \nu)$ & $p$-Wasserstein distance. \\
    $D_{\mathrm{KL}}(P \| Q)$ & Kullback--Leibler divergence. \\
    $\psi(P)$ & Target estimand (e.g., ATE, QTE, policy value). \\
    $\eta$ & Generic notation for nuisance functions (e.g., propensity score, conditional mean). \\
\end{longtable}

\subsection*{Differentiation, geometry, and vector fields}
\begin{longtable}{p{3.2cm} p{11cm}}
    $\nabla_x f(x)$ & Gradient of $f$ with respect to $x$. \\
    $\nabla \cdot v(x)$ & Divergence of a vector field $v$: $\sum_{j=1}^d \partial_{x_j} v_j(x)$. \\
    $\Delta f(x)$ & Laplacian: $\nabla \cdot \nabla f(x)$. \\
    $v_t(x)$ & Vector field (velocity field) at time $t$; also written as $v(t,x)$. \\
    $s_p(x)$ & Score function of density $p$: $\nabla_x \log p(x)$. \\
    $T_{\#} P$ & Pushforward of measure $P$ by a map $T$. If $X \sim P$, then $T(X) \sim T_{\#} P$. \\
    $\Phi_t(x)$ & Flow map of an ODE: the position at time $t$ of a particle starting at $x$ at time $0$. \\
\end{longtable}

\subsection*{Flow Matching and generative models}
\begin{longtable}{p{3.2cm} p{11cm}}
    $t$ & Pseudo-time or interpolation parameter, typically $t \in [0,1]$. \\
    $\pi(x)$ & Reference distribution, typically a known distribution such as the standard Gaussian. \\
    $\rho_t(x)$ & Density of the intermediate distribution (probability path) at time $t$. \\
    $u_t(x \mid x_1)$ & Conditional vector field (target velocity) with endpoint $x_1$ fixed. \\
    $X_0, X_1$ & Random variables representing the start point (noise) and the end point (data), respectively. \\
\end{longtable}

\subsection*{Causal inference}
\begin{longtable}{p{3.2cm} p{11cm}}
    $A$ & Treatment (intervention) variable, typically $A \in \{0,1\}$. \\
    $Y$ & Observed outcome variable. \\
    $Y(a)$ & Potential outcome under treatment $A=a$. \index{potential outcome@Potential outcome}\\
    $\mathrm{do}(A=a)$ & Intervention operator that sets $A$ to $a$ in a structural model. \\
    $\tau(x)$ & Conditional average treatment effect (CATE): $\mathbb{E}[Y(1)-Y(0)\mid X=x]$. \\
    $e(x)$ & Propensity score: $P(A=1 \mid X=x)$. \\
    $\mu_a(x)$ & Outcome regression: $\mathbb{E}[Y \mid A=a, X=x]$. \\
\end{longtable}

\bibliographystyle{plainnat}
\bibliography{refs-eng}

@book{amari2000methods,
  title={Methods of information geometry},
  author={Amari, Shun-ichi and Nagaoka, Hiroshi},
  volume={191},
  year={2000},
  publisher={American Mathematical Soc.}
}

@article{belloni2014pds,
  author  = {Belloni, Alexandre and Chernozhukov, Victor and Hansen, Christian},
  title   = {Inference on Treatment Effects after Selection among High-Dimensional Controls},
  journal = {The Review of Economic Studies},
  year    = {2014},
  volume  = {81},
  number  = {2},
  pages   = {608--650},
  doi     = {10.1093/restud/rdt044}
}

@article{zhang2014jssb,
  author  = {Zhang, Cun-Hui and Zhang, Stephanie S.},
  title   = {Confidence Intervals for Low Dimensional Parameters in High Dimensional Linear Models},
  journal = {Journal of the Royal Statistical Society: Series B (Statistical Methodology)},
  year    = {2014},
  volume  = {76},
  number  = {1},
  pages   = {217--242},
  doi     = {10.1111/rssb.12026}
}

@article{vandegeer2014aos,
  author  = {van de Geer, Sara and B\"{u}hlmann, Peter and Ritov, Ya'acov and Dezeure, Ruben},
  title   = {On Asymptotically Optimal Confidence Regions and Tests for High-Dimensional Models},
  journal = {The Annals of Statistics},
  year    = {2014},
  volume  = {42},
  number  = {3},
  pages   = {1166--1202},
  doi     = {10.1214/14-AOS1221}
}

@article{panaretos2019statistical,
  title   = {Statistical Aspects of Wasserstein Distances},
  author  = {Panaretos, Victor M. and Zemel, Yoav},
  journal = {Annual Review of Statistics and Its Application},
  year    = {2019},
  volume  = {6},
  pages   = {405--431},
  doi     = {10.1146/annurev-statistics-030718-104938}
}

@article{peyre2019computational,
  title   = {Computational Optimal Transport},
  author  = {Peyr{\'e}, Gabriel and Cuturi, Marco},
  journal = {Foundations and Trends in Machine Learning},
  year    = {2019},
  volume  = {11},
  pages   = {355--607},
  note    = {no.~5--6},
  doi     = {10.1561/2200000073}
}

@inproceedings{arjovsky2017wgan,
  title     = {Wasserstein Generative Adversarial Networks},
  author    = {Arjovsky, Martin and Chintala, Soumith and Bottou, L{\'e}on},
  booktitle = {Proceedings of the 34th International Conference on Machine Learning (ICML)},
  series    = {Proceedings of Machine Learning Research},
  volume    = {70},
  pages     = {214--223},
  year      = {2017},
  publisher = {PMLR}
}

@inproceedings{cuturi2013sinkhorn,
  title     = {Sinkhorn Distances: Lightspeed Computation of Optimal Transport},
  author    = {Cuturi, Marco},
  booktitle = {Advances in Neural Information Processing Systems},
  year      = {2013},
  volume    = {26},
  pages     = {2292--2300}
}

@article{kang2007demystifying,
  title={Demystifying double robustness: A comparison of alternative strategies for estimating a population mean from incomplete data},
  author={Kang, Joseph DY and Schafer, Joseph L},
  journal = {Statistical Science},
  year={2007},
  volume={22},
  page={523-539}
}

@article{lipman2022flowmatching,
  title   = {Flow Matching for Generative Modeling},
  author  = {Lipman, Yaron and Chen, Ricky T. Q. and Ben-Hamu, Heli and Nickel, Maximilian and Le, Matt},
  journal = {arXiv preprint arXiv:2210.02747},
  year    = {2022},
  url     = {https://arxiv.org/abs/2210.02747}
}

@article{tong2023cfm,
  title   = {Improving and Generalizing Flow-Based Generative Models with Minibatch Optimal Transport},
  author  = {Tong, Alexander and Malkin, Nikolay and Huguet, Guillaume and Zhang, Yanlei and Rector-Brooks, Jarrid and Fatras, Kilian and Wolf, Guy and Bengio, Yoshua},
  journal = {arXiv preprint arXiv:2302.00482},
  year    = {2023},
  url     = {https://arxiv.org/abs/2302.00482}
}

@article{liu2022rectified,
  title   = {Flow Straight and Fast: Learning to Generate and Transfer Data with Rectified Flow},
  author  = {Liu, Xingchao and Gong, Chengyue and Liu, Qiang},
  journal = {arXiv preprint arXiv:2209.03003},
  year    = {2022},
  url     = {https://arxiv.org/abs/2209.03003}
}

@article{song2020scoresde,
  title   = {Score-Based Generative Modeling through Stochastic Differential Equations},
  author  = {Song, Yang and Sohl-Dickstein, Jascha and Kingma, Diederik P. and Kumar, Abhishek and Ermon, Stefano and Poole, Ben},
  journal = {arXiv preprint arXiv:2011.13456},
  year    = {2020},
  url     = {https://arxiv.org/abs/2011.13456}
}

@inproceedings{grathwohl2019ffjord,
  title     = {FFJORD: Free-form Continuous Dynamics for Scalable Reversible Generative Models},
  author    = {Grathwohl, Will and Chen, Ricky T. Q. and Bettencourt, Jesse and Sutskever, Ilya and Duvenaud, David},
  booktitle = {International Conference on Learning Representations},
  year      = {2019},
  url       = {https://arxiv.org/abs/1810.01367}
}

@inproceedings{goodfellow2014gan,
  title     = {Generative Adversarial Nets},
  author    = {Goodfellow, Ian and Pouget-Abadie, Jean and Mirza, Mehdi and Xu, Bing and Warde-Farley, David and Ozair, Sherjil and Courville, Aaron and Bengio, Yoshua},
  booktitle = {Advances in Neural Information Processing Systems},
  volume    = {27},
  year      = {2014},
  pages     = {2672--2680}
}

@inproceedings{kingma2014auto,
  title     = {Auto-Encoding Variational Bayes},
  author    = {Kingma, Diederik P. and Welling, Max},
  booktitle = {International Conference on Learning Representations},
  year      = {2014},
  note      = {arXiv:1312.6114}
}

@inproceedings{rezende2015variational,
  title     = {Variational Inference with Normalizing Flows},
  author    = {Rezende, Danilo Jimenez and Mohamed, Shakir},
  booktitle = {International Conference on Machine Learning},
  year      = {2015},
  pages     = {1530--1538},
  publisher = {PMLR},
  note      = {arXiv:1505.05770}
}

@inproceedings{dinh2017realnvp,
  title     = {Density Estimation using Real {NVP}},
  author    = {Dinh, Laurent and Sohl-Dickstein, Jascha and Bengio, Samy},
  booktitle = {International Conference on Learning Representations},
  year      = {2017},
  note      = {arXiv:1605.08803}
}

@inproceedings{sohl2015deep,
  title     = {Deep Unsupervised Learning using Nonequilibrium Thermodynamics},
  author    = {Sohl-Dickstein, Jascha and Weiss, Eric and Maheswaranathan, Niru and Ganguli, Surya},
  booktitle = {International Conference on Machine Learning},
  year      = {2015},
  pages     = {2256--2265},
  publisher = {PMLR},
  note      = {arXiv:1503.03585}
}

@inproceedings{ho2020ddpm,
  title     = {Denoising Diffusion Probabilistic Models},
  author    = {Ho, Jonathan and Jain, Ajay and Abbeel, Pieter},
  booktitle = {Advances in Neural Information Processing Systems},
  volume    = {33},
  year      = {2020},
  pages     = {6840--6851},
  note      = {arXiv:2006.11239}
}

@inproceedings{song2021scorebased,
  title     = {Score-Based Generative Modeling through Stochastic Differential Equations},
  author    = {Song, Yang and Sohl-Dickstein, Jascha and Kingma, Diederik P. and Kumar, Abhishek and Ermon, Stefano and Poole, Ben},
  booktitle = {International Conference on Learning Representations},
  year      = {2021},
  note      = {arXiv:2011.13456}
}

@inproceedings{chen2018neuralode,
  title     = {Neural Ordinary Differential Equations},
  author    = {Chen, Ricky T. Q. and Rubanova, Yulia and Bettencourt, Jesse and Duvenaud, David},
  booktitle = {Advances in Neural Information Processing Systems},
  volume    = {31},
  year      = {2018},
  note      = {arXiv:1806.07366}
}

@article{hyvarinen2005score,
  title   = {Estimation of Non-Normalized Statistical Models by Score Matching},
  author  = {Hyv{\"a}rinen, Aapo},
  journal = {Journal of Machine Learning Research},
  volume  = {6},
  pages   = {695--709},
  year    = {2005}
}

@article{vincent2011connection,
  title   = {A Connection Between Score Matching and Denoising Autoencoders},
  author  = {Vincent, Pascal},
  journal = {Neural Computation},
  volume  = {23},
  number  = {7},
  pages   = {1661--1674},
  year    = {2011},
  doi     = {10.1162/NECO_a_00142}
}

@inproceedings{james1961estimation,
  title     = {Estimation with Quadratic Loss},
  author    = {James, William and Stein, Charles},
  booktitle = {Proceedings of the Fourth Berkeley Symposium on Mathematical Statistics and Probability},
  volume    = {1},
  pages     = {361--379},
  year      = {1961}
}

@article{stein1981estimation,
  title   = {Estimation of the Mean of a Multivariate Normal Distribution},
  author  = {Stein, Charles},
  journal = {The Annals of Statistics},
  volume  = {9},
  number  = {6},
  pages   = {1135--1151},
  year    = {1981}
}

@inproceedings{lipman2023flowmatching,
  title     = {Flow Matching for Generative Modeling},
  author    = {Lipman, Yaron and Chen, Ricky T. Q. and Ben-Hamu, Heli and Nickel, Maximilian and Le, Matt},
  booktitle = {International Conference on Learning Representations},
  year      = {2023},
  note      = {arXiv:2210.02747}
}

@article{lipman2024flowmatchingguide,
  title   = {Flow Matching Guide and Code},
  author  = {Lipman, Yaron and Havasi, M{\'a}rton and Holderrieth, Peter and Shaul, Neta and Le, Matt and Karrer, Brian and Chen, Ricky T. Q. and L{\'o}pez-Paz, David and Ben-Hamu, Heli and Gat, Itai},
  journal = {arXiv preprint},
  year    = {2024},
  note    = {arXiv:2412.06264}
}

@article{tong2023improving,
  title   = {Improving and Generalizing Flow-Based Generative Models with Minibatch Optimal Transport},
  author  = {Tong, Alexander and Malkin, Nikolay and Fatras, Kilian and Ahuja, Sarthak and Zhang, Tong and Wolf, Guy and Bengio, Yoshua},
  journal = {Transactions on Machine Learning Research},
  year    = {2023},
  note    = {arXiv:2302.00482}
}

@article{chen2023geom,
  title   = {Riemannian Flow Matching on General Geometries},
  author  = {Chen, Ricky T. Q. and others},
  journal = {arXiv preprint},
  year    = {2023},
  note    = {arXiv:2302.03660}
}

@article{albergo2025stochastic,
  title   = {Stochastic Interpolants: A Unifying Framework for Flows and Diffusions},
  author  = {Albergo, Michael S. and Vanden-Eijnden, Eric},
  journal = {Journal of Machine Learning Research},
  year    = {2025}
}

@article{benamou2000computational,
  title   = {A Computational Fluid Mechanics Solution to the Monge-Kantorovich Mass Transfer Problem},
  author  = {Benamou, Jean-David and Brenier, Yann},
  journal = {SIAM Journal on Scientific Computing},
  volume  = {22},
  number  = {1},
  pages   = {1--23},
  year    = {2000},
  doi     = {10.1137/S1064827597327209}
}

@book{villani2009optimal,
  title     = {Optimal Transport: Old and New},
  author    = {Villani, C{\'e}dric},
  publisher = {Springer},
  year      = {2009},
  series    = {Grundlehren der mathematischen Wissenschaften},
  volume    = {338},
  doi       = {10.1007/978-3-540-71050-9}
}

@article{eguchi2025robust,
  title={Robust inference using density-powered Stein operators},
  author={Eguchi, Shinto},
  journal={arXiv preprint arXiv:2511.03963},
  year={2025}
}

@article{huk2025diffusion,
  title={Diffusion and Flow-based Copulas: Forgetting and Remembering Dependencies},
  author={Huk, David and Damoulas, Theodoros},
  journal={arXiv preprint arXiv:2509.19707},
  year={2025}
}

@incollection{oksendal2003stochastic,
  title={Stochastic differential equations},
  author={{\O}ksendal, Bernt},
  booktitle={Stochastic differential equations: an introduction with applications},
  pages={38--50},
  year={2003},
  publisher={Springer}
}

@article{Rubin1976,
  title={Inference and missing data},
  author={Rubin, Donald B},
  journal={Biometrika},
  volume={63},
  number={3},
  pages={581--592},
  year={1976}
}

@book{LittleRubin2019,
  title={Statistical Analysis with Missing Data},
  author={Little, Roderick J A and Rubin, Donald B},
  edition={3rd},
  year={2019},
  publisher={John Wiley \& Sons}
}

@article{Blanchet2019,
  title={Robust Wasserstein profile inference and applications to machine learning},
  author={Blanchet, Jose and Kang, Yang and Murthy, Karthyek},
  journal={Journal of Applied Probability},
  volume={56},
  number={3},
  pages={830--857},
  year={2019}
}

@inproceedings{Muzellec2020,
  title={Missing data imputation using optimal transport},
  author={Muzellec, Boris and Josse, Julie and Boyer, Claire and Cuturi, Marco},
  booktitle={International Conference on Machine Learning},
  pages={7130--7140},
  year={2020},
  organization={PMLR}
}

@book{Rubin1987,
  title={Multiple Imputation for Nonresponse in Surveys},
  author={Rubin, Donald B},
  year={1987},
  publisher={John Wiley \& Sons}
}

@article{Duchi2021,
  title={Statistics of robust optimization: A generalized empirical likelihood approach},
  author={Duchi, John C and Namkoong, Hongseok},
  journal={Mathematics of Operations Research},
  volume={46},
  number={3},
  pages={946--969},
  year={2021},
  publisher={INFORMS}
}

@inproceedings{liu2016kernelized,
  title={A Kernelized {S}tein Discrepancy for Goodness-of-fit Tests},
  author={Liu, Qiang and Lee, Jason and Jordan, Michael},
  booktitle={International Conference on Machine Learning},
  pages={276--284},
  year={2016},
  organization={PMLR}
}

@inproceedings{chwialkowski2016kernel,
  title={A Kernel Test of Goodness of Fit},
  author={Chwialkowski, Kacper and Strathmann, Heiko and Gretton, Arthur},
  booktitle={International Conference on Machine Learning},
  pages={2606--2615},
  year={2016},
  organization={PMLR}
}

@inproceedings{gorham2015measuring,
  title={Measuring Sample Quality with {S}tein's Method},
  author={Gorham, Jackson and Mackey, Lester},
  booktitle={Advances in Neural Information Processing Systems},
  volume={28},
  year={2015}
}

@inproceedings{liu2016stein,
  title={{S}tein Variational Gradient Descent: A General Purpose {B}ayesian Inference Algorithm},
  author={Liu, Qiang and Wang, Dilin},
  booktitle={Advances in Neural Information Processing Systems},
  volume={29},
  year={2016}
}

@article{goldfeld2024statistical,
  author  = {Goldfeld, Ziv and Balakrishnan, Sivaraman and Munk, Axel and Niles-Weed, Jonathan},
  title   = {Statistical Inference for Optimal Transport Maps: Recent Advances and Perspectives},
  journal = {Statistical Science},
  year    = {2024},
  note    = {In press / Published online},
  doi     = {10.1214/23-STS912}
}

@article{Cox1972,
  author  = {Cox, D. R.},
  title   = {Regression Models and Life-Tables},
  journal = {Journal of the Royal Statistical Society: Series B (Methodological)},
  year    = {1972},
  volume  = {34},
  number  = {2},
  pages   = {187--220},
  doi     = {10.1111/j.2517-6161.1972.tb00899.x}
}

@article{Cox1975,
  author  = {Cox, D. R.},
  title   = {Partial likelihood},
  journal = {Biometrika},
  year    = {1975},
  volume  = {62},
  number  = {2},
  pages   = {269--276},
  doi     = {10.1093/biomet/62.2.269}
}

@article{AndersenGill1982,
  author  = {Andersen, P. K. and Gill, R. D.},
  title   = {Cox's Regression Model for Counting Processes: A Large Sample Study},
  journal = {The Annals of Statistics},
  year    = {1982},
  volume  = {10},
  number  = {4},
  pages   = {1100--1120},
  doi     = {10.1214/aos/1176345976}
}

@book{FlemingHarrington1991,
  author    = {Fleming, Thomas R. and Harrington, David P.},
  title     = {Counting Processes and Survival Analysis},
  publisher = {John Wiley \& Sons},
  address   = {New York},
  year      = {1991},
  isbn      = {047152218X}
}

@book{TherneauGrambsch2000,
  author    = {Therneau, Terry M. and Grambsch, Patricia M.},
  title     = {Modeling Survival Data: Extending the Cox Model},
  publisher = {Springer},
  address   = {New York},
  year      = {2000},
  doi       = {10.1007/978-1-4757-3294-8}
}

@book{vanDerVaart1998,
  author    = {van der Vaart, A. W.},
  title     = {Asymptotic Statistics},
  publisher = {Cambridge University Press},
  address   = {Cambridge},
  year      = {1998}
}

@book{Tsiatis2006,
  author    = {Tsiatis, Anastasios A.},
  title     = {Semiparametric Theory and Missing Data},
  publisher = {Springer},
  address   = {New York},
  year      = {2006},
  doi       = {10.1007/0-387-37345-4},
  isbn      = {978-0-387-32448-7}
}

@article{chernozhukov2018ddml,
  author  = {Chernozhukov, Victor and Chetverikov, Denis and Demirer, Mert and Duflo, Esther and Hansen, Christian and Newey, Whitney and Robins, James},
  title   = {Double/debiased machine learning for treatment and structural parameters},
  journal = {The Econometrics Journal},
  year    = {2018},
  volume  = {21},
  number  = {1},
  pages   = {C1--C68},
  doi     = {10.1111/ectj.12097}
}

@book{arnold_1989_mmc,
  title     = {Mathematical Methods of Classical Mechanics},
  author    = {Arnold, Vladimir I.},
  publisher = {Springer},
  year      = {1989},
  edition   = {2},
  series    = {Graduate Texts in Mathematics},
  volume    = {60}
}

@book{marsden_ratiu_1999_mech_sym,
  title     = {Introduction to Mechanics and Symmetry: A Basic Exposition of Classical Mechanical Systems},
  author    = {Marsden, Jerrold E. and Ratiu, Tudor S.},
  publisher = {Springer},
  year      = {1999},
  edition   = {2},
  series    = {Texts in Applied Mathematics},
  volume    = {17}
}

@article{murphy2003dtr,
  title   = {Optimal dynamic treatment regimes},
  author  = {Murphy, Susan A.},
  journal = {Journal of the Royal Statistical Society: Series B (Statistical Methodology)},
  volume  = {65},
  number  = {2},
  pages   = {331--355},
  year    = {2003},
  doi     = {10.1111/1467-9868.00389}
}

@article{robins2000msm,
  title   = {Marginal structural models and causal inference in epidemiology},
  author  = {Robins, James M. and Hern{\'a}n, Miguel A. and Brumback, Babette},
  journal = {Epidemiology},
  volume  = {11},
  number  = {5},
  pages   = {550--560},
  year    = {2000},
  doi     = {10.1097/00001648-200009000-00011}
}

@book{hernanrobins2020whatif,
  title     = {Causal Inference: What If},
  author    = {Hern{\'a}n, Miguel A. and Robins, James M.},
  publisher = {Chapman \& Hall/CRC},
  address   = {Boca Raton},
  year      = {2020}
}

@book{chakraborty2013dtr,
  title     = {Statistical Methods for Dynamic Treatment Regimes: Reinforcement Learning, Causal Inference, and Personalized Medicine},
  author    = {Chakraborty, Bibhas and Moodie, Erica E. M.},
  publisher = {Springer},
  address   = {New York},
  year      = {2013},
  doi       = {10.1007/978-1-4614-7428-9}
}

@book{suttonbarto2018rl,
  title     = {Reinforcement Learning: An Introduction},
  author    = {Sutton, Richard S. and Barto, Andrew G.},
  publisher = {The MIT Press},
  address   = {Cambridge, MA},
  year      = {2018}
}

@article{watkins1992qlearning,
  title   = {Q-learning},
  author  = {Watkins, Christopher J. C. H. and Dayan, Peter},
  journal = {Machine Learning},
  volume  = {8},
  number  = {3--4},
  pages   = {279--292},
  year    = {1992},
  doi     = {10.1007/BF00992698}
}

@article{robins1986hws,
  title   = {A new approach to causal inference in mortality studies with a sustained exposure period---application to control of the healthy worker survivor effect},
  author  = {Robins, James M.},
  journal = {Mathematical Modelling},
  volume  = {7},
  number  = {9--12},
  pages   = {1393--1512},
  year    = {1986},
  doi     = {10.1016/0270-0255(86)90088-6}
}

@article{kunzel2019metalearners,
  title   = {Metalearners for estimating heterogeneous treatment effects using machine learning},
  author  = {K{\"u}nzel, S{\"o}ren R. and Sekhon, Jasjeet S. and Bickel, Peter J. and Yu, Bin},
  journal = {Proceedings of the National Academy of Sciences},
  volume  = {116},
  number  = {10},
  pages   = {4156--4165},
  year    = {2019},
  doi     = {10.1073/pnas.1804597116}
}

@article{nie2021quasioracle,
  title   = {Quasi-oracle estimation of heterogeneous treatment effects},
  author  = {Nie, Xinkun and Wager, Stefan},
  journal = {Biometrika},
  volume  = {108},
  number  = {2},
  pages   = {299--319},
  year    = {2021},
  doi     = {10.1093/biomet/asaa076}
}

@article{athey2019grf,
  title   = {Generalized random forests},
  author  = {Athey, Susan and Tibshirani, Julie and Wager, Stefan},
  journal = {The Annals of Statistics},
  volume  = {47},
  number  = {2},
  pages   = {1148--1178},
  year    = {2019},
  doi     = {10.1214/18-AOS1709}
}

@article{friedman2008sparse,
  title   = {Sparse inverse covariance estimation with the graphical lasso},
  author  = {Friedman, Jerome and Hastie, Trevor and Tibshirani, Robert},
  journal = {Biostatistics},
  volume  = {9},
  number  = {3},
  pages   = {432--441},
  year    = {2008},
  doi     = {10.1093/biostatistics/kxm045}
}

@article{chernozhukov2022locally,
  title   = {Locally Robust Semiparametric Estimation},
  author  = {Chernozhukov, Victor and Escanciano, Juan Carlos and Ichimura, Hidehiko and Newey, Whitney K. and Robins, James M.},
  journal = {Econometrica},
  volume  = {90},
  number  = {4},
  pages   = {1501--1535},
  year    = {2022},
  doi     = {10.3982/ECTA16294}
}

@article{white1982maximum,
  title={Maximum likelihood estimation of misspecified models},
  author={White, Halbert},
  journal={Econometrica: Journal of the econometric society},
  pages={1--25},
  year={1982},
  publisher={JSTOR}
}

@book{Joe2014,
  author    = {Harry Joe},
  title     = {Dependence Modeling with Copulas},
  publisher = {Chapman \& Hall/CRC},
  year      = {2014},
  series    = {Chapman \& Hall/CRC Monographs on Statistics and Applied Probability},
  isbn      = {9781466583221},
  doi       = {10.1201/b17116}
}

@article{KamtheAssefaDeisenroth2021,
  author       = {Sanket Kamthe and Samuel Assefa and Marc Peter Deisenroth},
  title        = {Copula Flows for Synthetic Data Generation},
  journal      = {CoRR},
  volume       = {abs/2101.00598},
  year         = {2021},
  eprint       = {2101.00598},
  archivePrefix= {arXiv},
  primaryClass = {stat.ML}
}

@article{MitskopoulosAmvrosiadisOnken2022,
  author  = {Lazaros Mitskopoulos and Theoklitos Amvrosiadis and Arno Onken},
  title   = {Mixed vine copula flows for flexible modeling of neural dependencies},
  journal = {Frontiers in Neuroscience},
  volume  = {16},
  pages   = {910122},
  year    = {2022},
  doi     = {10.3389/fnins.2022.910122}
}

@book{Villani2009,
  author    = {C{\'e}dric Villani},
  title     = {Optimal Transport: Old and New},
  publisher = {Springer},
  year      = {2009},
  series    = {Grundlehren der mathematischen Wissenschaften},
  volume    = {338},
  doi       = {10.1007/978-3-540-71050-9}
}

@book{AmbrosioGigliSavare2008,
  author    = {Luigi Ambrosio and Nicola Gigli and Giuseppe Savar{\'e}},
  title     = {Gradient Flows: In Metric Spaces and in the Space of Probability Measures},
  edition   = {2},
  publisher = {Birkh{\"a}user Basel},
  year      = {2008},
  series    = {Lectures in Mathematics ETH Z{\"u}rich},
  doi       = {10.1007/978-3-7643-8722-8}
}

@book{CoddingtonLevinson1955,
  author    = {Earl A. Coddington and Norman Levinson},
  title     = {Theory of Ordinary Differential Equations},
  publisher = {McGraw--Hill},
  year      = {1955}
}

@book{vdVWellner1996,
  author    = {Aad W. van der Vaart and Jon A. Wellner},
  title     = {Weak Convergence and Empirical Processes: With Applications to Statistics},
  publisher = {Springer},
  year      = {1996},
  series    = {Springer Series in Statistics},
  doi       = {10.1007/978-1-4757-2545-2}
}

@article{BartlettMendelson2002,
  author  = {Peter L. Bartlett and Shahar Mendelson},
  title   = {Rademacher and Gaussian Complexities: Risk Bounds and Structural Results},
  journal = {Journal of Machine Learning Research},
  year    = {2002},
  volume  = {3},
  pages   = {463--482}
}

@article{ShenWong1994,
  author  = {Xiaotong Shen and Wing Hung Wong},
  title   = {Convergence Rate of Sieve Estimates},
  journal = {The Annals of Statistics},
  year    = {1994},
  volume  = {22},
  number  = {2},
  pages   = {580--615}
}

@incollection{Chen2007,
  author    = {Xiaohong Chen},
  title     = {Large Sample Sieve Estimation of Semi-Nonparametric Models},
  booktitle = {Handbook of Econometrics},
  editor    = {James J. Heckman and Edward E. Leamer},
  volume    = {6B},
  chapter   = {76},
  pages     = {5549--5632},
  publisher = {Elsevier},
  year      = {2007},
  doi       = {10.1016/S1573-4412(07)06076-X}
}

@article{ChenShen1998,
  author  = {Xiaohong Chen and Xiaotong Shen},
  title   = {Sieve Extremum Estimates for Weakly Dependent Data},
  journal = {Econometrica},
  year    = {1998},
  volume  = {66},
  number  = {2},
  pages   = {289--314}
}

@article{eguchi2025implicit,
  title={Implicit geometric regularization in flow matching via density weighted Stein operators},
  author={Eguchi, Shinto},
  journal={arXiv preprint arXiv:2512.23956},
  year={2025}
}

@book{inagaki2003,
  title     = {数理統計学},
  author    = {稲垣, 宣生},
  publisher = {裳華房},
  year      = {2003},
  isbn       = {978-4-7853-1406-4}
}

@book{Sato2023,
  author    = {佐藤, 竜馬},
  title     = {最適輸送の理論とアルゴリズム},
  publisher = {講談社},
  year      = {2023},
  series    = {機械学習プロフェッショナルシリーズ},
  isbn      = {978-4-06-530514-0}
}

@book{yanase2015probability,
title={確率と確率過程: 具体例で学ぶ確率論の考え方},
author={柳瀬, 眞一郎},
publisher={森北出版},
year={2015},
isbn={9784627061811}
}

\clearpage
\printindex
\end{document}